%% file: nlmc.tex
\documentclass[11pt,a4paper]{article}

\input{preamble_nlmc.tex}
\usepackage{authblk}
\newcommand{\aref}[1]{\hyperref[#1]{A\ref{#1}}}

\title{~~~~~Nonlinear matrix recovery using optimization \nl on the Grassmann manifold}

\author[1,2]{Florentin Goyens\footnote{This author’s work was supported by The Alan Turing Institute under the EPSRC Grant No. EP/N510129/1 and under the Turing Project Scheme.}}
\author[1,2]{Coralia Cartis\footnote{This author’s work was supported by The Alan Turing Institute under the Turing Project Scheme.}}
\author[2,3]{Armin Eftekhari}
\affil[1]{Mathematical Institute\\
University of Oxford\\
Oxford, United Kingdom}
\affil[2]{The Alan Turing Institute\\
London, United Kingdom}
\affil[3]{Department of Mathematics and Mathematical Statistics\\
Umeå University\\
Umeå, Sweden}
\date{\today}

\begin{document}
\bibliographystyle{abbrv}
\tabulinesep=1.2mm

\maketitle

\begin{abstract}
	We investigate the problem of recovering a partially observed high-rank matrix whose columns obey a nonlinear structure such as a union of subspaces, an algebraic variety or grouped in clusters. The recovery problem is formulated as the rank minimization of a nonlinear feature map applied to the original matrix, which is then further approximated by a constrained non-convex optimization problem involving the Grassmann manifold.
We propose two sets of algorithms, one arising from Riemannian optimization and the other as an alternating minimization scheme, both of which include first- and second-order variants.	 
	 Both sets of algorithms have theoretical guarantees. In particular, for the alternating minimization, we establish global convergence and worst-case complexity bounds. Additionally, using the Kurdyka-Lojasiewicz property, we show that the alternating minimization converges to a unique limit point.
	   We provide extensive numerical results for the recovery of union of subspaces and clustering under entry sampling and dense Gaussian sampling.  Our methods are competitive with existing approaches and, in particular, high accuracy is achieved in the recovery using Riemannian second-order methods.\nl
\textbf{Keywords:} nonlinear matrix recovery, nonconvex optimization, Riemannian optimization, second-order methods.
\end{abstract}

\section{Introduction}
\label{sec:introduction}
In the matrix recovery problem, one tries to estimate a matrix $M\in \mathbb{R}^{n\times s}$ from partial information.  The \emph{low-rank matrix recovery problem} deals with instances where the matrix $M$ is low-rank.  This problem has received great attention in the literature, as applications abound in recommender systems and engineering (see~\cite{davenport2016overview} and the references therein for an overview). It was shown in~\cite{Candes2009} that solving a convex program allows to recover the original matrix $M$ with very high probability, provided enough samples are available. However, solving this convex semi-definite problem for large instances is very costly in time and memory allocation. This has sparked the search for alternative nonconvex formulations of the problem~\cite{jain2010guaranteed,wen2012solving}. Riemannian optimization methods are used in some of the most efficient algorithm known to date for low-rank matrix completion. These methods solve optimization problems defined on smooth Riemannian manifolds, such as the manifold of fixed-rank matrices~\cite{Vandereycken2013} or the Grassmann manifold~\cite{boumal2015low}. 

 All traditional approaches to matrix completion fail if the matrix $M$ is high-rank. Our work is based on the recent discovery that an adaptation of traditional methods allows to recover specific classes of high-rank matrices~\cite{fan2018non,Ongie2017}. This problem is known as \emph{nonlinear matrix recovery} (or high-rank matrix recovery). Recovering high-rank matrices requires one to make assumptions on the structure of $M$. Let $m_1, \dots, m_s$ denote the columns of $M$. When the points $m_i \in \R^n$ belong to a low-dimensional subspace in $\R^n$, low-rank matrix recovery methods can be applied. Nonlinear matrix recovery attempts to recover $M$ when the points $m_i$ are related in a nonlinear way.
 
  Classically, for some integer $ m <ns$, the matrix $M$ satisfies $m$ linear equations of the type $\langle A_i,M\rangle =b_i$ for given matrices $A_i\in  \mathbb{R}^{n\times s} $ and a given vector $b\in \mathbb{R}^m$, where we use the usual inner product $\langle A_i,M\rangle = \trace(A_i^\top M)$. The matrices $A_i$ are assumed to be randomly drawn from a known distribution. One defines the linear operator 
\begin{equation}
\mathcal{A}: \R^{n\times s} \to \R^m \txt{ where }\A(M)_i = \langle A_i,M\rangle
\label{eq:measurements}
\end{equation}
 so as to have the compact notation $\mathcal{A}(M)=b$ for the measurements. When each matrix $A_i$ has exactly one non-zero entry which is equal to $1$, this is known as a matrix completion problem.  The matrix $M$ is then known on a subset $\Omega$ of the complete set of entries $\{1, \dots, n\} \times \{1, \dots, s\}$. Without loss of generality we assume $n\leq s$.
\paragraph{Problem description} 
  Nonlinear matrix recovery methods use features that map the columns of $M$ to a space of higher dimension. 
The \emph{feature map} is defined as
\begin{equation}
\varphi : \mathbb{R}^n \to \mathcal{F} : v \mapsto \varphi (v),
\end{equation}
 where $\mathcal{F}$ is a Hilbert space. If $\mathcal{F}$ is finite dimensional, we write $\mathcal{F}= \R^N$ where  $N$ is the dimension of the feature space, with $N\geq n$.  We obtain the \emph{feature matrix} $\Phi(M)$ by applying $\varphi$ to each column of $M$,
\begin{equation} 
 \Phi(M) = \begin{bmatrix}
\varphi(m_1) & \dots & \varphi(m_s)
\end{bmatrix}\in \R^{N\times s}.
\label{eq:features_matrix}
\end{equation}
The map $\varphi$ is chosen using a priori knowledge of the data so that the features of the data points $\varphi(m_i)$ for $i=1, \dots, s$, all belong to the same subspace in $\R^N$. The nonlinear structure in $M$ will cause a rank deficiency in the feature matrix $\Phi(M)$, even though $M$ may be full-rank.  This is illustrated in Figure~\ref{fig:map}. 
\begin{figure}[!h]
\centering
\begin{tikzpicture}[scale = 0.5]
\begin{axis}[hide axis,colormap/violet]
\addplot3[
	surf,
	domain=-2:2,
	domain y=-1.3:1.3,
] 
	{exp(-x^2-y^2)*x};
\end{axis}
\filldraw[fill=black] (1,2.5) circle (.1);
\filldraw[fill=black] (5,3.6) circle (.1);
\filldraw[fill=black] (5.4,2.9) circle (.1);
\filldraw[fill=black] (6,3.2) circle (.1);
\filldraw[fill=black] (6,3.2) circle (.1);
\filldraw[fill=black] (4.2,2) circle (.1);
\filldraw[fill=black] (4.5,2.5) circle (.1);
\filldraw[fill=black] (4,3) circle (.1);
\filldraw[fill=black] (3,3) circle (.1);
\filldraw[fill=black] (2,2.3) circle (.1);
\filldraw[fill=black] (1.7,2.8) circle (.1);
\filldraw[fill=black] (2.5,3.5) circle (.1);
\node[right] at (6.5,4) {$\mathbb{R}^n$};
\end{tikzpicture}
\begin{tikzpicture}[scale = 0.4]
\node[above] at (-6.6,0.5) {$\varphi$};
\draw[->,thick] (-7.6,-0.3) to [out = 0, in = 180] (-5.6,-0.3);
    \shade[left color=blue!40,right color=blue!30] (4,0) -- (3.6,-2.3) -- (-4.6,-0.3) -- (-2.7,1.5) -- (4,0);
\node[above] at (4,1) {$ \mathcal{F}$};
\filldraw[fill=black] (3,-1) circle (.1);
\filldraw[fill=black] (2.5,-0.8) circle (.1);
\filldraw[fill=black] (2,-1.2) circle (.1);
\filldraw[fill=black] (1.5,0.3) circle (.1);
\filldraw[fill=black] (1,-0.5) circle (.1);
\filldraw[fill=black] (3,-1) circle (.1);

\filldraw[fill=black] (-3,0.3) circle (.1);
\filldraw[fill=black] (-2.5,1.2) circle (.1);
\filldraw[fill=black] (-2,0.7) circle (.1);
\filldraw[fill=black] (-1.5,0.7) circle (.1);
\filldraw[fill=black] (-1,-0.5) circle (.1);
\filldraw[fill=black] (-3.3,-0.5) circle (.1);
\end{tikzpicture}
\caption{The feature map $\varphi$ is chosen to exploit the nonlinear structure.}
\label{fig:map}
\end{figure}
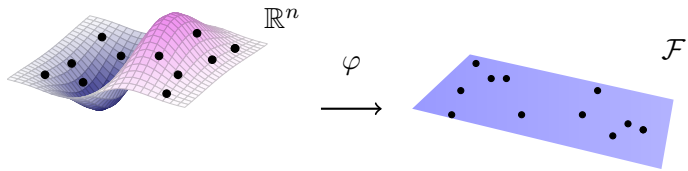

If the features are infinite dimensional or that $N$ is very large, the feature map should be represented using a kernel, which is known as the kernel trick. The set $\mathcal{F}$ is then called a reproducing kernel Hilbert space.  The kernel map represents the inner product between elements in the Hilbert space of features,
\begin{equation}
k: \R^n \times \R^n \to \R:	k(x,y) = \langle \varphi(x),\varphi(y)\rangle_\mathcal{F}.
\label{eq:def_kernel}
\end{equation} 
This allows to define the \emph{kernel matrix} of the data $\K(M,M)\in \R^{s \times s}$, with $\K_{ij}(M,M) = k(m_i,m_j)$. Throughout, we assume that $r = \rank(\Phi(M))$ is known and smaller than $\min(N,s)$. When $\mathcal{F}= \R^N$, we note that $\K(M,M) = \Phi(M)^\top \Phi(M) $ and therefore $\rank \big(\K(M,M)\big) = \rank \big(\Phi(M)\big)$. We use the term \emph{embedding} to denote a mapping to a higher-dimensional space, which may be performed using a kernel or a feature map.
In~\cite{Ongie2017}, Ongie et al. use the monomial kernel for the completion of matrices whose columns belong to an algebraic variety (a set defined by a finite number of polynomial equations). This can notably be applied to a union of subspaces. In~\cite{fan2018non}, Fan et al. use the monomial kernel and the Gaussian kernel on image inpainting problems. In Section~\ref{sec:features}, we detail why polynomial and Gaussian kernels may be used to model data which belongs to, respectively, an algebraic variety or a set of clusters.

\paragraph*{Problem formulation}
 For an appropriately chosen feature map, the nonlinear matrix recovery problem can be formulated as the rank minimization of the features matrix under the measurements constraint
 \begin{equation}
\left\lbrace
\begin{aligned}
& \underset{X\in \R^{n\times s}}{\min}
& & \rank(\Phi(X)) \\
&&&  \Acal(X) = b.
\end{aligned}
\right.
\label{eq:min_rank_phi}
\end{equation}
This seeks to find the matrix which fits the observations using a minimum number of independent features. As is the case for low-rank matrix recovery, minimizing the rank directly is NP-hard and should be avoided~\cite{Candes2009}. It is necessary to find a suitable approximation to the rank function. 

\paragraph{Related work}
\label{sec:related_work}
In essence,~\cite{fan2018non} and~\cite{Ongie2017} apply different minimization algorithms to the Schatten p-norm of the features which is defined by
\begin{equation} \label{eq:Schatten}
\left(\sum_{i=1}^{\min(N,s)} \sigma_i (\Phi(X))^p \right)^{1/p}~~\txt{ for } 0<  p \leq 1.
\end{equation}
When $p=1$, the sum of the singular values is the nuclear norm. Both~\cite{fan2018non} and~\cite{Ongie2017} use a kernel representation of the features, so that the features are never computed explicitly. In~\cite{fan2018matrix}, the authors introduce the algorithm \textsc{NLMC}, which applies a quasi-Newton method to minimize the Schatten p-norm. The Schatten p-norm for $0<p\leq 1$ is nonsmooth. This has the benefit of encouraging sparsity in the singular values, but it might prevent fast convergence near a minimizer. The algorithm \textsc{VMC}, introduced in~\cite{Ongie2017}, minimizes a smooth approximation of the Schatten p-norm. It uses a kernelized version of an iterative reweighted least-squares algorithm (IRLS). The IRLS method was originally proposed in~\cite{fornasier2011low} and~\cite{mohan2012iterative} for low-rank matrix recovery and rank minimization. The IRLS framework has the advantage that it generalizes seamlessly to the kernel setting.
 In~\cite{fan2019polynomial} a truncated version of the Schatten norm is proposed, where only the smallest singular values are minimized,
\begin{equation}
\left(\sum_{i=r+1}^{\min(N,s)} \sigma_i (\Phi(X))^p \right)^{1/p}~~\txt{ for } 0<  p \leq 1,
\end{equation}
where $r = \rank(\Phi(M))$. They use the kernel trick and propose an algorithm which alternates between truncated singular value decompositions and a step of the Adam method with an additional tuning of the stepsizes.

 In~\cite{fan2020robust}, the authors propose an extension to handle outliers in the data. This is achieved by introducing a sparse matrix in the model, which absorbs the outliers.
In~\cite{ongie2021tensor}, the authors build a tensor representation of the data and apply known matrix completion techniques in the tensor space. Their algorithm, LADMC, is a simple and efficient approach for which they are able to show that the sampling requirements nearly match the information theoretic lower bounds for recovery under a union of subspace model. This is remarkable as the sampling pattern in the tensor space is not random, and low-rank recovery results do not apply directly. Note that the approach in~\cite{ongie2021tensor} is only applicable to matrix completion problems, not matrix sensing.

In~\cite{Fan_2019_CVPR}, a new algorithm KFMC  (kernelized factorization matrix completion) is proposed, which lends itself to online completion. In this setting, the columns of the matrix $M$ are accessible as a stream and the matrix $M$ is never stored in its entirety. They also develop a variant algorithm to deal with out of samples extensions. That is, how to complete a new column without recomputing the model. The offline formulation applies the kernel trick to 
\begin{equation}
\left\lbrace
\begin{aligned}
& \min_{X,D,Z} 
& &\fronorm{\Phi(X) - \Phi(D)Z}^2 + \alpha \fronorm{\Phi(D)}^2 + \beta \fronorm{Z}^2 \\
& 
& & X_{ij} = M_{ij}, \; (i,j) \in \Omega, \\
&&&D \in \R^{n\times r}, Z\in \R^{r\times s}, X\in \R^{n\times s}.
\end{aligned}
\right.
\end{equation}
The variable $D\in \R^{n\times r}$ aims to find $r$ points in $\R^n$ so that their features will form a basis for $\Phi(M)$ in the Hilbert space. The last two terms in the objective are added as regularizers to improve the practical performances of the algorithm, as is often done in low-rank matrix completion~\cite{davenport2016overview}. 
In the method DMF, proposed in~\cite{fan2018matrix}, the embedding is replaced by a deep-structure neural network who is trained to minimize the reconstruction error for the observable entries of $M$. This work showcases the applicability and performance of nonlinear matrix completion with numerous examples including image inpainting and collaborative filtering problems. For data drawn from multiple subspaces, \cite{Fan18accelerated} proposes a sparse factorization where each subspace is represented in a low rank decomposition. They solve this model with an algorithm in the spirit of PALM~\cite{Bolte2013}. They are able to show sampling bounds to guarantee recovery.

\paragraph{Contribution and outline of the paper} 

In Section~\ref{sec:features} we describe the approach taken to recover high-rank matrices. It consists in using a feature map (or kernel) that exploits the nonlinear structure present in the matrix.
This is applied to data which follows algebraic variety models or grouped in clusters. For these, we respectively use the monomial kernel and the Gaussian kernel. We demonstrate that the Gaussian kernel can be used to perform clustering with missing data, which expands the use cases of nonlinear matrix recovery. 

In Section~\ref{sec:formulation},  we propose to use a new formulation for nonlinear matrix recovery. We use the feature map to write the recovery problem as a constrained nonconvex optimization problem on the Grassmann manifold. This extends the residual proposed in~\cite{eftekhari2019streaming} in the context of low-rank matrix completion to the nonlinear case.  

We propose to use Riemannian optimization methods to solve the recovery problem, which is new in the context of nonlinear matrix recovery. Riemannian optimization, as described in Section~\ref{sec:riemannian_opti}, provides a framework to design algorithms for problems with smooth constraints. This allows to seamlessly choose between standardized first- and second-order methods. The use of second-order methods allows to recover high-rank matrices up to high accuracy if desired. 

Section~\ref{sec:am_definition} presents an alternating minimization algorithm to solve the recovery problem. First- and second-order variants of the alternating minimization are discussed. We prove global convergence of the algorithm to first-order stationary points in Section~\ref{sec:am_convergence} and give a global complexity rate to achieve an arbitrary accuracy on the gradient norm from an arbitrary initial guess. In Section~\ref{sec:KL}, we also show convergence of the sequence of iterates to a unique limit point using the Kurdyka-Lojasiewicz property. 
Our alternating minimization method is a similar approach to the method proposed in~\cite{fan2019polynomial}. We provide extensive convergence analysis, which was not done in~\cite{fan2019polynomial}.
 
Section~\ref{sec:framework} summarizes the applications and algorithms covered in this paper with a framework to solve nonlinear matrix completion. 

We conclude with an extensive set of numerical experiments that compare the performances of the optimization approaches and the quality of the solutions that can be obtained (Section~\ref{sec:numerics}). We discuss the influence of the complexity of the data and the role of model parameters on the recovery. Moreover, we showcase that our approach is very efficient at clustering data with missing information. 
\paragraph{Notations}
Throughout the paper we use a notation consistent with~\cite{Absil2008} for the derivative of a function $f$ defined on a Riemannian manifold.  The unconstrained gradient of a function $f$ is written $\nabla f$, when the domain of $f$ is extended to an embedding Euclidean space. Conversely, we use $\grad f$ for the Riemannian gradient of $f$ defined over a Riemannian manifold. For matrices $A,B \in \Rns$, $\inner{A}{B} = \trace(A\transpose B)$ is the canonical inner product, $\range(A)$ is the column space of $A$, $\Null(A)$ is the null space of $A$. The identity matrix of size $n$ is denoted by $\I_n$ and $\Id$ is the identity operator.

\section{The feature map}
\label{sec:features}
As mentioned, our approach uses an embedding of the original matrix in a space of features, in the spirit of~\cite{fan2018non}\cite{Ongie2017}. Through the case studies below (\ref{example:variety},~\ref{example:uos} and~\ref{example:clusters}), we describe the embeddings that we use and some of the data structures to which they apply. 
\begin{assumption}
The feature map $\varphi$ is chosen such that $\Phi(M)$ is low rank. In addition, $\Phi(X)$ should be high rank if $X$ does not exhibit the same geometrical structure as $M$. 
\end{assumption}
The goal is to find an embedding that reveals the nonlinear relation between the points $m_i$, the columns of $M$. In~\cite{Ongie2017} the authors use the polynomial features for data sets represented by algebraic varieties. 

\begin{testcase}[Algebraic varieties~\cite{cox1994ideals}]
\label{example:variety}
Let $\R[x]$ be the set of real valued polynomials over $\R^n$. A real (affine) algebraic variety is defined as the zero set of a system of polynomials $P \subset \R[x]$:
\begin{equation}
V(P) = \lbrace x\in \R^{n}: p(x) = 0 \txt{ for all } p \in P\rbrace.
\end{equation}
We say that the matrix $M $ follows an algebraic variety model if every column of $M$ belongs to the same algebraic variety. 
\end{testcase}
Let 
\begin{equation}
N(n,d) = \begin{pmatrix}
n+d\\
n
\end{pmatrix}, 
\label{eq:N_features}
\end{equation}
which reads $n+d$ choose $n$, the number of monomials of degree $d$ or less that can be formed with $n$ variables. The monomial features $\varphi_d$ for some degree $d$ are defined as 
\begin{equation}
\varphi_d\colon \R^n \to \R^{N(n,d)} \colon \varphi_d(x) = \begin{pmatrix}
x^{\bm\alpha^1}\\
x^{\bm\alpha^2}\\
\vdots \\
x^{\bm\alpha^{N(n,d)}}\\
\end{pmatrix}
\label{eq:monomial-features}
\end{equation}
 where, for $i = 1, 2,\dots,N(n,d)$, the exponent $\bm{\alpha}^i = (\alpha^i_1,\alpha^i_2,\dots, \alpha^i_n)$ is a multi-index of non-negative integers; so that 
 $x^{\bm \alpha^i} := x_1^{\alpha^i_1} x_2^{\alpha^i_2}\dots x_n^{\alpha^i_n}$ and  $\alpha^i_1 + \alpha^i_2 + \dots + \alpha^i_n \leq d$.
  The dimension of the feature space $N(n,d)$ increases exponentially in $d$. Therefore, a kernel implementation is usually used in practice for moderate and large dimensions, or more precisely, whenever $s\leq N(n,d)$. The monomial kernel of degree $d$ is defined for any $X,Y\in \R^{n\times s}$ as
\begin{equation}
\K_d(X,Y) = (X^\top Y + c  \mathbf{1}_{s\times s} )^{ \odot d},
\label{eq:mono_kernel}
\end{equation}
where the value $c\in \R$ is a parameter of the kernel, $\mathbf{1}_{s\times s} $ is a square matrix of size $s$ full of ones and $\odot $ is an entry-wise exponent. If the equations describing the variety are known to be homogeneous, one can set $c=0$. Note that the monomial kernel in~\eqref{eq:mono_kernel} is not exactly the kernel associated with the monomial features in~\eqref{eq:monomial-features}. Instead, $\K_d(X,X) = \tilde{\Phi}_d(X)^\top \tilde{\Phi}_d(X)$ for a map of monomials $\tilde{\Phi}_d$ that has non-unitary coefficients given by the multinomial theorem. For $x,y\in \Rn$, we have
\begin{align}
k_d(x,y) &= (x\transpose y +c)^d = (x_1 y_1 + \cdots + x_n y_n +c)^d \\
		&= \sum_{\alpha^i_1 + \alpha^i_2 + \dots + \alpha^i_{n+1} = d} \dfrac{d!}{\alpha^i_1! \alpha^i_2! \dots  \alpha^i_{n+1} !} (x_1y_1)^{\alpha^i_1} \cdots (x_n y_n)^{\alpha^i_n} c^{\alpha^i_{n+1}}\\
		&= \sum_{\alpha^i_1 + \alpha^i_2 + \dots + \alpha^i_{n+1} = d} \left( \dfrac{\sqrt{d!}x_1^{\alpha^i_1} \cdots x_n ^{\alpha^i_n} \sqrt{c^{\alpha^i_{n+1}}}}{\sqrt{\alpha^i_1! \alpha^i_2! \dots  \alpha^i_{n+1} !}} \right)\left(  \dfrac{\sqrt{d!}y_1^{\alpha^i_1} \cdots y_n ^{\alpha^i_n} \sqrt{c^{\alpha^i_{n+1}}}}{\sqrt{\alpha^i_1! \alpha^i_2! \dots  \alpha^i_{n+1} !}}   \right).
\end{align}
It follows that $k_d(x,y) = \inner{\tilde \varphi_d(x)}{\tilde \varphi_d(y)}$ for a map $\tilde{\varphi}_d\colon \Rn \to \R^{N(n,d)}$ such that the entries of $\tilde{\varphi}_d(x)$ are of the form 
\begin{equation}
\sqrt{d!}	x_1^{\alpha^i_1} x_2^{\alpha^i_2}\dots x_n^{\alpha^i_n}\sqrt{c^{\alpha^i_{n+1}}}/\sqrt{\alpha^i_1! \cdots \alpha^i_n! \alpha^i_{n+1}!}
\end{equation} for some natural numbers $\alpha^i_1 + \alpha^i_2 + \dots + \alpha^i_{n+1} = d$. The meaningful consequence is that the kernel $\K_d$ corresponds to features $\tilde\varphi_d$ which form a basis of the set of polynomials in $n$ variables of degree at most $d$. Therefore, $\Phi_d(X)$ and $\tilde \Phi_d(X)$ have the same rank as $\K_d(X,X)$ by virtue of $\K_d(X,X) = \tilde{\Phi}_d(X)^\top \tilde{\Phi}_d(X)$.

In~\cite{Ongie2017}, the authors argue why using the monomial embedding is appropriate when the points $m_i$ belong to an algebraic variety. Suppose the variety $V(P)\subset \R^n$ is defined by the set of polynomials $P=\lbrace p_1,\dots,p_k\rbrace$ where each $p_i$ is at most of degree $d$. Then the columns of $X$ belong to the variety $V(P)$ if and only if there exists $C\in \R^{N\times k}$ such that $\Phi_d(X)^\top C=0$, where the columns of $C$ define the coefficients of the polynomials $p_i$ in the monomial basis. This implies that $\rank(\Phi_d(X)) \leq \min(N-k,s)$. This justifies that $\Phi_d(X)$ is rank deficient when there are sufficiently many data points such that $s\geq N-k$, and $X$ follows an algebraic variety model. The second case study below presents a union of subspaces as a particular type of algebraic variety.

\begin{testcase}[Union of subspaces]
\label{example:uos}
Given two affine subspaces $\mathcal{S}_1$, $\mathcal{S}_2\subset \R^n$ of dimension $r_1$ and $r_2$ respectively,  we can write $\mathcal{S}_1 =\lbrace x:q_i(x)=0 \txt{ for } i=1,\dots,n-r_1\rbrace$ 
and $\mathcal{S}_2 =\lbrace x:p_j(x)=0 \txt{ for } j=1,\dots, n-r_2\rbrace$ 
where the $q_i$ and $p_j$ are affine functions. The union $ \mathcal{S}_1 \cup \mathcal{S}_2$ can be expressed as the set where all possible products $q_i(x)p_j(x)$ vanish.  Therefore, $ \mathcal{S}_1 \cup \mathcal{S}_2$ is the solution of a system of $(n-r_1)(n-r_2)$ quadratic polynomial equations. Similarly, a union of $k$ affine subspaces of dimensions $r_1 , \dots, r_k$ is a variety described by a system of $\Pi_{i=1}^k (n-r_i)$ polynomial equations of degree $k$.
\label{test:uos}
\end{testcase}

\begin{proposition}[Rank of monomial features~\cite{Ongie2017}]
If the columns of a matrix $X\in \R^{n\times s}$ belong to a union of $p$ affine subspaces of dimension at most $\tilde r$, then for any degree $d\geq 1$, the matrix $\Phi_d(X) \in \R^{N(n,d)\times s}$ of monomial features, with $N(n,d)$ the dimension of the features space defined in equation~\eqref{eq:N_features}, satisfies 
\begin{equation}
\rank \Phi_d(X) \leq p\begin{pmatrix}
\tilde r+d\\
d
\end{pmatrix}.
\end{equation}
\label{prop:rank_phi_uos}
\end{proposition}
In practice, choosing the degree $d$ of the monomial kernel is a tricky task. In Section~\ref{sec:numerics}, we discuss the practical choice of this degree and how it impacts the rank of the feature matrix and the possibility to recover $M$. Previous works using the monomials kernel to recover high-rank matrices all restricted themselves to degrees two or three~\cite{fan2018non, Ongie2017}. Using a polynomial embedding of large degree would seem helpful to capture all the nonlinearity in some data sets. Unfortunately, increasing the degree will grow the dimensions of the optimization problem exponentially. Indeed, the dimension $N(n,d)$ blows up with $d$ for even moderate values of $n$ and the number of data points required is at least $N(n,d)-k$ where $k$ is the number of polynomial equations that define the variety.

We now define the Gaussian kernel which will be used when the columns of the matrix $M$ are grouped in several clusters. 
\begin{testcase}[Clusters]
\label{example:clusters}
 For $X, Y\in \R^{n\times s}$, the entry $(i,j)$ of the Gaussian kernel $\K^G: \R^{n\times s} \times \R^{n\times s} \to \R^{s\times s}$ is defined as
\begin{equation}
\K_{ij}^G(X,Y) = e^{ -\dfrac{\|x_i - y_j\|^2_2}{2\sigma^2}},
\label{eq:gaussian_kernel}
\end{equation}
where $\sigma>0$ is the width of the kernel. The Gaussian kernel acts as a proximity measure. For two columns of $X$, labelled $x_i$ and $x_j$, we observe that $x_i$ being close to $x_j$ gives $\K^G_{ij}(X,X) \approx 1$ and if $x_i$ is far from $x_j$ then $\K^G_{ij}(X,X) \approx 0$. Therefore, the rank of the Gaussian kernel approximately coincides with the number of clusters in $X$. More precisely, one can show that the singular values whose index exceeds the number of clusters decay rapidly~\cite{Singer2006}. The value of $\sigma$ should be chosen appropriately depending on the size of the clusters.
\end{testcase}
In Figure~\ref{fig:clusters-demo}, we present a small data set of 100 data points divided in two clusters in $\R^2$ with the singular values of the Gaussian kernel (in log-scale). We see that the two largest singular values are much greater than the third one, and that the following singular values decrease at an approximately exponential rate. Therefore, the Gaussian kernel is near a low rank matrix for clustered data, which will allow us to complete such data sets from partial measurements.
In~\cite{fan2018matrix}, the Gaussian kernel was also used effectively on image inpainting and denoising problems. 
\begin{figure}[!htb]
\centering
\subfigure{
\includegraphics[width = 0.45\textwidth]{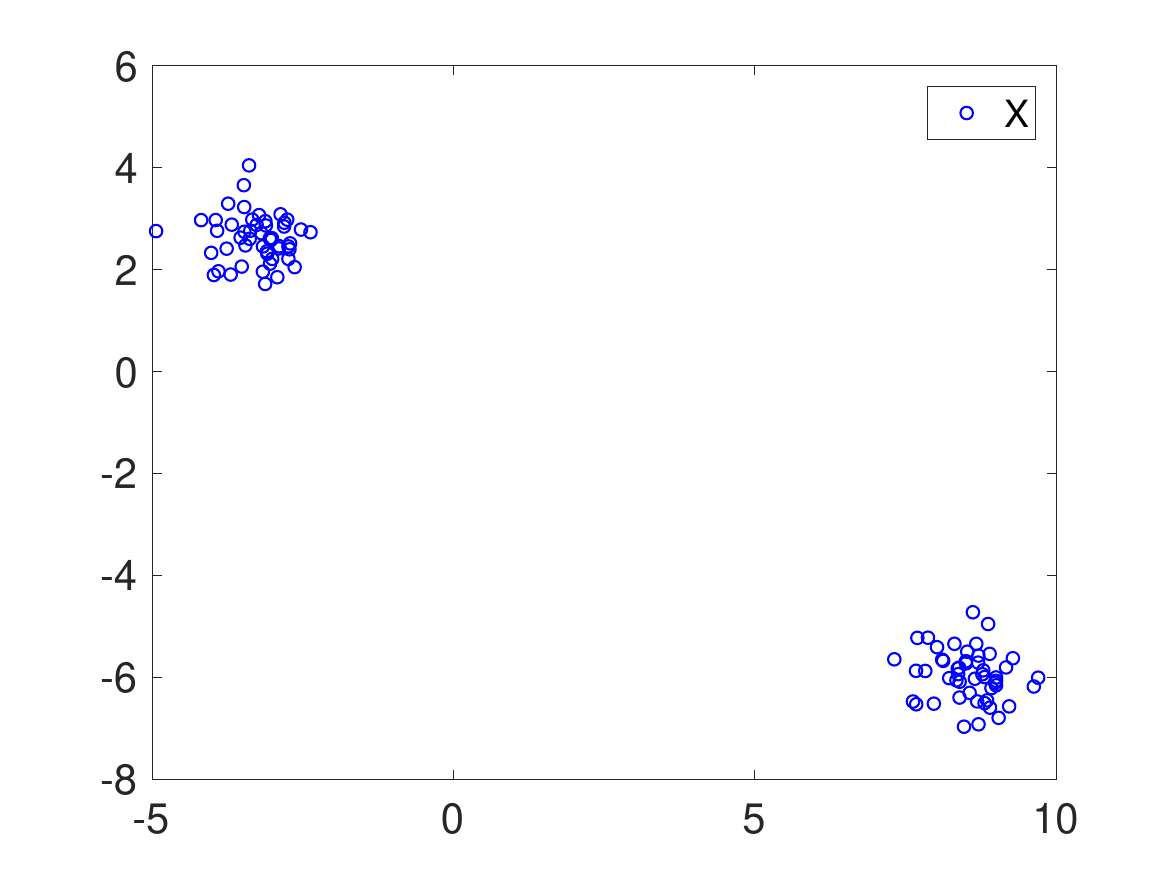}
}
\quad
\subfigure{
\includegraphics[width = 0.45\textwidth]{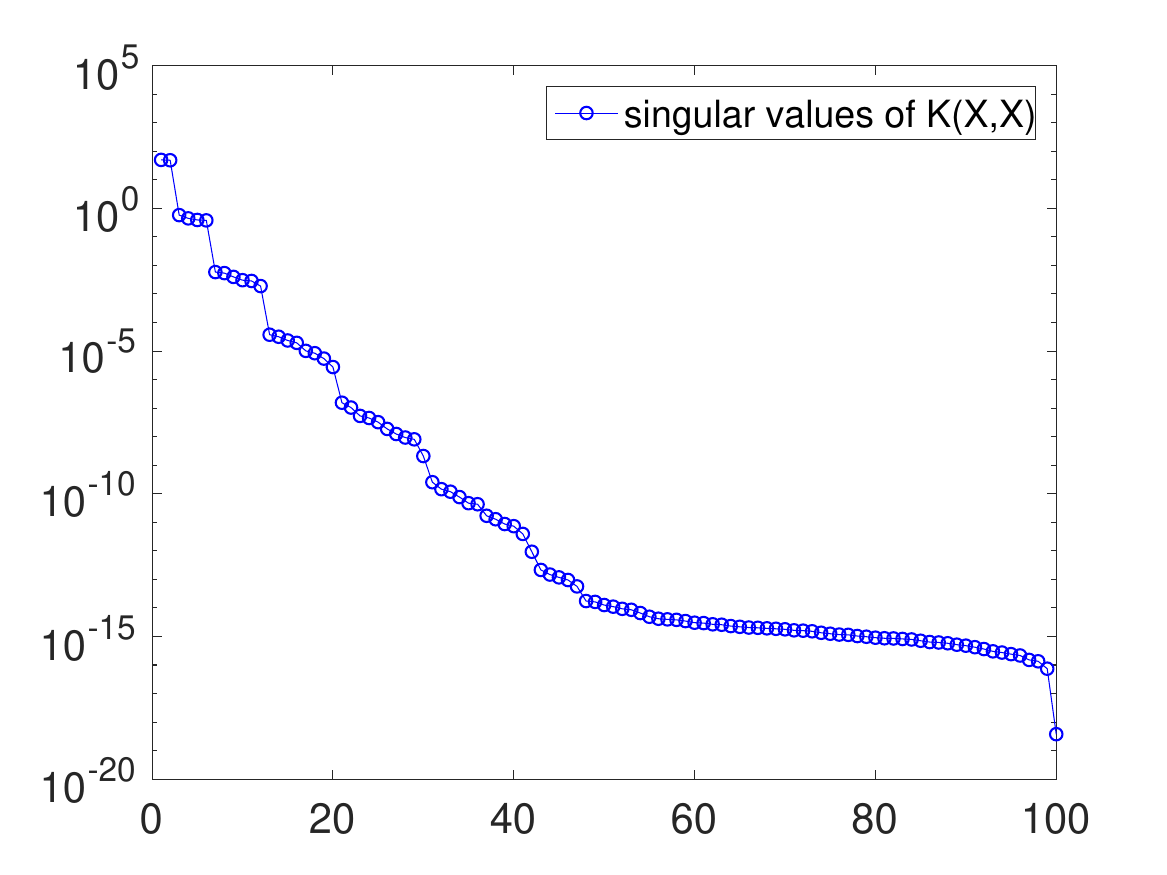}
}
\caption{Clustered data and the singular values of the Gaussian kernel in log-scale.}
\label{fig:clusters-demo}
\end{figure}

\FloatBarrier

\section{Nonlinear matrix recovery as an optimization problem}
\label{sec:formulation}
\paragraph{Noiseless measurements case}
Considering a noiseless measurements case, we would like to minimize the rank of the feature matrix, as in Equation~\eqref{eq:min_rank_phi}. This is unfortunately known to be intractable, even in the case where $M$ is low-rank~\cite{fazel2004rank}. We have to resort to approximations of this problem. The second difficulty is the nonlinearity of $\Phi$.

As a nonconvex approximation to~\eqref{eq:min_rank_phi}, we consider the formulation in~\cite{eftekhari2019streaming} for low-rank problems, and extend it to the nonlinear case. Assuming that $\Phi(M)$ has rank $r$ leads to the following formulation
\begin{equation}
\left\lbrace
\begin{aligned}
& \underset{X,\U}{\min}
& & f(X,\U) := \fronorm{ \Phi(X) - \P_{\U}\Phi(X)}^2 \\
& 
& & \U\in \Grass(N,r) \\
&&& \mathcal{A}(X)=b,
\end{aligned}
\right.
\label{eq:p1}
\end{equation}
where  $\Grass(N,r)$ is the Grassmann manifold, the set of all subspaces of dimension $r$ in $\R^{N}$, $\P_{\U}$ is the orthogonal projection on the subspace $\U$ and $\fronorm{.}$ denotes the Frobenius norm. The linear measurements, $\A(X)=b$, are defined in equation~\eqref{eq:measurements}.  Given $U\in \R^{N\times r}$ such that $\range(U) = \U$ and $U^\top U=I_r$, the projection is given by $\Prm_\U = UU^\top$. In~\eqref{eq:p1} ,the objective function is expected to be nonconvex but smooth for practical choices of $\varphi$, such as case study~\ref{example:variety},~\ref{example:clusters}. 
If the variable $\U$ is additionally constrained to be the range of the $r$ leading singular vectors of $\Phi(X)$, the cost function becomes 
$\sum_{i = r+1}^{\min(N,s)} \sigma_i(\Phi(X))^2$. The advantage of the formulation in \eqref{eq:p1} is that is it straightforward to express it as a finite sum of $s$ terms over all the data points. This allows to use stochastic sub-sampling algorithms that scale better to matrices with many columns (large $s$). 

The advantage of using the Grassmann manifold, which is a quotient space, instead of the Stiefel manifold of orthogonal matrices $\St(N,r) := \lbrace U\in \R^{N\times r}: U^\top U = \I_r\rbrace$ is that due to the invariance of the cost function with respect to the matrix that represents the subspace $\U$, local optimizers cannot possibly be isolated in a formulation over $\St(N,r)$. Therefore, the fast local convergence rates of some second-order algorithms might not apply on $\St(N,r)$, while they would apply on the quotient manifold.

Consider $U\p$ a basis of $\U\p$, the orthogonal complement of $\U$ in $\R^N$. The variable $\U\p\in \Grass(N,N-r)$ has a nice interpretation since $U\p$ spans $\Null(\Phi(X)^\top)$ when $f(X,\U) =0$. In the case of algebraic varieties (case study~\ref{example:variety}), $U\p$ gives the coefficients of the polynomials defining the variety in the basis given by $\Phi$. Recovering the equations of the variety is of interest in some applications~\cite{breiding2018learning,Goyens2020smoothing}.

\paragraph{Noisy measurements case}
When the measurements are known to be noisy, i.e.  $\A(M)=b + \eta$ for some noise $\eta\in \R^m$, the measurement constraint can be lifted into the cost function as a penalty.  This gives
\begin{equation}
\left\lbrace
\begin{aligned}
& \underset{X,\U}{\min}
& &  f_\lambda(X,\U):=  \fronorm{ \Phi(X) - \P_{\U}\Phi(X)}^2  + \lambda \norm{ \mathcal{A}(X) -b}^2_2\\
& 
& & \U\in \Grass(N,r), \\
\end{aligned}
\right.
\label{eq:p1_noise}
\end{equation}
where the parameter $\lambda>0$ has to be adjusted.  This allows to satisfy the noisy measurements approximately. 

\subsection{Kernel representation of the features}
When working with a kernel instead of a feature map, as shown in~\eqref{eq:def_kernel}, we want to find a cost function equivalent to that of~\eqref{eq:p1} which uses the kernel instead of the feature map. We find that they are related in the following way.
 \begin{proposition}
Given a feature map $\Phi$ and the associated kernel $\K : X\times X \mapsto \Phi(X)^\top \Phi(X)$, for $\W \in \Grass(s,r)$ we have
\begin{equation}
 \fronorm{\Phi (X)^\top - \P_{\W}\Phi (X)^\top}^2 = \trace\Big( \K(X,X) - \P_{\W}\K(X,X) \Big).
 \label{eq:trace-kernel}
\end{equation} 
\end{proposition}
\begin{proof}
We write $\P_{\W^{ \bot}} =  \I_{s\times s} - \P_{\W}$ and find
\begin{align}
\nonumber
 \trace\Big( \P_{\W^{ \bot}}\K(X,X) \Big) &=  \trace\Big( \P_{\W^{ \bot}}\Phi(X)\transpose \Phi(X) \Big)\\\nonumber
& = \trace\Big(\Phi (X)  \P_{\W^{ \bot}}\Phi(X)^\top \Big) \\ 
 &=   \trace\Big(\Phi (X)  \left(\P_{\W^{ \bot}}\right)\transpose \P_{\W^{ \bot}}\Phi(X)\transpose \Big)\\\nonumber
&  =  \trace\Big((\P_{\W\p}\Phi(X)^\top)^\top   \P_{\W\p}\Phi(X)\transpose \Big)\\
 & = \fronorm{\P_{\W\p}\Phi(X)^\top}^2.\nonumber
\end{align}
\end{proof}
Using the kernel formula~\eqref{eq:trace-kernel} corresponds to finding a subspace $\W$ of dimension $r$ in the row space of $\Phi(X)$. When a kernel is used for the embedding, the following optimization problem is solved, 
\begin{equation}
\left\lbrace
\begin{aligned}
& \underset{\W,X}{\min}
& & f(X,\W) := \trace\big(\K(X,X) - \P_{\W} \K(X,X)\big) \\
& 
& & \W\in \Grass(s,r) \\
&&& \mathcal{A}(X)=b.
\end{aligned}
\right.
\label{eq:p1kernel}
\end{equation}
  Replacing the features $\Phi$ by the corresponding kernel becomes beneficial when the dimension $N$ of the features is larger than the number of points $s$ and when a convenient formula is available to compute the kernel and its derivatives. For example, in the case of clusters (case study~\ref{example:clusters}), the features exist implicitly in an infinite dimensional space and we use the Gaussian kernel to represent them.
  
In the upcoming sections, we usually describe the algorithms and their properties using the notation of problem~\eqref{eq:p1} with a feature map $\Phi$ and cost function $f(X,\U)$. Unless specified otherwise, the developments also apply to problem~\eqref{eq:p1kernel} and the use of a kernel. 

\section{Riemannian optimization algorithms}
\label{sec:riemannian_opti}
Riemannian optimization methods provide an elegant and efficient way to solve optimization problems with smooth nonlinear constraints. The field of Riemannian optimization has  rapidly developed over the past two decades. In particular,
Riemannian optimization methods have proved very efficient in low-rank matrix completion~\cite{boumal2015low, Vandereycken2013}. In this section, we investigate the use of Riemannian optimization methods to solve~\eqref{eq:p1} or~\eqref{eq:p1kernel}. For an overview of optimization algorithms on Riemannian manifolds, see~\cite{Absil2008,boumal2022intromanifolds}. In order to formally express~\eqref{eq:p1} as a Riemannian optimization problem, we define a notation for the affine subspace that represents the measurements on the matrix $M$,
\begin{equation}
\LAb = \{X\in \R^{n\times s}: \A(X)=b\}.
\label{eq:Lab}
\end{equation}
We form the product manifold \begin{equation}
\M = \LAb \times \Grass(N,r),
\label{eq:M}
\end{equation}
 so that Problem~\eqref{eq:p1} can be viewed as  the unconstrained minimisation of a smooth cost function defined on the manifold $\M$,
\begin{equation}
\left\lbrace
\begin{aligned}
& \underset{(X,\U)}{\min}
&  &  \fronorm{ \Phi(X) - \P_{\U}\Phi(X)}^2 \\
& & &(X,\U) \in \M.
\end{aligned}
\right.
\label{eq:p1riem}
\end{equation}
We introduce the notation $z = (X,\U) \in \M$ to denote both variables that appear in the optimization problem. Riemannian optimization algorithms are feasible methods that iteratively exploit the local geometry of the feasible set. Analogously to unconstrained optimization in Euclidean spaces, each iteration of a Riemannian optimization algorithm uses derivatives to build a model that locally approximates the cost function. This model is then fully or approximately minimized. Most commonly, the model uses the gradient and possibly the Hessian or an approximation of it. This yields respectively a first- or second-order method. Riemannian optimization follows these principles, only the model is defined on a local linearisation of the manifold, namely, the tangent space. The Riemannian gradient, written $\grad f$ is a vector belonging to the tangent space of the manifold. The Riemannian Hessian, written $\Hess f$ is a symmetric operator on that tangent space.
In Appendix~\ref{sec:appendix_derivatives}, we show how to compute the Euclidean gradient and Hessian for the cost function of~\eqref{eq:p1kernel} in the case of the kernels presented in Section~\ref{sec:features} (monomial kernel, Gaussian kernel). Then, to find their Riemannian counterparts, 
$\nabla f$ and $\nabla^2 f$ are projected onto the tangent space of $\M$ using the tools defined later on in this section. When exploring the tangent space, it is necessary to have a tool that allows to travel on the manifold in a direction prescribed by a tangent vector. This operation is called a retraction~\cite[Def. 4.1.1]{Absil2008}. 
\begin{definition}[Retraction]
A retraction on a manifold $\M$ is a smooth mapping $Retr$ from the tangent bundle $\Trm\M$ to $\M$ with the following properties. Let $\Retr_z\colon \Trm_z\M \to \M$ denote the restriction of $\Retr$ to $\Trm_z \M$. 
\begin{itemize}
\item[(i)] $\Retr_z(0_z) = z$, where $0_z$ is the zero vector in $\Trm_z\M$;
\item[(ii)] the differential of $\Retr_z$ at $0_z$, $\D \Retr_z(0_z)$, is the identity map.
\end{itemize}
The retraction curves $t\mapsto \Retr_z(t\eta)$ agree up to first order with geodesics passing through $z$ with velocity $\eta$, around $t=0$. 
\end{definition}
Let us detail the tools necessary to use Riemannian optimization methods on the two manifolds that compose our search space $\M$. Note that the Cartesian product of two Riemannian manifolds is a Riemannian manifold. The geometry of $\LAb$ is rather trivial because the manifold is affine. It must nonetheless be implemented so we will describe how to handle this constraint in a Riemannian way. More generally, this gives a straightforward way to deal with affine equality constraints. 
\paragraph{Measurement subspace}
At any point $X\in \LAb$, the tangent space is the null space of $\A$,
\begin{equation}
\Trm_X \LAb = \Null(\A) = \{ \Delta \in \R^{n\times s}: \A(\Delta)=0\}.
\end{equation}
Since the tangent space does not depend on $X$, we write $\Trm \LAb$. This tangent space inherits an inner product from the embedding space $\R^{n\times s}$,
\begin{equation}
\langle \Delta_1 , \Delta_2 \rangle = \trace(\Delta_1^\top \Delta_2)~~\txt{for all }\Delta_1,\Delta_2 \in \Trm\LAb.
\end{equation}
The Riemannian gradient is the orthogonal projection of the Euclidean gradient onto  the tangent space $\Trm\LAb$. 
From the fundamental theorem of algebra, $\Null(\A) = \range(\A^\top)^\perp$. Therefore we can express $\Prm_{\Null(\A)} = \Id - \Prm_{\range(\A^\top)}$. 

The application $\A $ is represented by a flat matrix $A \in \R^{m \times ns}$ such that $\A(X)= AX(:)\in \R^m$, where $X(:)$ is a vector of length $ns$ made of the columns of $X$ taken from left to right and stacked on top of each other. Visually this gives,
\begin{equation}
\A(X)=\begin{pmatrix}
\langle A_1,X\rangle\\
\langle A_2,X\rangle\\
\vdots \\
\langle A_m,X\rangle\\
\end{pmatrix} = \begin{pmatrix}
A_1(:)^\top\\
A_2(:)^\top\\
\vdots\\
A_m(:)^\top\\
\end{pmatrix}
X(:) =:  AX(:).
\label{eq:Asensing}
\end{equation} 
The tall matrix $A^\top$ represents the linear application $\Acal^\top$. Viewing $A^\top$ as the matrix of an overdetermined system of linear equations convinces us that $\Prm_{\range(\A^\top)} = A^\top (AA^\top)^{-1}A$. Equivalently, if $Q \in \R^{m\times m}$ is an orthogonal basis for $\range(\A^\top)$ (which can be obtained by a reduced QR factorization of $A^\top$), we can apply $\Prm_{\range(\A^\top)} = Q Q^\top.$  The projection is given by 
\begin{equation}
 \Prm_{\Null(\A)} = \I_{ns} - A^\top(AA^\top)^{-1}A =  \I_{ns} - Q Q^\top .
 \label{eq:PnullA} 
\end{equation} 
Therefore,
\begin{equation}
\Prm_{\Trm \LAb}: \R^{n\times s} \to  \Trm \LAb: \Prm_{\Trm \LAb}(\Delta)  = (\I_{ns} - Q Q^\top) \Delta 
\label{eq:PTLAb}
\end{equation}
where $Q \in \R^{m\times m}$ is an orthogonal basis for $\range(\A^\top)$. 
In the case of matrix completion, the operator $\A$ selects the known entries of $M$. The description of the feasible subspace is simplified. We write $\calL_{\Omega,b} = \{ X: X_{ij} = M_{ij}, ij\in \Omega\}$ to make explicit that the measurements correspond to matrix completion. The tangent space is  $\Trm \mathrm{L}_{\Omega,b} = \{ \Delta: \Delta_{ij} =0 \txt{ for } ij \in \Omega\}$ and the projection onto $\Trm \mathrm{L}_{\Omega,b}$ simply amounts to setting the entries outside of $\Omega$ to zero, 
\[
\Prm_{\Trm \calL_{\Omega,b}} (\Delta) = \left\lbrace 
\begin{aligned}
&0 &~\txt{for } ij \in  \Omega\\
&\Delta_{ij} &~\txt{for } ij \notin \Omega.
\end{aligned} \right. 
\]
The natural retraction on $\LAb$ for $X\in \LAb$ and $\Delta \in \Trm \LAb$ is given by
\begin{equation}
\Retr_X: \Trm \LAb \to  \LAb: \Retr_X(\Delta) = X+\Delta,
\end{equation}
because the manifold is flat. These tools are also needed for the Grassmann manifold and we follow the representation given in~\cite{boumal2015low}.
\paragraph{Grassmann  manifold}
The Grassmann manifold, written $\Grass(N,r)$, is the set of all linear subspaces of dimension $r$ in $\R^N$. A point $\U\in \Grass(N,r)$ is represented by a full-rank matrix $U\in \R^{N\times r}$ such that $\range(U) = \U$. For any orthogonal matrix $Y\in \R^{r\times r}$, the matrix $UY$ is also a valid representation of $\U$, since $\range(UY) = \range(U)$. The set of matrices with orthonormal columns is defined as the Stiefel manifold, $\St(N,r) = \{ U \in \R^{N\times r}: U^\top U = \I_r\}$,  and the orthogonal group is defined as $O(r) = \{ Y\in \R^{r\times r}: Y^\top Y = \I_r\}$. The orthogonal group induces an equivalence relation on the Stiefel manifold, where any two matrices are equivalent if they have the same column space. In this regard, the Grassmann is a quotient space
\begin{equation}
\Grass(N,r) = \St(N,r)/\mathrm{O}(r),
\end{equation}
and each equivalence class consists of all matrices with the same span. The tangent space to the Stiefel manifold at $U\in \St(N,r)$ has the form 
\begin{align}
\Trm_U \St(N,r) = \left\lbrace U Z + U^\perp B: Z\in \mathrm{Skew}(r), B\in \R^{(N-r)\times r}\right\rbrace,
\end{align}
where $\mathrm{Skew}(r)$ is the set of skew-symmetric matrices of size $r$~\cite[Section 7.3]{boumal2022intromanifolds}. The equivalence class of $U\in \St(N,r)$—seen as a submanifold of $\St(N,r)$—has a tangent space at $U$, which is called the vertical space $\Vrm_U \St(N,r) = \left\lbrace UZ : Z\in \mathrm{Skew}(r)\right\rbrace \subseteq \Trm_U \St(N,r)$. 
The orthogonal complement of $\Vrm_U\St(N,r)$ in $\Trm_U\St(N,r)$ is called the horizontal space and is given by
\begin{equation}
\Hrm_U \St(N,r) = \left\lbrace U^\perp B: B\in \R^{(N-r)\times r}\right\rbrace = \range(U\p) \subseteq \Trm_U\St(N,r).
\end{equation}
As is common in differential geometry, the tangent space to $\Grass(N,r)$ at $\U$ is represented by the horizontal space $\Hrm_U\St(N,r)$, that is, $\Hrm_U\St(N,r) \simeq \Trm_\U \Grass(N,r)$. Any tangent vector $H_\U \in \Trm_\U \Grass(N,r)$ is represented by an horizontal vector $H_U \in \Hrm_U\St(N,r) $ called the horizontal lift of $H_\U$ at $U$. A thorough treatment of quotient manifolds, such as the Grassmann, and their usage in optimization can be found in~\cite{Absil2008,boumal2022intromanifolds}.
The projection onto the horizontal space is given by
\begin{equation}
\mathrm{Proj}_{\Hrm_U \St(N,r)}\colon \Trm_U \St(N,r) \to \Hrm_U\St(N,r)\colon H \mapsto \left(\Id - U U\transpose\right) H.
\label{eq:PTGr}
\end{equation}
The horizontal space is equipped with the usual inner product
\begin{align}
\langle H_1, H_2 \rangle_U &= \trace(H_1^\top H_2) &&\text{ for all } H_1, H_2 \in \Hrm_U  \St(N,r). 
\end{align}
The norm of a tangent vector to the Grassmann is given by the norm of its horizontal lift. Hence we understand the notation $\fronorm{H_\U}$ for $H_\U \in \Trm_\U \Grass(N,r)$ as $\fronorm{H_\U} = \fronorm{H_U}$, where $H_U$ is the horizontal lift of $H_\U$.
Let us call qf, the mapping that sends a matrix to the $Q$ factor of its (reduced) $QR$ decomposition with $Q\in \St(N,r)$ and $R$ upper triangular with positive diagonal entries. To move away from $\Ucal\in \Grass(N,r)$ in the direction $H \in \Hrm_U\St(N,r)$, we use the following retraction
\begin{equation}
\Retr_\U \colon \Hrm_U\St(N,r) \to \Grass(N,r)\colon H \mapsto \range\left(\txt{qf}(U + H)\right).
\end{equation}

We are now in a position to present the Riemannian trust-region algorithm~\cite[Ch.7]{Absil2008}. This is an extension of the classical trust-region methods~\cite{conn2000trust} to Riemannian manifolds.

\paragraph{Riemannian trust-region (RTR)}
 At each iterate, RTR builds a local model of the function.  The method sequentially minimizes this model under a ball constraint that prevents undesirably large steps where the model does not accurately represent the function. The trust-region subproblem takes the following form around $z_k \in \M$
\begin{equation}
\min_{\substack{\eta\in \Trm_{z_k}\M\\
                   \norm{\eta}_{z_k} \leq \Delta_k}} \Hat m_{z_k}(\eta):= f(z_k) + \langle \eta , \grad f(z_k)\rangle_{z_k} + \dfrac{1}{2}\langle \eta, H_k [\eta]\rangle_{z_k},
\label{eq:TRsub}                   
\end{equation}
where $H_k\colon \Trm_{z_k}\M \to \Trm_{z_k}\M$ is a symmetric operator on $\Trm_{z_k}\M$, $\Delta_k$ is the trust-region radius and the model $\hat m_{z_k}\colon \Trm_{z_k} \M \to \R$ is a quadratic approximation of the \emph{pullback} $\hat f_{z_k}  = f\circ \Retr_{z_k} $, defined on the tangent space at $z_k\in \M$. 

\paragraph{First-order Riemannian trust-region}
When the Hessian of the cost function is expensive to compute or not available altogether, one can use a first-order model and set $H_k = 0$. This method will be very similar to gradient descent. But the trust-region is used to ensure global convergence, as opposed to a line-search. 

\paragraph{Second-order Riemannian trust-region}
When the true Hessian of the cost function is available, the classical second-order trust-region method is obtained with $H_k = \Hess f(z_k)$. It is also possible to use an approximation of the true Hessian for $H_k$.

\begin{algorithm}
\caption{Riemannian trust-region (RTR)~\cite{boumal2019global}}\label{algo:RTR}
\begin{algorithmic}[1]
\State {\bf Given: } $z_0 \in \M$ and $0< \Delta_0<\bar{\Delta}$, $\varepsilon_g>0$, $\varepsilon_H>0$ and $0<\rho'< 1/4$ 
\State \textbf{Init: } $k=0$
\While{true}
\If{$\norm{\grad f(z_k)}>\varepsilon_g $}
\State Obtain $\eta_k\in \Trm_{z_k}\M$ satisfying~\aref{assu:1st-order-decrease-rtr}
\ElsIf{$\varepsilon_H <\infty$}
\If{$\lambdamin(H_k) <-\varepsilon_H$}
\State Obtain $\eta_k\in \Trm_{z_k}\M$ satisfying~\aref{assu:2nd-order-decrease-rtr}
\Else
\State  \textbf{return } $z_k$
\EndIf
\Else
\State \textbf{return }$z_k$
\EndIf
\State $z_k^+ = \Retr_{z_k}(\eta_k)$
\State $\rho = f(z_k) - f(z_k^+) / (\hat m_{z_k}(0) - \hat m_{z_k}(\eta_k))$
\If{$\rho <1/4$}
\State $\Delta_{k+1} = \Delta_k/4$
\ElsIf{ $\rho >3/4$ \textbf{ and } $\norm{\eta_k}= \Delta_k$ }
\State $\Delta_{k+1} = \min(2\Delta_k, \bar{\Delta})$
\Else
\State $\Delta_{k+1} = \Delta_k $
\EndIf
\If{$\rho >\rho'$}
\State $z_{k+1} = z^+_k$
\EndIf
\State $k =k+1$
\EndWhile
\end{algorithmic}
\end{algorithm}

We apply RTR, as described in Algorithm~\ref{algo:RTR}, to problem~\eqref{eq:p1} or~\eqref{eq:p1kernel}. If a first-order critical point is sought, set $\varepsilon_H = \infty$. The second-order version of RTR provably converges to second-order critical points for any initialization under a weak decrease condition in the subproblems and satisfies global worst-case complexity bounds matching their unconstrained counterparts, as was shown in~\cite{boumal2019global}. The local convergence rate is quadratic for an appropriate choice of parameters~\cite[Chap.7]{Absil2008}. We note that there is no guarantee on the quality of the stationary point, due to the nonconvexity. Nonetheless, we see in Section~\ref{sec:numerics} that the method performs very well in practice for nonlinear matrix recovery. We introduce the following assumptions. 
\begin{assumption}
There exists $f^*>-\infty$ such that $f(x)\geq f^*$ for all $x\in\M$.
\label{assumption:bounded_below}
\end{assumption}
The cost functions of problems~\eqref{eq:p1} and~\eqref{eq:p1kernel} are nonnegative, so~\aref{assumption:bounded_below} is satisfied throughout this paper.
We also state a regularity assumption on the gradient and Hessian of the pullback which was introduced in~\cite{boumal2019global}. In Appendix~\ref{sec:appendix_derivatives}, we detail how these conditions relate to the smoothness of the Riemannian derivatives and discuss the practicality of these assumptions for problem~\eqref{eq:p1kernel}.

\begin{assumption}[Lipschitz gradient of the pullback]
There exists $L_g\geq 0$ such that for all $z = (X,\U)\in \M$, the pullback $\hat f_z = f\circ \Retr_{z}$ has Lipschitz continuous gradient with constant $L_g$, that is, for all $\eta \in \Trm_{z}\M$, it holds that
\begin{equation}
\left| \hat f_z(\eta) - [f(z) + \inner{\eta}{\grad f(z)}]\right| \leq \dfrac{L_g}{2}\norm{\eta}^2.
\label{eq:grad-pullback-lipschitz}
\end{equation}
\label{assu:lip-gradient-tr}
\end{assumption}

\begin{assumption}[Lipschitz Hessian of the pullback]
There exists $L_H\geq 0$ such that, for all $z = (X,\U)\in \M$, the pullback $\hat f_z = f\circ \Retr_{z}$ has Lipschitz continuous Hessian with constant $L_H$, that is, for all $\eta \in \Trm_{z}\M$, it holds that
\begin{equation}
\left| \hat f_z(\eta) - \left[f(z) + \inner{\eta}{\grad f(z)} + \dfrac{1}{2}\inner{\eta}{\nabla^2 \hat f_z(0_{z})[\eta]} \right]\right| \leq \dfrac{L_H}{6}\norm{\eta}^3.
\label{eq:hess-pullback-lipschitz}
\end{equation}
\label{assu:lip-hessian-tr}
\end{assumption}
Algorithm~\ref{algo:RTR} is flexible in that it does not specify how the subproblems are solved. We discuss the implementation of RTR in Section~\ref{sec:numerics}. For the complexity results, the following decreases in the model for first- and second-order steps are required.
\begin{assumption}
There exists $c_2>0$ such that all first-order steps $\eta_k$ satisfy 
\begin{equation}
\hat m_k(0_{z_k}) - \hat m_k(\eta_k) \geq c_2 \min\left(\Delta_k, \dfrac{\varepsilon_g}{c_0}\right)\varepsilon_g.
\end{equation}
\label{assu:1st-order-decrease-rtr}
\end{assumption}

\begin{assumption}
There exists $c_3>0$ such that all second-order steps $\eta_k$ satisfy 
\begin{equation}
\hat m_k(0_{z_k}) - \hat m_k(\eta_k) \geq c_3 \Delta_k^2\varepsilon_H.
\end{equation}
\label{assu:2nd-order-decrease-rtr}
\end{assumption}

\begin{assumption}
There exists $c_0 \geq 0$ such that, for all first-order steps, $\norm{H_k}\leq c_0$ and $H_k$ is radially linear, that is, for all $ \alpha \geq 0$ and $\eta \in \Trm_{z_k} \M$, it holds $H_k[\alpha \eta] = \alpha H_k[\eta]$.
\label{assu:rtr-H-first-order}
\end{assumption}

\begin{assumption}
There exists $c_1\geq 0$ such that, for all second-order steps,
$$\left| \inner{\eta_k}{ \left(\nabla^2 \hat{f}_{z_k} (0_{z_k}) - H_k\right) [\eta_k] } \right| \leq \dfrac{c_1 \Delta_k}{3} \norm{\eta_k}^2.$$
In addition, for all second-order steps, $H_k$ is linear and symmetric. 
\label{assu:rtr-H-second-order}
\end{assumption}

Define the following constants
\begin{align}
\lambda_g &= \dfrac{1}{4}\min\left(\dfrac{1}{c_0}, \dfrac{c_2}{L_g + c_0}\right) &\text{ and } && \lambda_H &= \dfrac{3}{4}\dfrac{c_3}{L_H + c_1}.
\end{align}
The following (sharp) worst-case bound for Riemannian trust-region was recently established. 
\begin{theorem}[Global complexity of RTR~\cite{boumal2019global}]
Under~\aref{assu:lip-gradient-tr},~\aref{assu:1st-order-decrease-rtr},~\aref{assu:rtr-H-first-order} and assuming $\varepsilon_g \leq \dfrac{\Delta_0}{\lambda_g}$, Algorithm~\ref{algo:RTR} produces an iterate $z_{N_1}$ satisfying $\norm{\grad f(z_{N_1})} \leq \varepsilon_g$ with 
\begin{equation}
N_1 \leq \bigO(1/\varepsilon^2_g).
\end{equation}
Furthermore, if $\varepsilon_H <\infty$ then under additionally~\aref{assu:lip-hessian-tr},~\aref{assu:2nd-order-decrease-rtr},~\aref{assu:rtr-H-second-order} and assuming $\varepsilon_g \leq \dfrac{c_2}{c_3}\dfrac{\lambda_H}{\lambda_g^2}$ and $\varepsilon_H \leq \dfrac{c_2}{c_3\lambda_g}$, Algorithm~\ref{algo:RTR} also produces an iterate $z_{N_2}$ satisfying $\grad f(z_{N_2})\leq \varepsilon_g$ and $\lambdamin(H_{N_2}) \geq - \varepsilon_H$ with 
\begin{equation}
N_1\leq N_2 \leq \bigO\left( \dfrac{1}{\varepsilon^2\varepsilon_H}\right).
\end{equation} 
\end{theorem}

\section{An alternating minimization algorithm}
\label{sec:am_definition}

In this section, we propose an alternating minimization algorithm to solve~\eqref{eq:p1} or~\eqref{eq:p1kernel} (Algorithm~\ref{algo:am}). This comes from the natural separation of the variables into two blocks $X$ and $\U$, yielding two distinct minimization subproblems. Alternating minimization type methods have been very popular in recent years to solve large-scale nonconvex problems~\cite{Bolte2013,wen2012solving}. This is due to their good practical performances and ease of implementation, as often one or both of the subproblems have a closed form solution. Strictly speaking, this is still a Riemannian optimization approach, as all iterates will be feasible for the constraints. But this section describes a two-block coordinate minimization, whereas the previous section was considering full block variants. 

We set the initial guess $X_0$ as any solution of the underdetermined linear system $\mathcal{A}(X)=b$ and $\U_0$ as the span of the $r$ leading singular vectors of $\Phi(X_0)$. The framework is as follows, for $k \geq 0$:\nl
With $\U_{k}$ fixed, solve
\begin{equation}
X_{k+1} = \left\lbrace
\begin{aligned}
& \underset{X\in \R^{n \times s}}{\argmin}
& & \fronorm{\Phi (X) - \P_{\U_{k}}\Phi (X)}^2 \\
& 
& & \mathcal{A}(X)=b.
\end{aligned}
\right.
\label{eq:s1}
\end{equation}
With $X_{k+1}$ fixed, solve
\begin{equation}
\U_{k+1} = \left\lbrace
\begin{aligned}
& \underset{\U}{\argmin}
& & \fronorm{\Phi (X_{k+1}) - \P_{\U}\Phi (X_{k+1})}^2 \\
& 
& &\U\in \Grass(N,r).
\end{aligned}
\right.
\label{eq:s2}
\end{equation}
\paragraph{}
This separation of the variables takes advantage of the fact that problem~\eqref{eq:s2}, even though nonconvex, is solved to global optimality by computing the $r$ leading left singular vectors of the matrix $\Phi(X_{k+1})$. The result is a consequence of the celebrated Eckart-Young-Mirsky theorem, which gives the best rank $r$ approximation in Frobenius norm of a matrix by the $r$ leading terms of the singular value decomposition~\cite{eckart1936approximation, mirsky1960symmetric}. In particular, let $\Phi(X_k) = \sum_{i=1}^{\min(N,s)} \sigma_i u_i v_i^\top$, then 
\begin{equation} 
\U_{k+1} = \txt{span}(u_1, \dots, u_r)
\label{eq:truncate_svd}
\end{equation}
 is a global minimizer of~\eqref{eq:s2}\footnote{Note that the solution need not be unique, in the case where $\sigma_r = \sigma_{r+1}$.}. This truncated singular value decomposition (SVD) is denoted by \verb=truncate_svd= in Algorithm~\ref{algo:am}. The singular vectors can be computed to high accuracy in polynomial time~\cite{trefethen1997numerical}. 
\paragraph{} 
 Problem~\eqref{eq:s1} is in general hard to solve to global optimality.  The difficulty comes from the nonconvexity of the cost function, which is due to $\Phi$. One can choose from a variety of first-or second-order methods to find an approximate first-or second-order critical point.  We will present the merits of both possibilities in Section~\ref{sec:numerics}. 
\paragraph{i) First-order version of alternating minimization} When only gradient information is available, a first-order method will be used to minimize subproblem~\eqref{eq:s1}. For the sake of illustration, in Algorithm~\ref{algo:am} we present a projected gradient descent with line search for~\eqref{eq:s1}. The gradient of the cost function is projected onto the null space of $\A$. This ensures that the iterates remain in the feasible set $\LAb$. The line search is a classical backtracking with an Armijo condition for sufficient decrease. Variants in the line search or even constant step sizes are possible. 
\paragraph{ii) Second-order version of alternating minimization}
In the subproblem~\eqref{eq:s1}, it is possible to use second-order methods to speed up the convergence and reach a higher accuracy. We apply RTR on the affine manifold $\LAb$ with the Hessian of the cost function $\nabla_{XX}^2 f(X_k,\U_k)$ in the model. 
\paragraph*{}	Algorithm~\ref{algo:am} details the alternating minimization where gradient descent with an Armijo line search, a standard inexact procedure in nonconvex optimization, is applied in the subproblem~\eqref{eq:s1}. The  Armijo line search is described in Algorithm~\ref{algoLS}. 

\begin{algorithm}
\caption{Alternating minimization scheme for Problem~\eqref{eq:p1} or~\eqref{eq:p1kernel}}\label{algo:am}
\begin{algorithmic}[1]
\State \textbf{Given:} The sensing matrix $A\in \R^{m\times ns}$, measurements $b\in \R^m$, tolerances $\varepsilon_u \geq 0, \varepsilon_x\geq 0$, an estimation of $r = \rank(\Phi(M))$.
\State Set $k=0$
\State Find $X_0$ that satisfies $AX_0 = b$
\State $U_{0} = \verb= truncate_svd= (\Phi(X_0))$ \Comment{ Equation~\eqref{eq:truncate_svd}} 
\While{$ \fronorm{\grad_X f(X_{k},\U_k) } > \varepsilon_x$ or $ \fronorm{\grad_\U f(X_{k},\U_k)} > \varepsilon_u$}
\State Set $X_k^{(0)} = X_k, i=0$
\State Choose $\varepsilon_{x,k}$ using Equation~\eqref{eq:greedy_decrease} or~\eqref{eq:adaptive_decrease}
\While{$ \|\grad_X f(X_{k}^{(i)},\U_{k}) \| > \varepsilon_{x,k}$}
\State $ \grad_X f(X_{k}^{(i)},\U_{k}) =  \Prm_{\Trm \LAb}\Big(\nabla_X f(X_{k}^{(i)},\U_{k}) \Big)$ \Comment{Equation~\eqref{eq:PTLAb}}
\State $ \alpha_k^{(i)} = $ Armijo$\left((X_{k}^{(i)},\U_{k}), - \grad_X f(X_{k}^{(i)},\U_{k})\right)$ \Comment{Algorithm~\ref{algoLS} }
\State  $X_{k}^{(i+1)} = X_k^{(i)} - \alpha_k^{(i)} \grad_X f(X_{k}^{(i)},\U_{k})$
\State $i =i+1$
\EndWhile
\State $X_{k+1} = X_k^{(i)}$
\If{$ \|\grad_\U f(X_{k+1},\U_k) \| \leq \varepsilon_{u}$}
\State $U_{k+1} = U_k$
\Else
\State $U_{k+1} = \verb= truncate_svd=(\Phi(X_{k+1}))$ 
\EndIf
\State $k = k +1$
\EndWhile
\State {\bf Output:} $(X_k,\U_k)$ such that $\|\grad f(X_k,\U_k)\| \leq \varepsilon_u + \varepsilon_x$.
\end{algorithmic}
\end{algorithm}

\begin{algorithm}
\caption{Armijo($z_k$, $d_k$): Line search with Armijo condition}\label{algoLS}
\vskip1mm
\textit{{\bf INPUT}: Function $f$ and gradient $\grad_X f$, current iterate $(X_{k}^{(i)},\U_{k})$ and a descent direction $d_k$ such that $\langle \grad_X  f(X_{k}^{(i)},\U_{k}),d_k \rangle < 0$, a sufficient decrease coefficient $\beta\in ]0,1[$, initial step $\alpha_0>0$ and $\tau \in ]0,1[$. }
\vskip1mm
\textit{{\bf OUTPUT}: Step size $\alpha_k^{(i)}$. }
\vskip1mm
\hrule
\vskip2mm
\begin{algorithmic}[1]
\State Set $\alpha=\alpha_0$.
\While{$  f(X_k^{(i)} + \alpha d_k,\U_k)  > f(X_{k}^{(i)},\U_{k}) + \beta \alpha \langle  \grad_X  f(X_{k}^{(i)},\U_{k}),d_k \rangle $ }
\State $\alpha = \tau \alpha$.
\EndWhile
\State Set $\alpha_k^{(i)} = \alpha$.
\end{algorithmic}
\end{algorithm}

\paragraph*{iii) Accuracy of the subproblems solution}
Algorithm~\eqref{algo:am} alternatively solves the subproblems~\eqref{eq:s1} and~\eqref{eq:s2}. For the solution of~\eqref{eq:s1}, there is no incentive to solve it to very high accuracy early on in the run of the algorithm, as we could still be far from convergence and the variable $\U$ might still change a lot. At iterations $k$, we use the following stopping criterion for some $\varepsilon_{x,k}>0$,

\begin{equation}
\fronorm{\grad_X f(X_{k+1},\U_{k})} \leq \varepsilon_{x,k}.
\end{equation}
We propose the two following strategies for the choice of $\varepsilon_{x,k}$,
 \begin{equation}
\varepsilon_{x,k} = \varepsilon_x \txt{ for all } k, \label{eq:greedy_decrease}
\end{equation}
or
\begin{equation}
	\varepsilon_{x,k} = \max\left(\varepsilon_x,  \theta \fronorm{\grad_X f(X_k, \U_{k})} \right) \txt{ for some user-chosen } 0<\theta<1.\label{eq:adaptive_decrease}
\end{equation}			

To solve~\eqref{eq:s2}, it is possible to use a randomized SVD procedure. The randomized SVD is a stochastic algorithm that approximately computes the singular value decomposition of a matrix that exhibits a low-rank pattern~\cite{halko2011finding}. The matrix must be of low rank or have a fast decay in its singular values for the random SVD to be accurate. As the iterates $X_k$ converge towards the solution $M$, the matrix $\Phi(X_k)$, for which we have to compute  an SVD, becomes low-rank and therefore it is natural to use a randomized SVD in Algorithm~\ref{algo:am}. In the first iterations, for a random starting point of the algorithm, the feature matrix is not expected to be low-rank. In those case, the random SVD should not be used. When the matrix $\Phi(X_k)$ is only approximately low-rank, we can apply power iterations to make the singular values decrease faster. This will make the randomized decomposition more costly, but will improve the accuracy of the computed SVD. 

Our strategy is as follows, choose two parameters $0<\tau_1\ll\tau_2<1$. As long as $f(X_{k+1},\U_k) > \tau_2$, use an exact SVD algorithm, without randomization. When $\tau_1 < f(X_{k+1},\U_k)\leq  \tau_2$ we know that the energy of $\Phi(X_{k+1})$ is mostly contained in the  span of $\U_k$ which has dimension $r$.  So we are justified in using a randomized $\SVD$, which we start up with a step of the power method to improve the accuracy. When $f(X_{k+1},\U_k)\leq \tau_1$, the matrix $\Phi(X_{k+1})$ is even closer to low-rank and we no longer need to apply a power iteration before computing the randomized SVD.

\FloatBarrier

\section{Convergence of the alternating minimization algorithm}
\label{sec:am_convergence}
In this section we present convergence results for the alternating minimization Algorithm~\ref{algo:am}. We consider a first-order version where subproblem~\eqref{eq:s1} is minimized with the gradient descent method and the Armijo backtracking line-search (Algorithm~\ref{algoLS}).
We will first show asymptotic convergence of the gradient norms to zero. We also give a worst-case global complexity bound on the number of iterations necessary to achieve a small gradient from an arbitrary initial starting point. Note that we chose the Armijo linesearch for the sake of example. Minor adjustments of the proof below allow to prove similar results for other minimization methods in subproblem~\eqref{eq:s1}.  

\subsection{Assumptions}
\begin{assumption}
There exist constants $L_x$ and $L_u$ (which are both independent of $X$ and $\U$) such that for all $z=(X, \U) \in \M$, the pullback $\hat f_z = f\circ \Retr_z$ has a Lipschitz continuous gradient in $X$ and $\U$, with constants $L_x$ and $L_u$ respectively. That is, for all $(\eta_x,\eta_u) \in  \Trm_{(X,\U)}\M$,
\begin{equation}
\left| f\circ \Retr(\eta_x,0) - [f(X,\U) + \langle \grad_X f(X,\U), \eta_x \rangle] \right| \leq \dfrac{L_x}{2} \norm{\eta_x}^2
\label{eq:lipschitz_x_am}
\end{equation}
and 
\begin{equation}
\left| f\circ \Retr(0,\eta_u) - [f(X,\U) + \langle \grad_\U f(X,\U), \eta_u \rangle] \right| \leq \dfrac{L_u}{2} \norm{\eta_u}^2.
\label{eq:lipschitz_u_am}
\end{equation}
In words, this means that, in each variable, the pullback is well approximated by its first-order Taylor approximation. 
\label{assumption:pullback_lipschitz_altmin}
\end{assumption}
\begin{remark}
Note that if~\aref{assu:lip-gradient-tr} holds, then~\aref{assumption:pullback_lipschitz_altmin} holds with $L_x = L_u = L_g$. 
\end{remark}
 Let us discuss, under which conditions on the kernel one can ensure that~\aref{assu:lip-gradient-tr} or~\aref{assumption:pullback_lipschitz_altmin} are satisfied for the cost function of~\eqref{eq:p1kernel}. The following discussion requires to use the exponential map as the retraction. The exponential map follows geodesics along the manifold in directions prescribed by tangent vectors. Using the exponential map on $\M$ is not a restriction, as the exponential map on the Grassmann manifold is computable~\cite{absil2004riemannian} and the exponential map on $\LAb$ is the identity. 

\begin{proposition}
Consider the cost function of~\eqref{eq:p1kernel} and assume that the retraction being used is the exponential map. If $\D \K(X,X)$ is Lipschitz continuous over $\LAb$, then~\eqref{eq:lipschitz_x_am} holds where $L_x$ is the Lipschitz constant of $\D \K(X,X)$.  If $\fronorm{\K(X,X)}\leq M$ for all $X\in \LAb$, condition~\eqref{eq:lipschitz_u_am} holds with $L_u = 2M$.
\label{prop:condition-kernel-grad-lip}
\end{proposition}
\begin{proof}
See Appendix~\ref{sec:appendix_derivatives}.
\end{proof}

The conditions listed in Proposition~\ref{prop:condition-kernel-grad-lip} on the kernel and its derivatives are not always satisfied or can be difficult to verify. For instance, the Gaussian kernel $\K_G$ is bounded above on $\LAb$, but the monomial kernel $\K_d$ is not for any degree $d\geq 1$.  For the Gaussian kernel, the map $\D \K_G(X)$ is always Lipschitz continuous on $\LAb$. Whereas for the monomial kernel,  the map $\D \K_d(X)$ is Lipschitz continuous for $d\leq 2$, and only locally Lipschitz continuous for $d\geq 3$.

Fortunately, the picture is much simpler if the sequence of iterates $(X_k)_{k\in \mathbb{N}}$ generated by Algorithm~\ref{algo:RTR} or~\ref{algo:am} is contained in a bounded set. This ensures that we can find Lipschitz constants such that the bounds of \aref{assu:lip-gradient-tr},~\aref{assu:lip-hessian-tr} and~\aref{assumption:pullback_lipschitz_altmin} hold at every iterate of the algorithm (and trial points if any), which is all that is needed in the convergence analysis. 

\begin{proposition}
Consider the cost function of either~\eqref{eq:p1} or~\eqref{eq:p1kernel} and apply Algorithm~\ref{algo:am} or Algorithm~\ref{algo:RTR}  with the exponential map as the retraction. If the convex hull of the sequence of iterates $(X_k)_{k\in \mathbb{N}}$ and the trial points of the algorithm is a bounded set, then~\eqref{eq:grad-pullback-lipschitz},
\eqref{eq:hess-pullback-lipschitz} and \eqref{eq:lipschitz_u_am}-\eqref{eq:lipschitz_x_am} hold at every iterate $\left(X_k,\U_k\right)_{k\in\mathbb{N}}$ and trial points of the algorithm.
\end{proposition}
\begin{proof}
See appendix~\ref{sec:appendix_derivatives}.
\end{proof}

\subsection{Global convergence results}
We now carry on with the convergence analysis of the alternating minimization algorithm. The next lemma is an adaptation of the classical descent lemma for the SVD step. 
\begin{lemma}[{Descent lemma based on~\cite[Theorem 4]{boumal2019global}}]
Let \aref{assumption:bounded_below} and \aref{assumption:pullback_lipschitz_altmin} hold for $f:\M \to \R$. 
Then, for any $k\geq 0$, the iterates generated by Algorithm~\ref{algo:am} satisfy
\begin{equation}
f(X_{k},\U_k) - f(X_{k+1},\U_{k+1}) \geq \dfrac{1}{2L_u}\fronorm{\grad_\mathcal{U} f(X_{k+1},\U_k)}^2,
\end{equation}
where $L_u$ is the Lipschitz constant of the gradient of the pullback~(\aref{assumption:pullback_lipschitz_altmin}).
\label{lemma:gradU_bound}
\end{lemma}
\begin{proof}
See Appendix~\ref{appendix:proofs-6}.
\end{proof}
Throughout this section we use the following notation. Let the number of gradient steps between $X_k$ and $X_{k+1}$ be $n_k \geq 0$ and the intermediate iterates,
$$X_k=X_k^{(0)},~X_k^{(1)},~X_k^{(2)},~\dots,~ X_k^{(n_k)}=X_{k+1} .$$
The next lemma gives upper and lower bounds on the step sizes given by the Armijo linesearch. This is an adaptation of a standard argument for linesearch methods~\cite{nocedal2006numerical} where the constraint $\A(X)=b$ is added. 
\begin{lemma}
Under~\aref{assumption:pullback_lipschitz_altmin}, for the direction
$-\grad_X f(X_k^{(i)},\U_k) \in T_{X_k}\LAb$, 
the linesearch Algorithm~\ref{algoLS} produces a step size $\alpha_k^{(i)}$ that satisfies 
\begin{equation}
\underline{\alpha} := \min \left\lbrace \alpha_0,   \dfrac{ 2\tau(1  - \beta)  }{L_x } \right\rbrace\leq \alpha_k^{(i)} \leq \alpha_{0}
\end{equation}
and produces the following decrease
\begin{equation}
   f(X_k^{(i)},\U_k) -  f(X_k^{(i+1)},\U_k) \geq   \beta \alpha \fronorm{\grad_X  f(X_k^{(i)},\U_k)}^2,
   \label{eq:armijo_decrease_grad}
\end{equation}
where $X_k^{(i+1)} = X_k^{(i)} - \alpha_k^{(i)} \grad_X f(X_k^{(i)},\U_k)$.
\label{lemma:step_bound}
\end{lemma}
\begin{proof}
See Appendix~\ref{appendix:proofs-6}.
\end{proof}
We are now ready to prove global convergence of the alternating minimization algorithm. 
\begin{theorem}[Global convergence for Alternating minimization]\label{thm:global_convergence}
Let \aref{assumption:pullback_lipschitz_altmin} hold for $f:\M \to \R$ from~\eqref{eq:p1} or~\eqref{eq:p1kernel}. 
Let $\varepsilon_x=0$, $\varepsilon_u=0$ and use Equation~\eqref{eq:adaptive_decrease} to set $\varepsilon_{x,k}$, for any starting point $(X_0,\U_0)\in \M$,  Algorithm~\ref{algo:am} produces a sequence $\Big(X_{k}, \U_{k}\Big)_{k\in \mathbb{N}}$ such that 
\begin{equation}
\lim_{\substack{k\to\infty }}\fronorm{ \grad f(X_{k}, \U_{k})}=0.
\label{eq:grad_to_zero}
\end{equation}
\end{theorem}
\begin{proof}
First note that $f$ is bounded below by $f_* = 0$. For any $k \geq 0$,
\begin{align}
\norm{ \big( \grad_X f(X_{k}, \U_{k}),  \grad_\U f(X_{k}, \U_{k}) \big)} &\leq \fronorm{ \grad_X f(X_{k}, \U_{k})}+ \fronorm{ \grad_\U f(X_{k}, \U_{k})}\\
 &\leq  \fronorm{ \grad_X f(X_{k}, \U_{k})}
\label{eq:boundgradUX}
\end{align}
since $\varepsilon_u =0$. 
Given that each step is non-increasing,
\begin{align*}
f(X_{k},\U_{k})- f(X_{k+1},\U_{k+1}) &\geq f(X_{k},\U_{k})- f(X_{k+1},\U_{k})  \\
&\geq f(X_{k},\U_{k}) - f(X_{k}^{(1)},\U_{k})  \\
  &\geq \beta \alpha_k^{(0)} \fronorm{ \grad_X f(X_k, \U_k)}^2\\
  &\geq \beta \underline{\alpha} \fronorm{ \grad_X f(X_k, \U_k)}^2,
\end{align*}
where we used Lemma~\ref{lemma:step_bound} about Armijo steps.
Summing over all iterations gives a telescopic sum on the left-hand side. For any $\bar{k}\geq 0$,
\begin{equation}
f(X_0,\U_0)- f^* \geq f(X_0,\U_0)- f(X_{\bar{k}},\U_{\bar{k}})
 \geq \beta  \underline{\alpha} \sum_{k=0}^{\bar{k}} \fronorm{ \grad_X f(X_k, \U_k)}^2.\\ 
\end{equation}
The series is convergent since it is bounded independently of $\bar{k}$. Therefore 
\begin{equation}
\lim_{k\to\infty}\fronorm{\grad_X f(X_k,\U_k)}=0.
\end{equation}
We have $\norm{\grad_\U f(X_k,\U_k)}=0$ for all $k\geq 0$ since $\varepsilon_u=0$. This corresponds to taking exact SVDs. Taking $k\to \infty$ in equation~\eqref{eq:boundgradUX} gives convergence of the gradient norms to zero~\eqref{eq:grad_to_zero}. 
\end{proof}
\begin{theorem}[Global complexity for Alternating minimization]
Let \aref{assumption:pullback_lipschitz_altmin} hold for $f:\M \to \R$ from~\eqref{eq:p1} or~\eqref{eq:p1kernel}. 
Let $\varepsilon_x>0$, $\varepsilon_u>0$, and use Equation~\eqref{eq:greedy_decrease} or~\eqref{eq:adaptive_decrease} to set $\varepsilon_{x,k}$. For any starting point $z_0=(X_0,\U_0)\in \M$, Algorithm~\ref{algo:am} produces a sequence $\Big(X_{k}, \U_{k}\Big)_{k\in \mathbb{N}}$ such that 
\begin{equation}
\fronorm{\Big(\grad_X f(X_{k}, \U_{k}), \grad_\U f(X_{k}, \U_{k}) \Big)}\leq  \varepsilon_x + \varepsilon_u
\end{equation}
is achieved using at most $N_{grad}$ gradient steps and $N_{svd}$ singular value decompositions with 
\begin{align}
N_{grad}  &\leq \dfrac{(f(z_0) -f_*)}{\underline{\alpha
}\beta \varepsilon^2_x} & \text{ and } && N_{svd} &\leq \dfrac{2L_u(f(z_0) -f_*)}{\varepsilon^2_u},
\end{align}
where $\underline{\alpha}=\min\left\lbrace \alpha_0, 2\tau (1-\beta)/L_x\right\rbrace $ is a constant depending on parameters of the line search~\eqref{eq:s1}.
\label{thm:global_complexity_altmin}
\end{theorem}
\begin{proof}
Note that $f$ is bounded below by $f_*=0$. Define $N_{iter}$ as the number of iterations performed by Algorithm~\ref{algo:am}, the smallest $k$ such that $\norm{\grad_X f(X_k,\U_k) }\leq \varepsilon_x$ and $\norm{\grad_\U f(X_k,\U_k) }\leq\varepsilon_u$.
Let $N_{svd}$ be the number of SVDs that have to be performed to get $\| \grad_{\U}f(X_{k+1},\U_{k})\|\leq \varepsilon_u$, at which point the algorithm would return without performing another computation. 
For any $k\leq N_{svd}$, from Lemma~\ref{lemma:gradU_bound} we have  
\begin{equation}
f(X_{k},\U_k)- f(X_{k+1},\U_{k+1}) \geq \frac{1}{2L_u}\fronorm{\grad_\U f(X_{k+1},\U_k)}^2 \geq \frac{1}{2L_u}\varepsilon_u^2. 
\end{equation}
Summing from $k=0$ to $N_{svd}$ gives, 
\begin{equation}
f(z_0) - f_* \geq f(z_0) - f(z_{N_{svd}}) \geq \sum_{k=0}^{N_{svd}}\frac{\varepsilon_u^2}{2L_u} = \frac{\varepsilon_u^2 N_{svd}}{2L_u}.
\end{equation}
Hence, this bounds the number of SVDs to ensure  $\| \grad_{\U}f(X_{k+1},\U_{k})\|\leq \varepsilon_u$, as
\begin{equation}
N_{svd} \leq 2L_u \frac{(f(z_0) - f_*)}{\varepsilon_u^2}.
\end{equation}
 For $0\leq i\leq n_k -1$, we have $ \fronorm{ \grad_X f(X_{k}^{(i)},\U_k)}^2 \geq  \varepsilon_{x,k}^2 $ by definition since the stopping criterion is $ \fronorm{\grad_X f(X_{k}^{(n_k)},\U_k)}\leq \varepsilon_{x,k}$.
Combined with the Armijo decrease this gives
\begin{equation}
f(X_{k}^{(i)},\U_k)- f(X_{k}^{(i+1)},\U_k) \geq  \alpha^{(i)}_k\beta \fronorm{ \grad_X f(X_{k}^{(i)},\U_k)}^2 \geq \alpha^{(i)}_k\beta \varepsilon_{x,k}^2.  
\label{eq:armijo_decreaseUX}
\end{equation}
We sum these bounds for the $n_k$ gradient steps from  $X_k$ to $X_{k+1}$, 
\begin{align}
\sum_{i=0}^{n_k-1}\left[ f(X_{k}^{(i)},\U_k)- f(X_{k}^{(i+1)},\U_k) \right] \geq \sum_{i=0}^{n_k-1} \alpha^{(i)}_k\beta \varepsilon_{x,k}^2 . \\
\intertext{Using that the step sizes $\alpha^{(i)}_k$ are bounded below by $\underline{\alpha} = \min\left\{\alpha_0, 2\tau (1-\beta)/L_x \right\}$ (Lemma~\ref{lemma:step_bound}),}
f(X_{k},\U_k)- f(X_{k+1},\U_k) \geq n_k \underline{\alpha}\beta \varepsilon_{x,k}^2 ~~ \text{ for all } k\geq 0. 
\end{align}
 The SVD is nonincreasing, meaning $f(X_{k},\U_k) - f(X_{k+1},\U_{k+1}) \geq  f(X_{k},\U_k)- f(X_{k+1},\U_k)$. This yields,
\begin{equation}
\begin{aligned}
f(X_{k},\U_{k})- f(X_{k+1},\U_{k+1}) \geq n_k \underline{\alpha}\beta \varepsilon_{x,k}^2 \geq n_k \underline{\alpha}\beta \varepsilon_{x}^2 ~~ \text{ for all } k \leq N_{iter}, \\
\end{aligned}
\end{equation}
as both~\eqref{eq:greedy_decrease} and~\eqref{eq:adaptive_decrease} satisfy $\varepsilon_{x,k} \geq \varepsilon_x$. We sum once again over the iterations,
\begin{equation}
\begin{aligned}
f(X_{0},\U_0)- f_* \geq f(X_{0},\U_{0})- f(X_{N_{iter}+1},\U_{N_{iter}+1}) \geq \sum^{N_{iter}}_{k=0} n_k\underline{\alpha}\beta \varepsilon_{x}^2.  \\
\end{aligned}
\end{equation}
We conclude that
\begin{equation}
\dfrac{(f(z_{0})- f_*) }{ \underline{\alpha}\beta \varepsilon_x^2} \geq \sum^{N_{iter}}_{k=0} n_k =: N_{grad} . 
\end{equation}
\end{proof}

A similar algorithm using fixed step sizes for the update of $X$ will also converge, provided the step sizes are small enough. 
\begin{corollary}
If the Armijo linesearch in Algorithm~\ref{algo:am} is replaced by a gradient descent with constant stepsizes $\alpha$ satisfying $\alpha < \dfrac{2}{L_x}$, Algorithm~\ref{algo:am} also converges
\begin{equation}
\dis \lim_{k\to\infty}\fronorm{ \Big(\grad_X f( X_{k}, \U_{k}),   \grad_\U f( X_{k},\U_{k}) \Big)}=0.
\label{eq:grad_to_zero2}
\end{equation}
We also have the worst case bound
\begin{equation}
N_{grad}  \leq \dfrac{L_x(f_0 -f_*)}{\alpha \varepsilon^2_x}.
\label{eq:N_grad}
\end{equation}
\end{corollary}
\begin{proof}
We derive the usual descent lemma from Lipschitz continuity of the gradient. This gives
\begin{equation}
f\big(X_k - \alpha \grad_X f(X_k,\U_k), \U_k\big) \leq f(X_k,\U_k) - \alpha \fronorm{\grad_X f(X_k,\U_k)}^2 + \alpha^2L_x/2\fronorm{\grad_X f(X_k,\U_k)}^2
\end{equation}
which simplifies to 
\begin{equation}
f(X_k,\U_k) - f(X_k^{(1)},\U_k) \geq (\alpha - \alpha^2L_x/2)\fronorm{\grad_X f(X_k,\U_k)}^2.
\end{equation}
This bound replaces the Armijo decrease of Equation~\eqref{eq:armijo_decrease_grad}.
The rest of the proofs from Theorems~\ref{thm:global_convergence} and~\ref{thm:global_complexity_altmin} holds verbatim with stepsize $\alpha$ for every iteration. Note that for $\alpha>0$, the factor $(\alpha - \alpha^2L_x/2)$ is positive for $\alpha<2/L_x$ and is maximized at $\alpha= 1/L_x$. 
\end{proof}

\section{Convergence of the iterates using the Kurdyka-Lojasiewicz property}
\label{sec:KL}
 This section proves convergence of the sequence of iterates to a unique stationary point for a simplified version of the alternating minimization scheme. This section considers an algorithm where only one gradient step is performed in between the truncated SVDs (Algorithm~\ref{algo:am-simple}). This is analogue to the algorithm described in~\cite{fan2019polynomial} which does not provide theoretical convergence guarantees. Our observations indicate that Algorithm~\ref{algo:am-simple} is expected to behave similarly to Algorithm~\ref{algo:am} in the limit. Asymptotically, there is usually only one gradient step needed between two truncated SVDs. It is only in the early iterations that Algorithm~\ref{algo:am} differs by making several gradient steps in between SVDs. For the purpose of this theoretical section, we will assume that the singular value decompositions in Algorithm~\ref{algo:am-simple} are exact and not approximated or randomized. This corresponds to setting $\varepsilon_u=0$ in Algorithm~\ref{algo:am}. This section is written using the notation of a feature matrix $\Phi$ as in problem~\eqref{eq:p1}, but the results apply similarly to problem~\eqref{eq:p1kernel} if one assumes that the Lipschitz condition~\aref{assumption:lipschitz} applies to a kernel $\K$ instead of $\Phi$. 
\begin{algorithm}
\caption{A simple alternating minimization scheme for Problem~\eqref{eq:p1} or~\eqref{eq:p1kernel}}\label{algo:am-simple}
\begin{algorithmic}[1]
\State \textbf{Given:} The sensing matrix $A\in \R^{m\times ns}$, measurements $b\in \R^m$, a tolerance $\varepsilon _x>0$, an estimation of $r = \rank(\Phi(M))$.
\State Set $k=0$
\State Find $X_0$ that satisfies $AX_0 =b$.
\State $U_{0} = \verb= truncate_svd=(\Phi(X_{0}))$ \Comment{Equation~\eqref{eq:truncate_svd} }
\While{$ \fronorm{\grad_X f(X_{k},\U_k) } > \varepsilon_x$}
\State $ \grad_X f(X_{k},\U_{k}) = \mathrm{P}_{\mathrm{T}\LAb} (\nabla_X f(X_{k},\U_{k}))$ \Comment{Equation~\eqref{eq:PnullA}} 
\State $ \alpha_k = \text{Armijo}\left( (X_k,\U_k), - \grad_X f(X_k,\U_k)\right)$ \Comment{Algorithm~\ref{algoLS}}
\State  $X_{k+1} = X_k - \alpha_k \grad_X f(X_{k},\U_{k})$
\State $U_{k+1} = \verb= truncate_svd=(\Phi(X_{k+1}))$  \Comment{exact SVD, not randomized}
\EndWhile
\State {\bf Output:} $(X_k,\U_k)$ such that $\fronorm{\grad f(X_k,\U_k)} \leq \varepsilon_x$.
\end{algorithmic}
\end{algorithm}

\FloatBarrier 
  Let us define a distance on the manifold $\M$ (Equation~\eqref{eq:M}).
\begin{definition}[Distance on $\M$]
Given two subspaces $\U_1,\U_2 \in \Grass(N,r)$, the canonical angles $\theta_i $ for $i = 1,\dots,r$ are defined as  $\theta_i= cos^{-1}(\sigma_i)$ where $\sigma_i$ are the $r$ singular values of $U_1^\top U_2 $, with $\range(U_1) = \calU_1$ and $\range(U_2) = \calU_2$. For all \(
 \U_1, \U_2$ in $\Grass(N,r)\) define \( \dist( \U_1, \U_2) := \sqrt{ \sum_{i=1}^r \sin^2 \theta_i}.\)
This gives a distance on $\M$,
\begin{equation}
\dist\Big( (X_1,\U_1), (X_2,\U_2) \Big) := \sqrt{ \fronorm{X_1 - X_2 }^2 + \sum_{i=1}^r \sin^2 \theta_i}
\label{eq:dist-sin-theta}
\end{equation}
for all $(X_1,\U_1), (X_2,\U_2)$ in $\M$.
\end{definition}
We will prove finite length of the sequence of iterates in this metric on $\M$. A more mainstream approach to define the distance between $\calU_1$ and $\calU_2$ on the Grassmann would be to use $\fronorm{\Theta}$ instead of $\fronorm{\sin \Theta}$, where $\Theta = \txt{diag}(\theta_i)$ is the diagonal matrix containing the principal angles. We do so because the distance~\eqref{eq:dist-sin-theta} makes it easier to derive perturbation bounds for the SVD and is equivalent to the usual distance. 

The following assumption ensures a useful non-degeneracy of the spectrum of the feature matrix. 
\begin{assumption}[Gap between the singular values] \label{assumption:sigma}
There exists $\delta > 0$ such that, the accumulation points of the sequence generated by Algorithm~\ref{algo:am-simple} satisfy 
\begin{equation}
 \sigma_{r}(\Phi(X)) - \sigma_{r+1}(\Phi(X)) \geq \delta>0. 
\end{equation}
\end{assumption}
This property ensures that the minimizer of the function $f( X,\cdot) : \mathrm{Grass}(N,r) \to \R$ is well defined, i.e., that its truncated SVD is unique.  As $\sigma_{r+1}(\Phi(X)) \geq 0$, this assumption also implies that 
\[ \sigma_{r}(\Phi(X)) \geq \delta >0 \]
In particular it means that we cannot overestimate the rank. If the true rank is $r-1$, then $\sigma_r =0$ and the assumption does not hold. Let us stress that this is an artefact of the convergence proof and does not imply poor practical performance of the algorithm when the rank is overestimated. We investigate this in the numerics Section~\ref{sec:bad_rank_estimation}. We will need Assumption~\ref{assumption:sigma} to derive a Lipschitz continuity result on the SVD. We now show the two main lemmas (\ref{lemma:grad_lower_bound-1} and~\ref{lemma:suffDecr-1}), inspired by~\cite{Bolte2013}. 
\begin{lemma}[Gradient lower bound on iterates gap]
Assume that Algorithm~\ref{algo:am-simple} generates a bounded sequence of iterates. Then, there exists $\rho_2 >0$ such that, for all $k\in \mathbb{N}$,
\begin{equation}
\fronorm{\grad f(X_{k+1},\U_{k+1})} \leq \rho_2 \dist\Big(  (X_{k+1},\U_{k+1}) , (X_{k},\U_{k}) \Big)
\label{eq:rho2-1}
\end{equation}
with $\rho_2 :=2(L_g + 1/\underline{\alpha})$ for some $L_g \geq 0$.
\label{lemma:grad_lower_bound-1}
\end{lemma}
\begin{proof}
The expression
\begin{equation}
X_{k+1} = X_{k} -\alpha_k \grad_X f(X_{k},\U_k)
\end{equation}
implies
\begin{equation}
\grad_X f(X_{k+1},\U_{k+1}) = (  X_{k} - X_{k+1})/\alpha_k + \grad_X f(X_{k+1},\U_{k+1}) - \grad_X f(X_{k},\U_{k})
\end{equation}

Define the set $\tilde S = \cl\left(\conv( (X_k)_{k\in \mathbb{N}})\right)$, the closure of the convex hull of the sequence of iterates, and $S = \tilde S \times \Grass(N,r)$. We show that the vector field $\grad f\restr{S}	\colon S \to \mathrm{T}\M$ is $L_g$-Lipschitz continuous in the sense of Definition~\ref{def:lipschitz-continuous-manifold} for some $L_g\geq 0$. Since $S$ is bounded and the Hessian is continuous, there exists $L_g\geq 0$ such that $\norm{\Hess f(x)} \leq L_g $ for all $x\in S$. By Proposition~\ref{prop:bounded-hessian-lip-gradient}, $\grad f\restr{S}$ is $L_g$-Lipschitz continuous.
Using the triangular inequality and the fact that $\underline{\alpha}$ is a lower bound of $\alpha_k$ for all $k$ gives
\begin{align*}
\fronorm{\grad_X f(X_{k+1}, \U_{k+1})} &\leq \fronorm{  X_{k+1} - X_k}/\underline{\alpha}+ \fronorm{ \grad_X f(X_{k+1},\U_{k+1}) - \grad_X f(X_{k},\U_{k})}\\\nonumber
&\leq    \dist\Big(  (X_{k+1},\U_{k+1}) , (X_{k},\U_{k}) \Big)/\underline{\alpha} + L_g   \dist\Big(  (X_{k+1},\U_{k+1}) , (X_{k},\U_{k}) \Big)\\ \nonumber
&\leq (1/\underline{\alpha}+L_g)  \dist\Big(  (X_{k+1},\U_{k+1}) , (X_{k},\U_{k}) \Big).
\end{align*}
This gives~\eqref{eq:rho2-1} recalling that, since $\grad_\calU f(X_{k+1}, \calU_{k+1}) = 0$,
\begin{align*}
\fronorm{\grad f(X_{k+1},\U_{k+1})} &= \fronorm{\grad_X f(X_{k+1},\U_{k+1})}. \qedhere
\end{align*}
\end{proof}
Further auxiliary results are needed.
\begin{lemma}[Wedin's theorem~\cite{stewart1998perturbation}]\label{thm:Wedin}
Let $Y, \check Y \in \mathbb{R}^{N \times s}$ with singular value decompositions
\begin{align*}
Y &= \sum_{i = 1}^{\min(N,s)} \sigma_i u_i (v_i)\transpose && \andt & \check Y &= \sum_{i = 1}^{\min(N,s)} \check \sigma_i \check u_i   (\check v_i)\transpose,
\end{align*}
with $\sigma_1 \geq \sigma_2 \geq \cdots \geq \sigma_{\min(N,s)}$ and similarly for $\check Y$. If there exists $\delta > 0$ such that
\begin{align}\label{eq:wedin_assumption1}
 \min_{\substack{ 1\leq i \leq r\\
r+1\leq j \leq \min(N,s)}}  |\check  \sigma_i -  \sigma_j | \geq \delta
\end{align}
and
\begin{align*}
\check \sigma_r \geq \delta,
\end{align*}
 then
 \begin{equation}
   \fronorm{\sin \Theta}^2  \leq \dfrac{2 \fronorm{\check Y - Y}^2}{\delta^2}
 \label{eq:wedin2}
\end{equation}
 with $\Theta$ the matrix of the principal angles between $\begin{bmatrix}
 u_1  & u_2 & \cdots & u_r
\end{bmatrix}$ and $\begin{bmatrix}
\check u_1  & \check u_2 & \cdots & \check u_r
\end{bmatrix}$.
\end{lemma}
The following lemma is a direct consequence of Wedin's theorem. 
\begin{lemma}\label{lemma:wedin}
Let $Y, \check Y \in \R^{N \times s}$. Consider the singular value decomposition of $Y = \sum_{i=1}^{\min(N,s)} \sigma_i u_i v_i\transpose$, with $\sigma_1 \geq \sigma_2 \geq \dots \geq \sigma_{\min(N,s)}$. Let us also write $U_r :=  \begin{bmatrix}
 u_1  & u_2 & \cdots & u_r
\end{bmatrix}$, a matrix whose columns span the left principal subspace associated to the $r$ largest singular values. Similarly,  $\check Y = \sum_{i=1}^{\min(N,s)} \check \sigma_i  \check u_i \check v_i\transpose$, with $\check \sigma_1 \geq \check \sigma_2 \geq \dots \geq \check \sigma_{\min(N,s)}$. Let us also write $\check U_r :=  \begin{bmatrix}
 \check u_1  & \check u_2 & \cdots & \check u_r
\end{bmatrix}$. If there exists $\delta > 0$ such that $\sigma_r - \sigma_{r+1} \geq \delta$ and  $\check \sigma_r - \check\sigma_{r+1} \geq \delta$, then
\begin{equation*}
\dist( \check \U_r,  \U_r)^2 \leq \frac{2}{\delta^2} \fronorm{\check Y - Y }^2,
 \end{equation*}
where $\dist(\U_r, \check \U_r) = \sqrt{\sum_{i=1}^r \sin(\theta_i)^2}$ (with $\theta_i$ the principal angles between $\U_r$ and $ \check \U_r$) is the distance between the subspaces $\U_r$ and $\check \U_r$.
\end{lemma}
\begin{proof}
The result follows from the $\sin \Theta$ bound~\eqref{eq:wedin2} in Wedin's theorem. Let us verify the assumptions. From the assumptions we know that $\sigma_{r} \geq \delta$ and $\check \sigma_{r} \geq \delta $. If Wedin's theorem does not apply, condition~\eqref{eq:wedin_assumption1} is not satisfied
and neither is it satisfied with the roles of $Y$ and $\check Y$ reversed.
In that case, since there exists no $\delta>0$ such that~\eqref{eq:wedin_assumption1} holds, one must have $\sigma_i = \check \sigma_{j}$, for some $i \leq r$, $j \geq r+1$, and $\check \sigma_l = \sigma_{m}$, for $l \leq r, m \geq r+1$.
	However, since the singular values are ordered decreasingly, this gives:
\[\sigma_{m} \leq \sigma_i = \check \sigma_{j} \leq \check \sigma_l = \sigma_{m}, \]
which implies that there exists $i\leq r$ and $m\geq r+1$ such that
\[\sigma_{m} = \sigma_i = \check \sigma_{j} = \check \sigma_l. \]
This is a contradiction with $\sigma_r - \sigma_{r+1} \geq \delta$ and $\check \sigma_r - \check \sigma_{r+1} \geq \delta$.
Therefore, these conditions guarantee that Wedin's theorem applies.
\end{proof}
In the next lemma we combine the previous bound with the Lipschitz continuity of $\Phi$. 
\begin{assumption}[Lipschitz continuity of the features] \label{assumption:lipschitz}
For Problem~\eqref{eq:p1}, there exists $L_{\Phi} \geq 0$ such that for any $X_{k}, X_{k+1}$ produced by Algorithm~\ref{algo:am-simple},
$\fronorm{\Phi(X_{k+1}) - \Phi(X_k)} \leq L_{\Phi} \fronorm{ X_{k+1} - X_k}$. For Problem~\eqref{eq:p1kernel}, there exists $L_{K}\geq 0$ such that $\fronorm{\K(X_{k+1},X_{k+1}) - \K(X_k,X_k)} \leq L_{\Phi} \fronorm{ X_{k+1} - X_k}$.
\end{assumption}
If we assume that the sequence $\left(X_k\right)_{k\in \mathbb{N}}$ is bounded, which we do in the main result of this section (Theorem~\ref{thm:KL}), then it is sufficient for the features and kernel to be locally Lipschitz in order for~\aref{assumption:lipschitz} to hold. We also note that if the sublevel set $\{	(X,\U) \in \M: f(X,\U) \leq f(X_0,\U_0)\}$ is bounded, then the iterates are contained in a bounded set since Algorithm~\ref{algo:am-simple} is a descent method.
\begin{lemma} \label{lemma:normUnormX}
Let~\aref{assumption:sigma} and~\aref{assumption:lipschitz} hold. It follows that 
\[\dist( \U_{k},\U_{k+1})^2 \leq \frac{2 L_\Phi^2}{\delta^2}  \fronorm{X_{k+1} - X_k}^2.  \]
\end{lemma}
\begin{proof}
By definition of $\U_{k}$, Lemma~\ref{lemma:wedin} ensures that 
\begin{equation}
\dist( \U_{k},\U_{k+1})^2 \leq \frac{2}{\delta^2} \fronorm{\Phi(X_{k+1}) - \Phi(X_k)}^2. 
\end{equation}
Indeed, \(
\U_k = \texttt{truncate-svd}(\Phi(X_{k}))\) is composed of the $r$ first left singular vectors of $\Phi(X_{k})$. They are uniquely defined due to~\aref{assumption:sigma}.
The result then follows from the Lipschitz continuity of $\Phi$.
\end{proof}
This lemma allows us to show the following crucial result.

\begin{lemma}[Sufficient decrease property] \label{lemma:suffDecr-1}
Assume that~\aref{assumption:sigma} and~\aref{assumption:lipschitz} hold at $X_k$ and $X_{k+1}$. Then, there exists $\rho_1 > 0$, independent of $k$, such that the iterates of Algorithm~\ref{algo:am-simple} satisfy
\begin{equation}
 f(X_k,\U_k) - f(X_{k+1},\U_{k+1}) \geq \rho_1\dist\Big( (X_k,\U_k),(X_{k+1},\U_{k+1})\Big)^2. 
 \label{eq:suff_decrease_rho1}
 \end{equation}
\end{lemma}
\begin{proof}
From the Armijo decrease of Lemma~\ref{lemma:step_bound}, 
\begin{equation}
f(X_k,\U_k) - f(X_{k+1},\U_k) \geq   \dfrac{\beta}{\alpha_0}\fronorm{X_k  - X_{k+1}}^2,
\end{equation}
where $\alpha_0$ is the largest step allowed by the backtracking. Set $M^2 = 2L_\Phi^2/\delta^2$.
Using that $f(X_{k+1},\U_{k+1}) \leq f(X_{k+1}, \U_k)$, we get
\begin{align}
 f(\U_k, X_k) - f(\U_{k+1}, X_{k+1}) &\geq f(\U_k, X_k) - f(\U_{k}, X_{k+1}) \\
 &\geq \dfrac{\beta}{\alpha_0} \fronorm{X_{k+1} - X_k}^2   \\ 
 &= \dfrac{\beta}{\alpha_0(1 + M^2)} (1 + M^2) \fronorm{X_{k+1} - X_k}^2  \\
 &\geq  \dfrac{\beta}{\alpha_0(1 + M^2)}\left( \fronorm{X_{k+1} - X_k}^2 + \dist^2( \U_k,\U_{k+1}) \right)
 \end{align} 
where the last inequality comes from Lemma~\ref{lemma:normUnormX}. This establishes~\eqref{eq:suff_decrease_rho1}
with $\rho_1 := \dfrac{\beta}{\alpha_0(1 + M^2)}$. 
\end{proof}

We now show convergence of the gradient norms to zero for Algorithm~\ref{algo:am-simple}. 
\begin{corollary}[Global convergence for Algorithm~\ref{algo:am-simple}]
Set $\varepsilon_x = 0$,	for any starting point $z_0 = (X_0,\U_0)\in \M$, Algorithm~\ref{algo:am-simple} applied to~\eqref{eq:p1} or~\eqref{eq:p1kernel} produces a sequence $\Big(X_{k}, \U_{k}\Big)_{k\in \mathbb{N}}$ such that 
\begin{equation}
\lim_{\substack{k\to\infty }}\fronorm{ \grad f(X_{k}, \U_{k})}=0.
\label{eq:grad_to_zeroKL}
\end{equation}
\end{corollary}
\begin{proof}
Using Lemmas~\ref{lemma:grad_lower_bound-1} and~\ref{lemma:suffDecr-1} gives, 
 \begin{align}
 f(X_k,\U_k) - f(X_{k+1},\U_{k+1}) &\geq \rho_1\dist\Big( (X_k,\U_k),(X_{k+1},\U_{k+1})\Big)^2\\
  &\geq \rho_1/\rho_2 \fronorm{ \grad f(X_{k+1},\U_{k+1})}^2.
 \end{align}
 The telescopic sum is bounded above by $f(z_0)$ independently of $k$ since $f$ is nonnegative, which ensures $\lim_{k\to\infty}  \fronorm{ \grad f(X_{k+1},\U_{k+1})} = 0$. 
\end{proof}
The bounds in Lemmas~\ref{lemma:grad_lower_bound-1} and~\ref{lemma:suffDecr-1} are standard and hold for most descent methods. The values $\rho_1, \rho_2$ depend on the specifics of the algorithm used~\cite{Bolte2013}.

 We now define the Kurdyka-Lojasiewicz inequality on Riemannian manifolds, which was already introduced in~\cite{hosseini2015convergence}. 
\begin{definition}[The Kurdyka-Lojasiewicz inequality] A locally Lipschitz function $f: \M \to \R$ satisfies the Kurdyka-Lajasiewicz inequality at $x \in \M$ iff there exist $\eta \in (0,\infty)$, a neighbourhood $V \subset \M$ of $x$, and a continuous concave function $\kappa : [0,\eta] \to [0,\infty [ $ such that 
\begin{itemize}
\item[•] $ \kappa(0) = 0$,
\item[•] $\kappa$ is continuously differentiable on $(0,\eta)$, 
\item[•] $\kappa' > 0$ on $(0,\eta)$,
\item[•] For every $y \in V$ with $f(x) < f(y) < f(x) + \eta$, we have 
$$
\kappa'(f(y) - f(x)) \norm{\grad f(y)} \geq 1.
$$
\end{itemize}
If $f$ satisfies the KL inequality at every point $x\in \M$ we call $f$ a KL function. 
\label{def:KL}
\end{definition}

\begin{lemma}
Let $\{ a_k \}_{k\in \mathbb{N}}$ be a sequence of nonnegative numbers. If $\dis \sum_{k = 1}^\infty \dfrac{a_k^2}{a_{k-1}}$ converges, then $\dis \sum_{k = 1}^\infty a_k$ converges as well.
\label{lemma:series_convergence}
\end{lemma} 
\begin{proof}
This is a standard result. A proof can be found in~\cite[Lemma 4.1]{de2016new}.
\end{proof}

\begin{theorem}\label{thm:KL}
Assume that Algorithm~\ref{algo:am-simple} is applied to problem~\eqref{eq:p1} or~\eqref{eq:p1kernel}, for case study~\ref{example:variety} or~\ref{example:clusters}, generates a bounded sequence $(X_k,\U_k)_{k\in \mathbb{N}}$. If~\aref{assumption:sigma} and~\aref{assumption:lipschitz} hold, then, the sequence has finite length, that is
\begin{equation}
\sum_{k=1}^\infty \dist\Big((X_k,\U_k) ,(X_{k+1},\U_{k+1}) \Big) <\infty.
\label{eq:finite_length}
\end{equation}
Therefore $\left(X_k,\U_k\right)_{k\in \mathbb{N}}$ converges to a unique point $(X_*,\U_*)$, which is a critical point of $f$ on $\Mcal$.
\end{theorem}
\begin{proof}
For case-studies~\ref{example:variety} and~\ref{example:clusters}, the feature map or kernel is an algebraic or exponential function. These functions are known to be KL functions~\cite{Bolte2013}, so the cost function $f$ is a KL function (Definition~\ref{def:KL}). 
For convenience, we write $z_k = (X_k,\U_k)$. Since the sequence $(z_k)_{k\in \mathbb{N}}$ is bounded, there is a subsequence $(z_{k_q})_{q\in \mathbb{N}}$ which converges to some $\bar{z}\in\M$. Let $\omega(z_0)$ denote the set of limit points for some starting point $z_0$. The set $\omega(z_0)$ is bounded by assumption and clearly closed, therefore it is compact.  We want to show that $\omega(z_0)$ is a singleton, i.e. $\omega(z_0)= \{\bar{z}\}$. The function $f$ is continuous, which implies  $\lim_{q\to \infty} f(X_{k_q},\U_{k_q}) = f(\bar{z})$. Since $f(z_{k_q})_{q\in \mathbb{N}}$ is non-increasing, the function $f$ is also constant on $\omega(z_0)$. Since $f$ is a KL function, for every point $z \in \omega(z_0)$, there exists a neighbourhood $V_z$ of $z$ and a continuous concave function $\kappa_z: [0,\eta_z]\rightarrow [0,\infty [$ of class $C^1$ on $]0,\eta_z[$ with $\kappa_z(0)= 0$, $\kappa'_z >0$ on $]0,\eta_z[$ such that, for all $y \in V(z)$ with $f(z) < f(y) < f(z) + \eta_z$, we have 
\begin{equation}
\kappa_z'(f(y) - f(z)) \fronorm{\grad f(y)} \geq 1.
\end{equation}
By compactness of $\omega(z_0)$, we find a finite number of points $\bar{z}_1, \dots, \bar{z}_p \in \omega(z_0)$ such that $\dis\cup_{i=1}^p V_{\bar{z}_i}$ covers $\omega(z_0)$. We choose $\varepsilon >0$, such that $V:= \{z \in \M: \dist\big(z,\omega(z_0)\big) < \varepsilon\}$ is contained in $ \cup_{i = 1}^p V_{\bar{z}_i}$. Then, we set $\eta = \min_{i=1,\dots,p}\eta_{\bar{z}_i}$, $\kappa'(t) = \max_{i = 1,\dots, p} \kappa'_{\bar{z}_i}(t) $ and $\kappa(t) = \int_0^t \kappa'(\tau)d\tau$. We claim that for every $z \in \omega(z_0)$, and $y \in V$, with $ f(z) < f(y) < f(z) + \eta$, we have 
\begin{equation}
\kappa'(f(y) - f(z)) \fronorm{\grad f(y)} \geq 1.
\end{equation}
Indeed, there exists some $\bar{z}_i$ such that $y \in V_{\bar{z}_i}$. Then, from the definition of $\kappa'$ and the fact that $f$ is constant on $\omega(z_0)$,
\begin{equation}
\kappa'(f(y) - f(z)) \fronorm{\grad f(y)} \geq \kappa'_{\bar{z}_i}(f(y) - f({\bar{z}_i})) \fronorm{\grad f(y)} \geq 1.
\end{equation}
For $\eta >0$ given above,  there exists $k_0 $ such that for all $ k>k_0$,
\begin{equation}
 f(z_k) < f(\bar{z}) +\eta.  
\end{equation}
 By definition of the accumulation points,  there exists $ k_1 $ such that  for all $ k>k_1$,
\begin{equation}
\dist(z_k,\omega(z_0)) <\varepsilon.
\end{equation}
Since $\sigma_r (\Phi(X))> \sigma_{r+1}(\Phi(X))$ for any $X$ such that $(X,\U)\in \omega(z_0)$~(\aref{assumption:sigma}), by continuity of the singular values, there exists $\bar{\delta}>0$ such that for all points $z_k$ satisfying $\dist(z_k,\omega(z_0)) <\bar{\delta}$, we have $ \sigma_r (\Phi(X_k))> \sigma_{r+1}(\Phi(X)_k)$. Again by definition, there exists $k_2$ such that for all $ k>k_2$, 
\begin{equation}
\dist(z_k,\omega(z_0)) <\bar{\delta}.
\end{equation}
For $k>l= \max\{k_0, k_1,k_2\}$, we have
\begin{equation}
\kappa'(f(z_k) - f(\bar{z})) \fronorm{\grad f(z_k) } \geq 1.
\end{equation}
Using  $\fronorm{\grad f(z_{k+1})} \leq \rho_2 \dist\Big(z_k ,z_{k+1} \Big)$ (Equation~\eqref{eq:rho2-1}), gives
\begin{equation}
\kappa'(f(z_k) - f(\bar{z})) \geq \dfrac{1}{\rho_2 \dist\Big(z_{k-1},z_k \Big)}.
\end{equation}
Concavity of $\kappa$ gives
\begin{equation}
\kappa\Big(f(z_k) - f(\bar{z})\Big) - \kappa\Big(f(z_{k+1}) - f(\bar{z})\Big) \geq \kappa'\Big(f(z_k) - f(\bar{z})\Big) \Big(f(z_k) - f(z_{k+1})\Big).
\end{equation}
Since $k>l\geq k_2$, we have that $\rho_1 \dist^2\Big(z_k,z_{k+1} \Big) \leq f(z_k) - f(z_{k+1})$ by Equation~\eqref{eq:suff_decrease_rho1} , 
\begin{equation}
\kappa\Big(f(z_k) - f(\bar{z})\Big) - \kappa\Big(f(z_{k+1}) - f(\bar{z})\Big) \geq  \dfrac{1}{\rho_2  \dist\Big(z_{k-1},z_k \Big)} \rho_1  \dist^2\Big( z_k,z_{k+1} \Big),
\end{equation}
and so
\begin{equation}
\dfrac{ \dist^2\Big(z_k,z_{k+1} \Big)}{ \dist\Big(z_{k-1},z_k\Big)} \leq \dfrac{\rho_2}{\rho_1}\kappa\Big(f(z_k) - f(\bar{z})\Big) - \kappa\Big(f(z_{k+1}) - f(\bar{z})\Big)\\
\label{eq:736}
\end{equation}
For any $N>l$, we sum~\eqref{eq:736} for all $l\leq k \leq N$, using that the right hand side is a telescopic sum,
\begin{align}
\nonumber
\sum_{k \geq l}^N \dfrac{ \dist^2\Big(z_k ,z_{k+1} \Big)}{\dist\Big(z_{k-1} ,z_k \Big)} &\leq \sum_{k \geq l}^N \dfrac{\rho_2}{\rho_1}  \Big[ \kappa\Big(f(z_k) - f(\bar{z})\Big) - \kappa\Big(f(z_{k+1}) - f(\bar{z})\Big)\Big] \\\nonumber
&\leq  \dfrac{\rho_2}{\rho_1}\Big[   \kappa\Big(f(z_l) - f(\bar{z})\Big) - \kappa\Big(f(z_{N}) - f(\bar{z})\Big)\Big] \\\nonumber
&\leq  \dfrac{\rho_2}{\rho_1}\Big[   \kappa\Big(f(z_l) - f(\bar{z})\Big) - \kappa\Big(f(\bar{z}) - f(\bar{z})\Big)\Big] \\
&=   \dfrac{\rho_2}{\rho_1}   \kappa\Big(f(z_l) - f(\bar{z})\Big),
\label{eq:740}
\end{align}
where we used that $f(\bar{z}) \leq f(z_N)$ , $\kappa$ is increasing and $\kappa(0)=0$. Letting $N\to \infty$ in~\eqref{eq:740}, we deduce that the left-hand side of~\eqref{eq:740} converges.
By Lemma~\ref{lemma:series_convergence}, $\sum_{k \geq l}^\infty \dist\Big(z_k,z_{k+1} \Big) $ also converges and therefore  
\begin{equation}
\sum_{k=1}^\infty  \dist\Big(z_k ,z_{k+1} \Big) < \infty. 
\end{equation}
This concludes the proof and shows finite length of the sequence of iterates, which implies convergence of the Cauchy sequence $\big(X_k, \U_k\big)_{k\in \mathbb{N}}$ to a unique point $(X_*, \U_*)$. 
\end{proof}

\input{framework.tex}

\section{Numerical experiments}
\label{sec:numerics}
In this section we validate our approach with numerical results on randomly generated test problems. We also compare the performances of the different algorithms we propose.

\subsection{Implementation of the algorithms}
Let us describe the implementation of the different algorithms  and variants that are considered. \texttt{Altmin1} is a first-order version of alternating minimization (Algorithm~\ref{algo:am}) which uses gradient descent with Armijo linesearch to solve subproblem~\eqref{eq:s1}. It uses the monomial kernel (Equation~\eqref{eq:mono_kernel}). The degree of the kernel that gives the best results is almost always $d=2$. We set the constant $c=1$ in the monomial kernel. 
 In \ttt{Altmin2}, a second-order trust region method using the exact Hessian is applied to the minimization of~\eqref{eq:s1}. This is the only difference with \ttt{Altmin1}. The default values for the parameters of Algorithm~\ref{algo:am} and the Gaussian and monomial kernels are presented in the table below. 
 
\begin{center}
 \begin{tabular}{|c|c|c|c|}
 \hline 
  Parameter & Default value & Parameter & Default value \\ 
 \hline 
 $\varepsilon_x$, $\varepsilon_u$  & $10^{-6}$ & $c$ in Equation~\eqref{eq:mono_kernel} & 1 \\ 
 \hline 
  $\varepsilon_{x,k}$ & Equation~\eqref{eq:greedy_decrease}  & $\alpha_0$ in Algorithm~\ref{algoLS} & 2 \\  
   \hline 
  $\sigma$ in Equation~\eqref{eq:gaussian_kernel} & 2.5 & $\tau$ in Algorithm~\ref{algoLS} & 0.5 \\  
 \hline 
  $d$ in Equation~\eqref{eq:mono_kernel} & 2 & $\beta$ in Algorithm~\ref{algoLS} & $10^{-4}$ \\  
 \hline 
 \end{tabular} 
\end{center}
 Our code is available at \url{https://github.com/flgoyens/nonlinear-matrix-recovery} in both Matlab and Python. We use the Manopt~\cite{manopt} and Pymanopt~\cite{townsend2016pymanopt} libraries for optimization on manifold solvers. The Riemannian trust-region \ttt{RTR2}, which implements Algorithm~\ref{algo:RTR}, is the corresponding Manopt solver for optimization on manifolds. We used a second-order version with the Hessian in the model and the default parameters of the solver. The maximum number of iterations is set at $500$ for \ttt{RTR2}. In Manopt, the subproblems are solved with a truncated conjugate gradient method and the final termination criterion is only a first-order condition (the norm of the gradient) which we set at $10^{-6}$ for \ttt{RTR2}. Pymanopt uses automatic differentiation and does not require to compute the derivatives by hand, while the Manopt uses finite differences if the Hessian is not given as an input.

\subsection{Test problems}

We describe the set of parameters that we want to vary and test the dependence of each algorithm with respect to these parameters.

\paragraph{Union of subspaces}
Case study~\ref{test:uos} depends on the following parameters: ambient dimension $n$,
number of subspaces, dimension of each subspace, number of points on each subspace. To generate a random union of subspaces, we place the same number of points on each subspace and take subspaces of the same dimension. We calculate a basis for a random subspace and generate each point on that subspace by taking a random combination of the columns of that basis.

\paragraph{Clusters}
For case study~\ref{example:clusters}, the parameters defining a point cloud divided in clusters in $\R^n$ are: the number of clusters, the number of points in each cluster and the standard deviation $\sigma_c$ of each cluster. We first generate random centres in $\R^n$.  We then add to each centre a cluster of points with multivariate Gaussian distribution with zero mean and covariance $\sigma_c^2 \Id$ with $\sigma_c = 0.5$.

\subsection{Testing methodology}

Throughout we say that an algorithm successfully recovers the matrix $M\in \R^{n\times s}$ if it returns a matrix $X^{output}$ such that the root mean square error (RMSE) is below $10^{-3}$,
\begin{equation}
\mathrm{RMSE}(M,X^{output}) :=\fronorm{X^{output} - M}/\sqrt{ns} \leq 10^{-3}.
\end{equation}
Our goal is to test the ability of our methods to recover the original matrix $M$. We measure the performance against an increase in difficulty of the problem for several parameters. Parameters that increase the difficulty of the recovery include: 
\begin{enumerate}
\item Reducing the number of measurements $m$;
\item Increasing the rank in the feature space. 
\end{enumerate}
In the case of unions of subspaces, for a fixed number of points, the rank of $\Phi_d(M)$ depends on the number and the dimension of the subspaces, as indicated by Proposition~\ref{prop:rank_phi_uos}. For clusters, the rank increases with the number of clusters.
The undersampling ratio is defined as 
\begin{align}\label{eq:undersampling}
\delta = \dfrac{m}{ns},
\end{align}
it is the number of measurements over the number of entries in $M$. 
We present phase transition results to numerically show which geometries can be recovered and which undersampling ratios are needed. 
Typical phase transition plots for matrix completion vary the undersampling ratio and the rank of the matrix~\cite{Tanner2013}. For union of subspaces, the rank of the feature space is difficult to control, therefore we vary the number and dimension of the subspaces. For each value of the varying parameter, we generate $10$ random matrices $M$ that follow the desired structure. We try to recover each with varying $\delta$ from $0.1$ to $0.9$ for a random initial guess. If the RMSE is below $10^{-3}$ in the maximum number of iterations allowed by the algorithm, we consider the recovery to be successful. The phase transition plots record which of the 10 random problems is solved for each configuration. In Figures~\ref{fig:phase_s} through~\ref{fig:uos_rank} the grayscale indicates the proportion of problems solved, with white $= 100\%$ of instances solved and black $=0\%$.
\subsection{Numerical results}

\subsubsection{Comparing the performance of RTR and Alternating minimization}

Figure~\ref{fig:convergence} compares the performance of \ttt{RTR2} (Algorithm~\ref{algo:RTR} using a second-order Taylor model), \ttt{Altmin1} and \ttt{Altmin2} which are first- and second-order alternating minimization (Algorithm~\ref{algo:am}). We chose a problem of matrix completion over a union of subspaces. We find that \ttt{RTR2}  has a local quadratic rate of convergence, which makes it the method of choice if we want to recover $M$ to high accuracy. Both \ttt{Altmin1} and \ttt{Altmin2} make faster progress during the early iterations; thus these methods should be considered if the required accuracy is low. We also observed that, in general, the distance to the solution $M$ is of the same order of magnitude than the gradient norm. That is, using an algorithm, such as \ttt{RTR2}, which terminates with a smaller gradient norm yields a greater accuracy for the recovery. We noticed that the first-order methods, as well as the second-order \ttt{Altmin2}, typically stall numerically when the gradient norm is around $10^{-7}$, but that is not the case for \ttt{RTR2}.

\begin{figure}[!htb]
\centering
\subfigure{
\includegraphics[width = 0.45\textwidth]{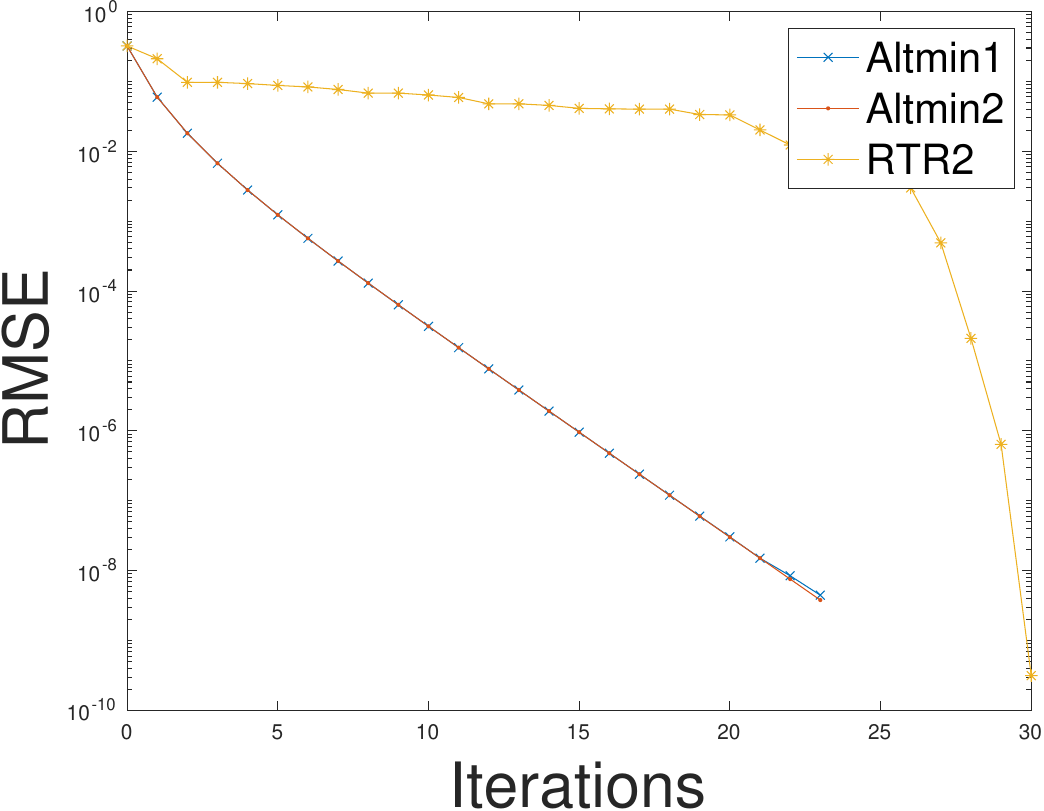}
}
\quad
\subfigure{
\includegraphics[width = 0.45\textwidth]{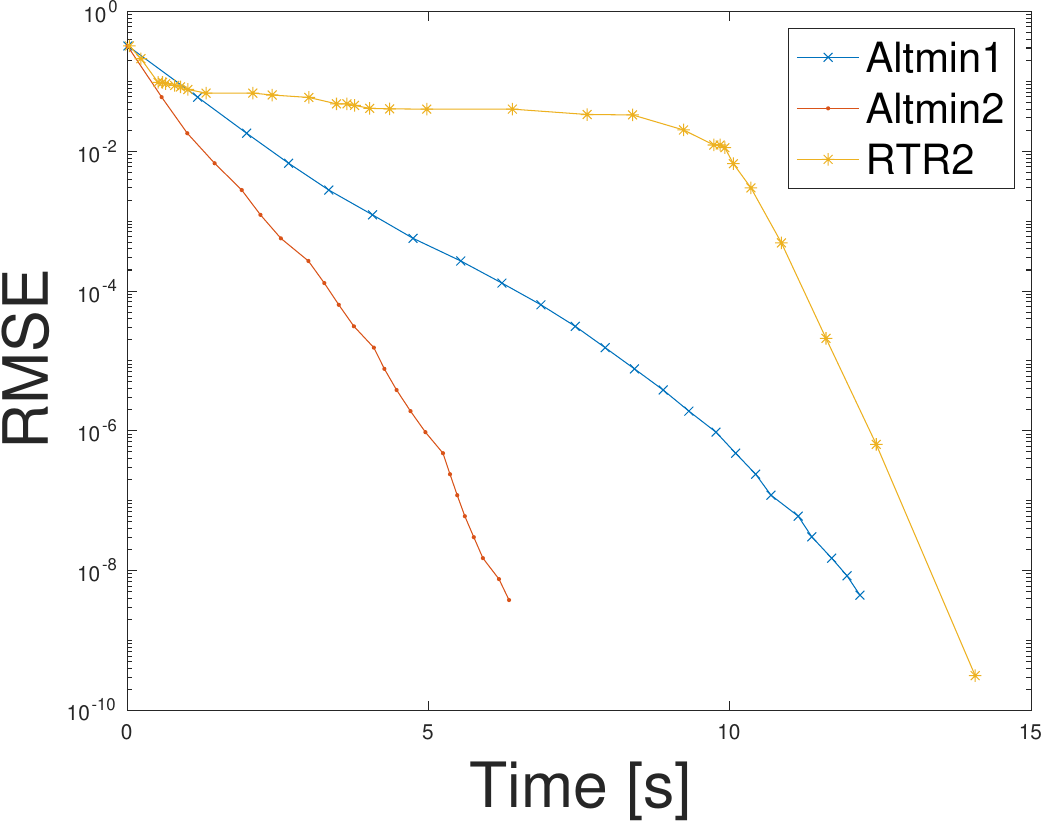}
}
\quad
\subfigure{
\includegraphics[width = 0.45\textwidth]{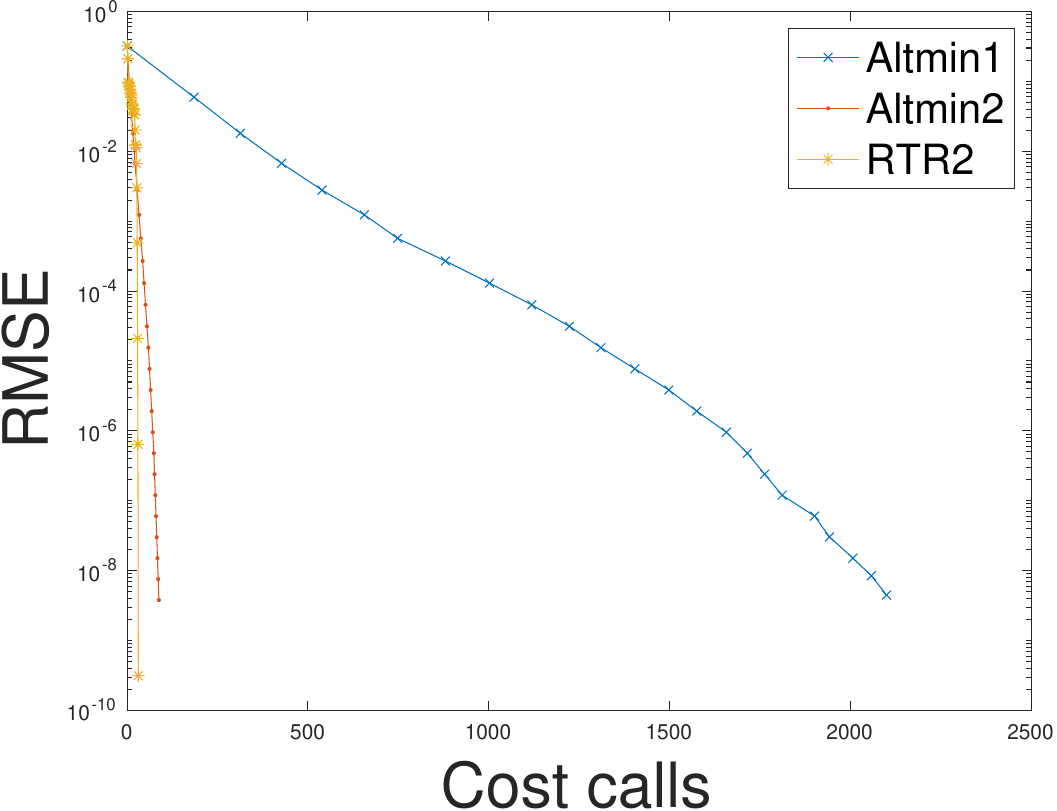}
}
\quad
\subfigure{
\includegraphics[width = 0.45\textwidth]{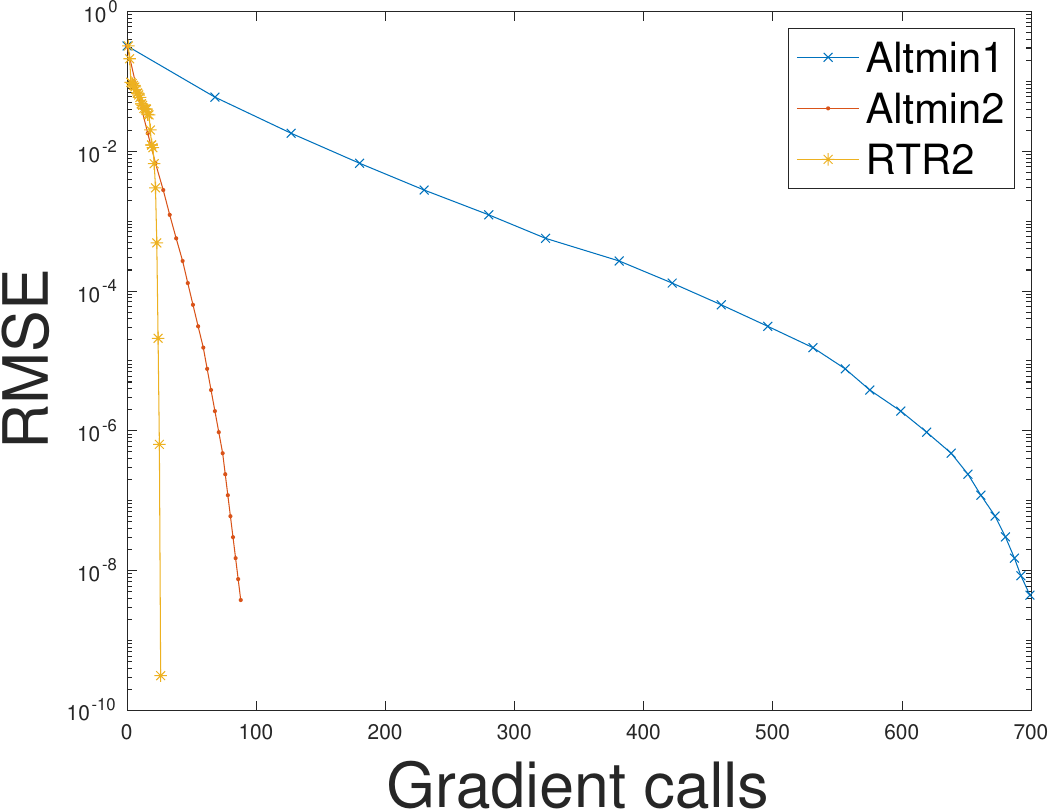}
}
\quad
\subfigure{
\includegraphics[width = 0.45\textwidth]{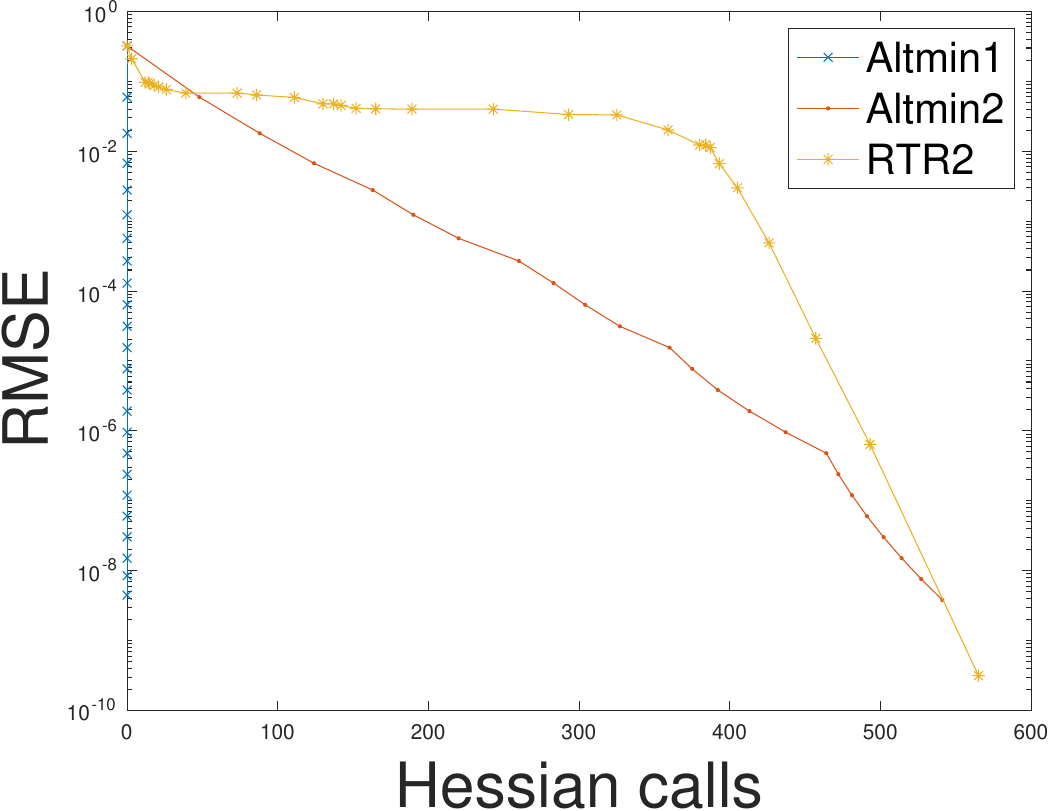}
}
\caption{Comparing alternating minimization (first-order Altmin1 and second-order Altmin2) with the Riemannian trust-region algorithm (RTR2) for a union of subspaces recovery.}
\label{fig:convergence}
\end{figure}

\subsubsection{Recovery of unions of subspaces}

We now illustrate how the parameters at play affect the recovery for data that follows a union of subspaces model.

\paragraph{Degree of the polynomial features}
Deciding which degree $d$ to use in practice requires a careful choice. Previous works limit themselves to $d=2$ and $d=3$.  This is understandable because the dimension of the features $N(n,d)$ increases exponentially with $d$, and so the dimension of the Grassmannian variable in~\eqref{eq:p1} becomes too large and the problem becomes practically intractable for even moderate values of $d$ and $n$. For example, $N(n=20, d=5) = 53130$ while $N(n=20, d=2) = 231$.  

 Therefore, the other natural option is to solve the kernel-based problem~\eqref{eq:p1kernel}, where the dimension of the feature space $N(n,d)$ does not appear explicitly and the Grassmann has dimension $s\times r$. This is attractive because, a priori, $s$ may not be as large as $N(n,d)$. However, there are important requirements on the number of samples needed to allow recovery. We need to ensure that $s \geq N- q$ where $q$ is the number of linearly independent vectors $v$ such that  $v^\top \Phi_d(M) = 0$. That is, $s$ needs to be large enough so that $\rank(\Phi_d(M))$ is not limited by $s$ but the dimension of the variety. Some analysis in~\cite{Ongie2017} shows that the number of points $s$ needed to allow recovery increases exponentially with $d$. For that reason, if $d$ is not small, it is not realistic to solve problems where $s$ is large enough to enable recovery. The monomial basis is also known to be ill-conditioned for large degrees. This gives two obstacles to the performances of these algorithms when the degree increases.

In Figure~\ref{fig:phase_s}, we solve the recovery problems using \ttt{RTR2} with an increasing number of data points, and using monomial kernels of degree one, two and three to compare the recovery that is possible for each degree.
In Figure~\ref{fig:s_uos_d1}, the degree used is $d=1$. For $n=15$, the dimension of the feature space is $N(15,1) = 16$. For a large number of data points spread over 4 subspaces of dimension 2, the rank of the monomial kernel is 9. This explains why recovery is impossible when $s\leq 9$, since the kernel is not rank deficient at the solution $M$.
In Figure~\ref{fig:s_uos_d2}, the degree used is $d=2$. For $n=15$, the dimension of the feature space $N(15,2) = 136$. For a large number of data points spread over  4 subspaces of dimension 2, the rank of the monomial kernel is 21. This explains why recovery is impossible when $s\leq 21$, since the kernel is not rank deficient at the solution $M$. 
In Figure~\ref{fig:s_uos_d3}, the degree used is $d=3$. For $n=15$, the dimension of the feature space is $N(15,3) = 816$. For a large number of data points spread over 4 subspaces of dimension 2, the rank of the monomial kernel is 37. This  explains why recovery is impossible when $s\leq 37$, since the kernel is not rank deficient at the solution $M$. We notice that the recovery is still poor for $s\geq  37$, which is likely induced by a worse conditioning of the monomial embedding. In general, $d=2$ seems to give the best results for the majority of data sets.

\begin{figure}[!htb]
\centering	
\subfigure[$d=1$]{
\includegraphics[width = 0.45\textwidth]{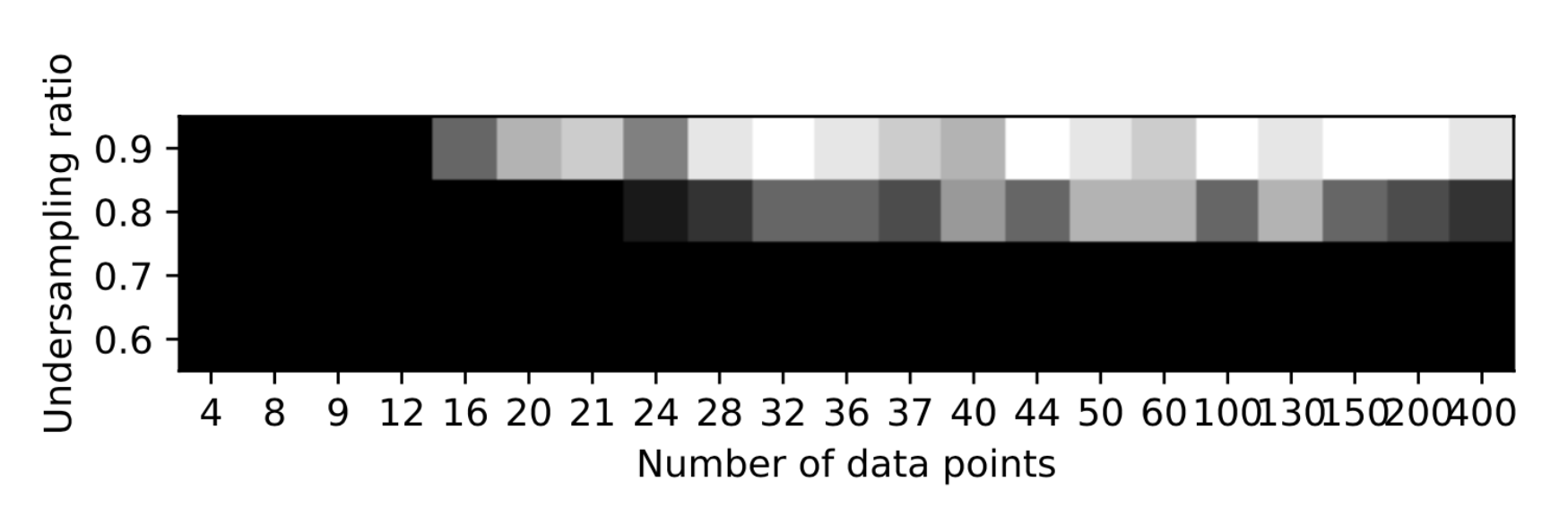}
\label{fig:s_uos_d1}
}
\quad
\subfigure[$d=2$]{
\includegraphics[width = 0.45\textwidth]{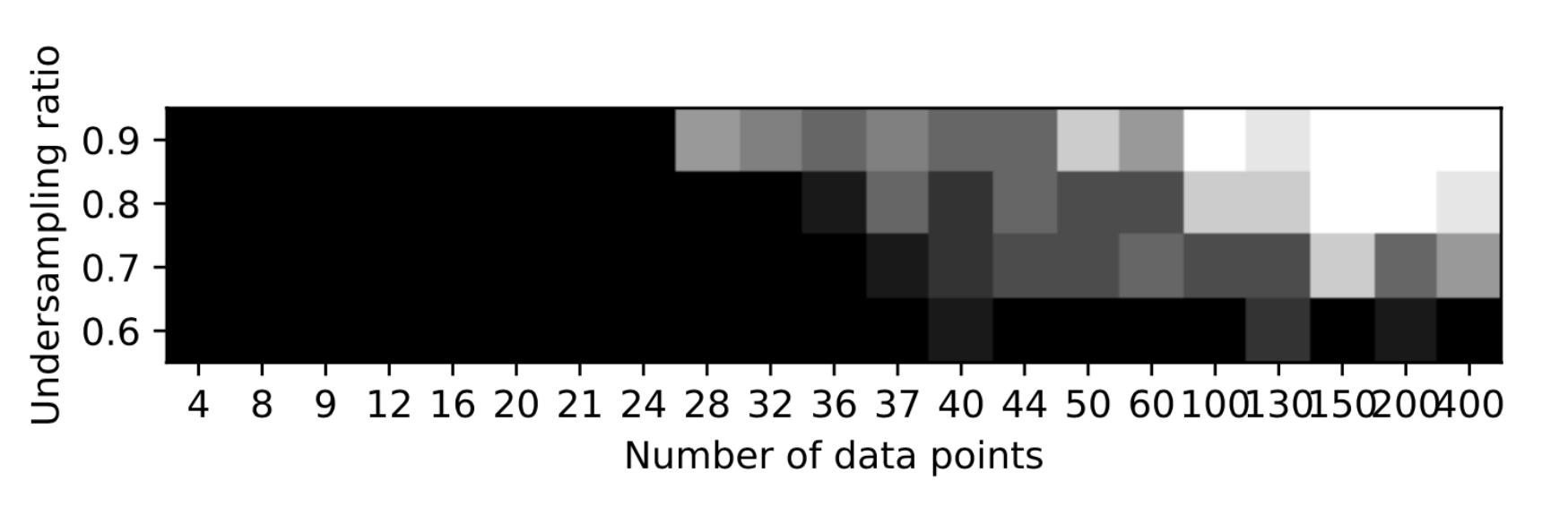}
\label{fig:s_uos_d2}
}
\quad
\subfigure[$d=3$]{
\includegraphics[width = 0.45\textwidth]{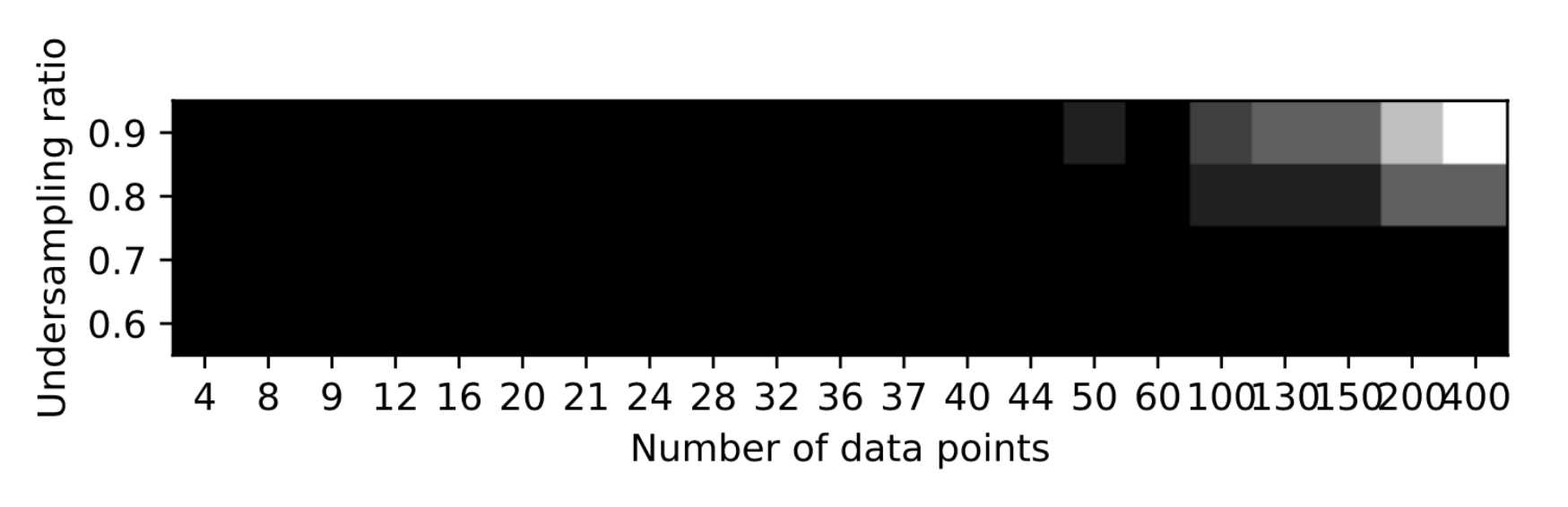}
\label{fig:s_uos_d3}
}
\caption{Phrase transition for data belonging to a union of $4$ subspaces of dimension $2$ in $\mathbb{R}^{15}$ for an increasing number of data points spread across the subspaces. Each square gives the proportion of problems solved over $50$ randomly generated problems, $\text{white}=100\%$ of instance recovered and black~$=0\%$.}
\label{fig:phase_s}
\end{figure}

The dimension of the subspaces that we aim to recover plays a role in the possibility to recover. In Figure~\ref{fig:phase_dimsub}, we increase  the dimension of the subspaces while the other parameters of the data remain fixed. For a fixed number of data points $s$, increasing the dimension of the subspaces increases the rank of the monomial features (see Proposition~\ref{prop:rank_phi_uos}), and therefore, if the dimension of the subspaces becomes too large, the recovery is compromised. For 2 subspaces of dimension smaller than 4 in $\R^{10}$, we can observe good recovery depending on the undersampling ratio.
\begin{figure}[!h]
\centering
\includegraphics[width = 0.45\textwidth]{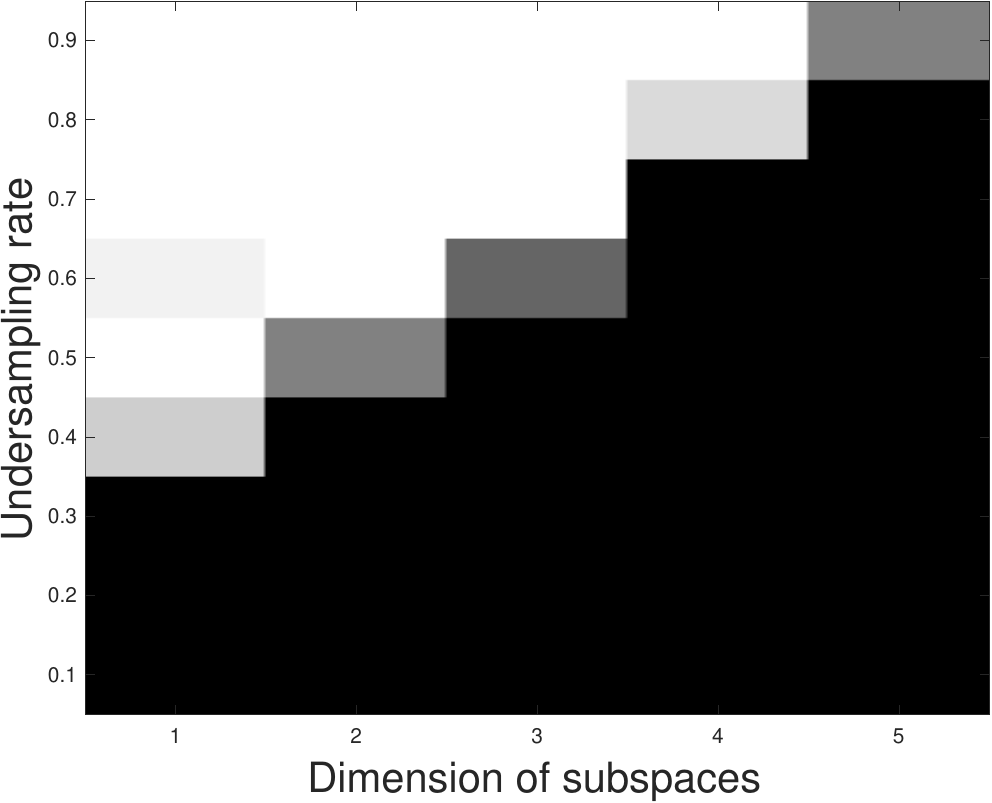}
\caption{Phrase transition for data belonging to a union of 2 subspaces of increasing dimension in $\mathbb{R}^{10}$ with 20 points on each subspace. Each square gives the proportion of problems solved over $50$  randomly generated problems, $\text{white}=100\%$ of instance recovered and black~$=0\%$, using the polynomial kernel of degree $d=2$ for the embedding.}
\label{fig:phase_dimsub}
\end{figure}

The same phenomenon is observed when the number of subspaces is increased for a fixed number of data points, see Figure~\ref{fig:phase_nsub}.

\begin{figure}[!htb]
\centering
\subfigure[Monomial kernel degree $d=1$]{
\includegraphics[width = 0.45\textwidth]{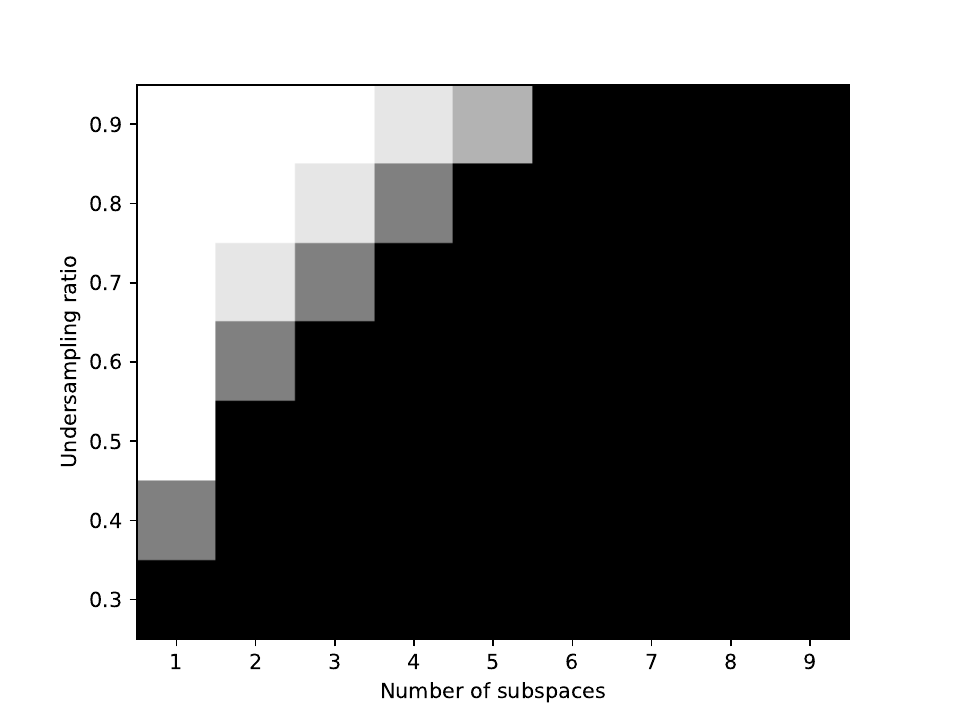}
}
\quad
\subfigure[Monomial kernel degree $d = 2$]{
\includegraphics[width = 0.45\textwidth]{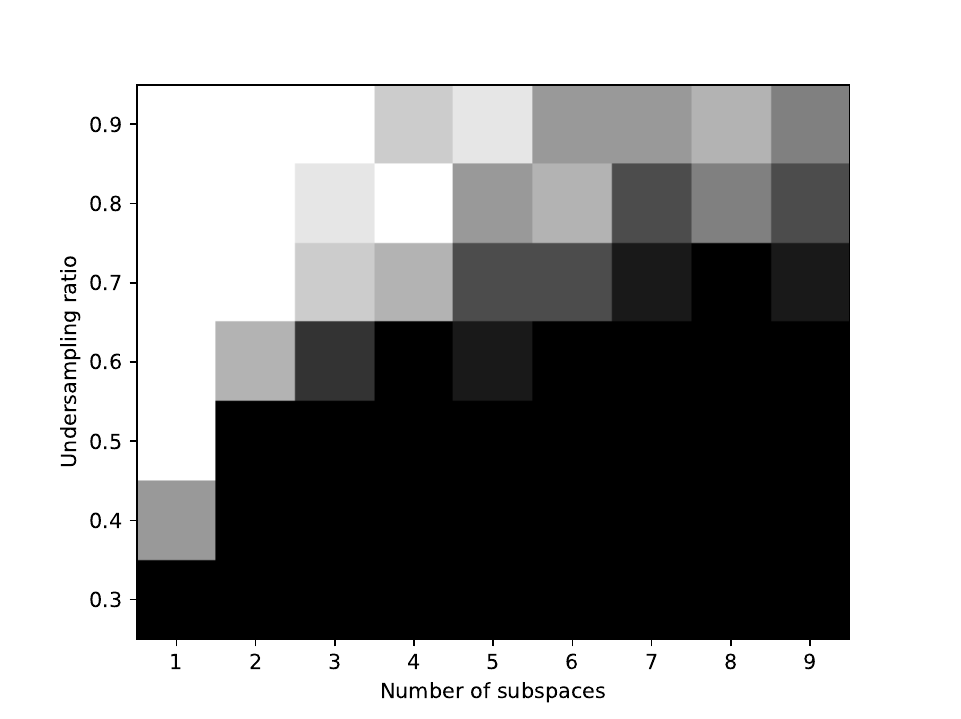}
}
\quad
\subfigure[Monomial kernel degree $d = 3$]{
\includegraphics[width = 0.45\textwidth]{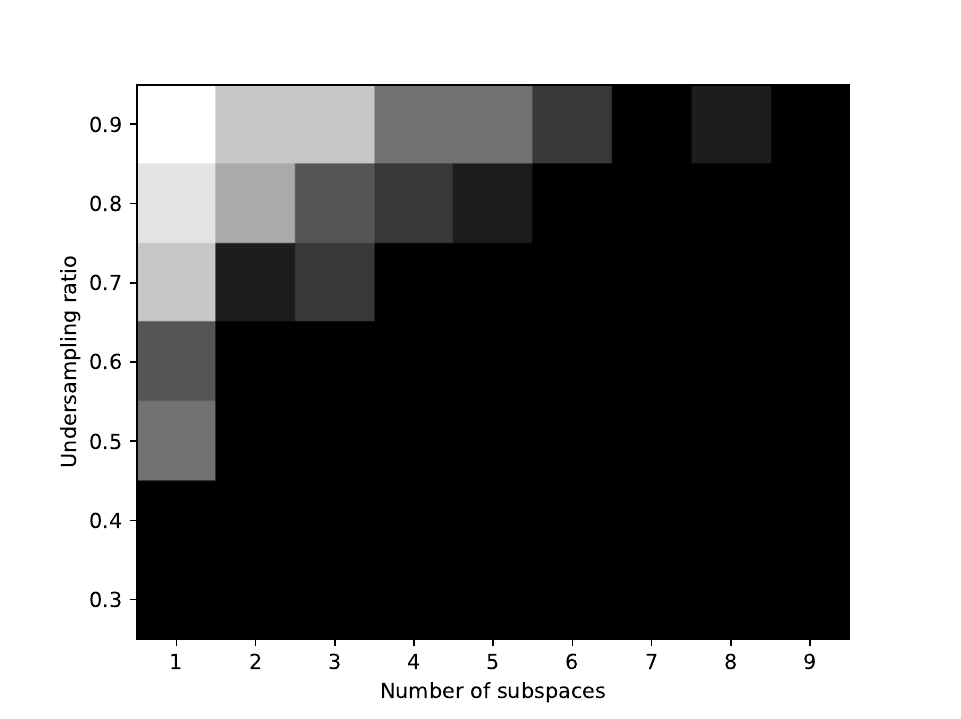}
}
\caption{Phrase transition for data belonging to a union of an increasing number of subspaces  of dimension $2$ in $\mathbb{R}^{15}$ with $150$ data points spread across the subspaces. Each square gives the proportion of problems solved over $10$  randomly generated problems, $\text{white}=100\%$ of instance recovered and black~$=0\%$.}
\label{fig:phase_nsub}
\end{figure}

\FloatBarrier

\subsubsection{Clustering with missing data}
In the case of clusters (case study~\ref{example:clusters}), there is a noise inherent to the model because the kernel at the solution is only approximately low-rank. In fact, the matrix $\K^G(M,M)$ has numerical full rank, but there is a big gap in the singular values. These matrices are notoriously difficult to recover in low rank matrix completion. For this reason the recovery error $\fronorm{M-X^*}^2$ rarely converges to high accuracy and recovering the matrix up to 2 digits of accuracy is typical. This completed matrix $X^*$ allows to do a clustering of the data points starting from missing entries. We are interested in determining when the matrix $X^*$ has the same clustering as $M$. 
We use the Rand index to measure the compatibility of two different clusterings of the same set~\cite{rand1971objective}. We can see in Figure~\ref{fig:clusters} that for 5 clusters or less, the original clustering can be recovered even though up to 40\% of the entries in the original matrix are missing.

\begin{figure}[!h]
\centering
\includegraphics[width = 0.45\textwidth]{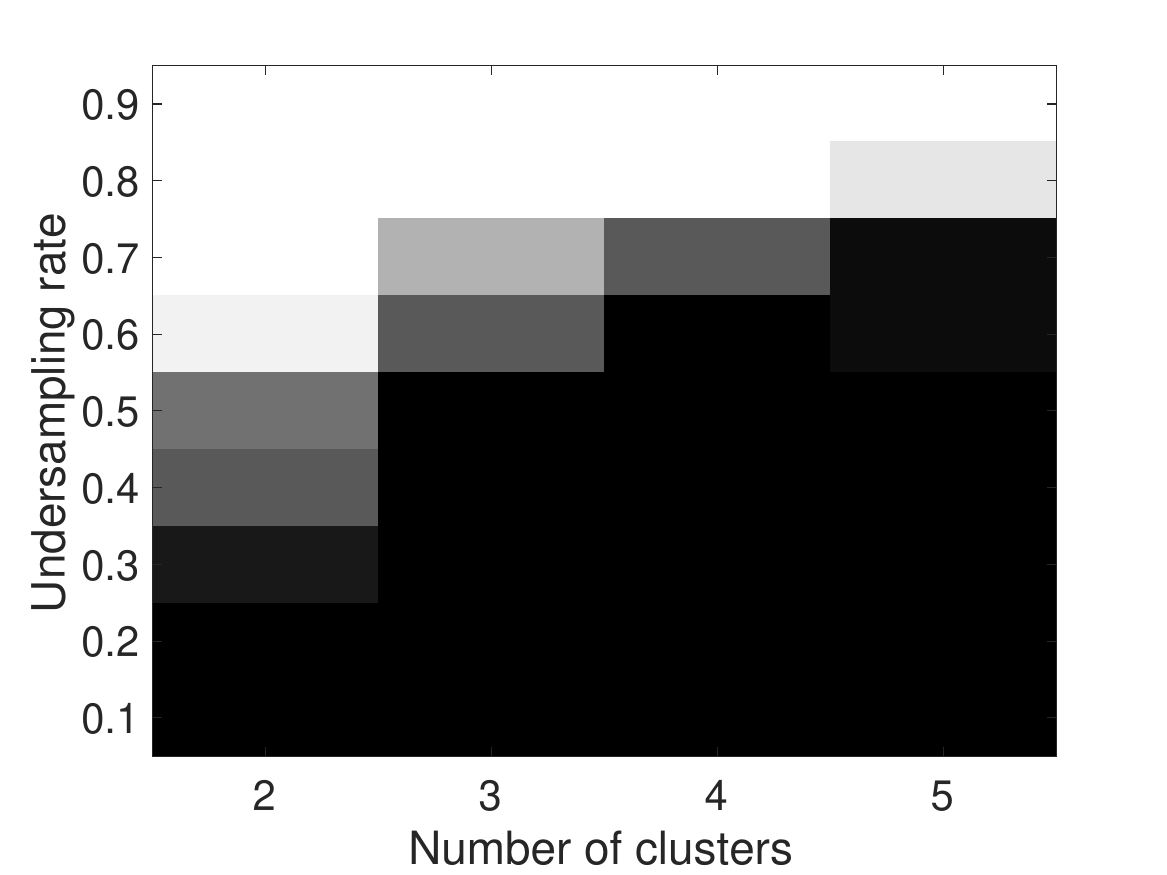}
\caption{Percentage of problems correctly clustered for different sampling rates and increasing number of clusters. 50 random instances generated in each case; $\text{white}=100\%$ of instance correctly clustered and black~$=0\%$. Clusters belong to $\R^5$ with 20 points in each cluster.}
\label{fig:clusters}
\end{figure}
\FloatBarrier

 \subsubsection{Robustness to measurement noise}
 In applications, it is common to assume some noise on the measurements, namely, $\langle A_i, M\rangle = b_i + \xi_i =:  \tilde b_i$ where $ \xi_i\in \R$ is some noise. In the following numerical test we generate white Gaussian noise, i.e. $\xi_i \sim \mathcal{N}(0,\sigma^2)$ for some variance $\sigma^2$. The problem formulation then becomes Problem~\eqref{eq:p1_noise} with $\lambda>0$ as the penalty parameter that should be tuned based on the noise level.

Estimating an appropriate value for $\lambda$ without knowledge of the noise variance $\sigma^2$ is an intricate task. The solution of~\eqref{eq:p1_noise} for $\lambda \in [0, \infty[$ represents the trade-off curve between minimization of the rank residual and minimization of the residual on the linear measurements. In practical settings, a user may be able to determine which trade-off is more meaningful for a particular application. As a general strategy, we use a scheme which increases $\lambda$ over successive calls to the solver, while warm starting each solve with the previous solution to~\eqref{eq:p1_noise}. We have found that starting with the value $\lambda=10^{-6}$ is satisfactory and we multiply $\lambda$ by a factor 10 at each iteration. Figure~\ref{fig:noise12} shows, for three different noise levels, the evolution of the solution of~\eqref{eq:p1_noise}, labelled $X^*$, as the penalty parameter $\lambda$ increases. We see that where the blue and red lines cross, the green line is still near its lowest point, that is, the solution is still minimizing the true measurement residual as well as for any other value of $\lambda$. This allows to recommend the simple strategy of choosing the value of $\lambda$ where the values of the red and blue curves are the closest (which approximates the value for which they intersect). This choice gives equal weight to the rank minimization and satisfaction of the measurements.
 
Table~\ref{table:noise} shows the accuracy of the solution $X^*$ for that choice of $\lambda$. We see that both the infeasibility ($\norm{ \A(X^*)-b }$) and the distance to the solution ($\norm{X^*-M}$) are proportional to the noise level and decreases with the later. This shows that the warm start scheme to find a good value for $\lambda$ in conjunction with Problem~\eqref{eq:p1_noise} handles the presence of noise in the measurements very well. 
 
\begin{figure}[!htb]
\centering
\subfigure[$\sigma = 10^{-2}$]{
\includegraphics[width = 0.45\textwidth]{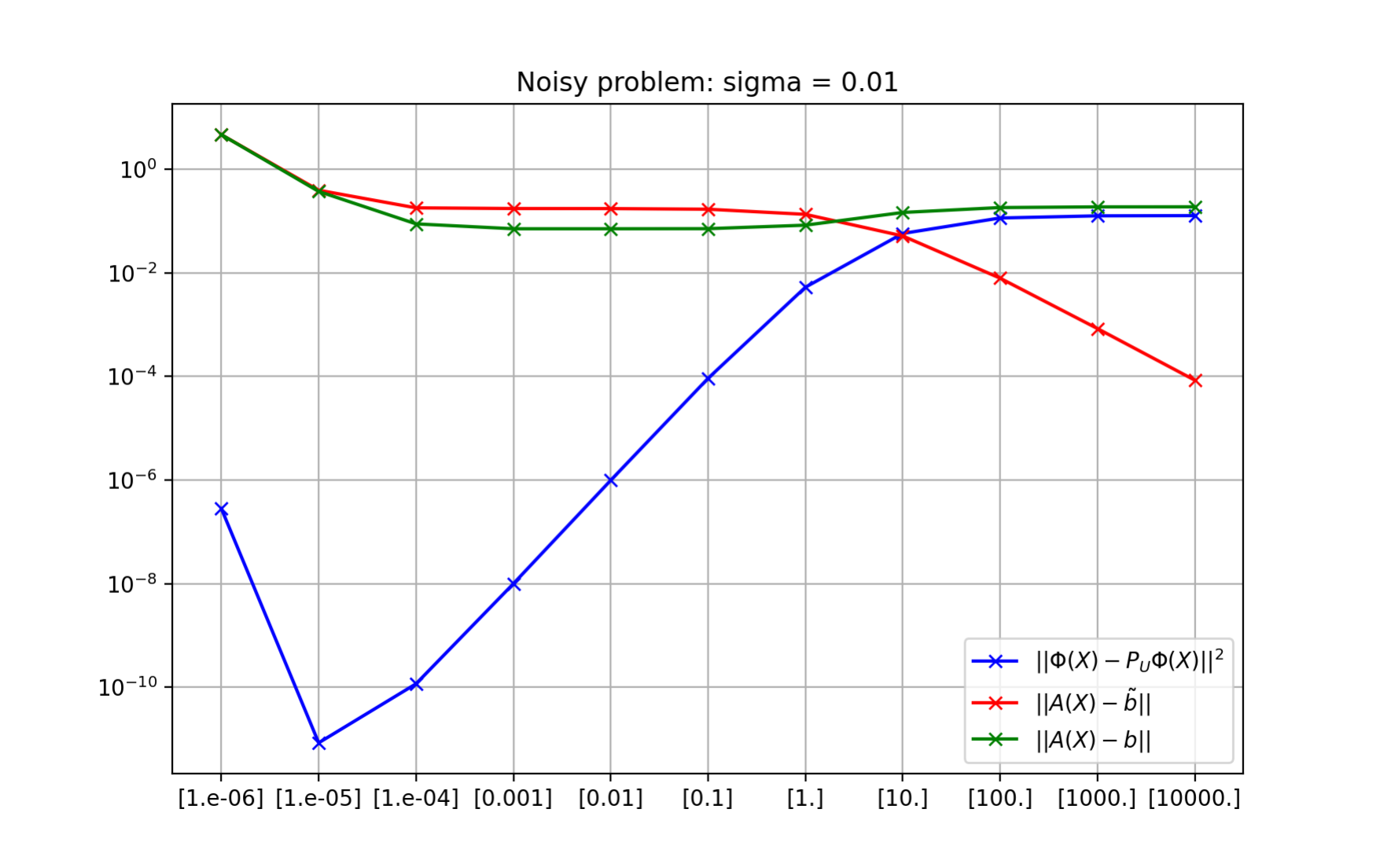}
}
\quad
\subfigure[$\sigma = 10^{-3}$]{
\includegraphics[width = 0.45\textwidth]{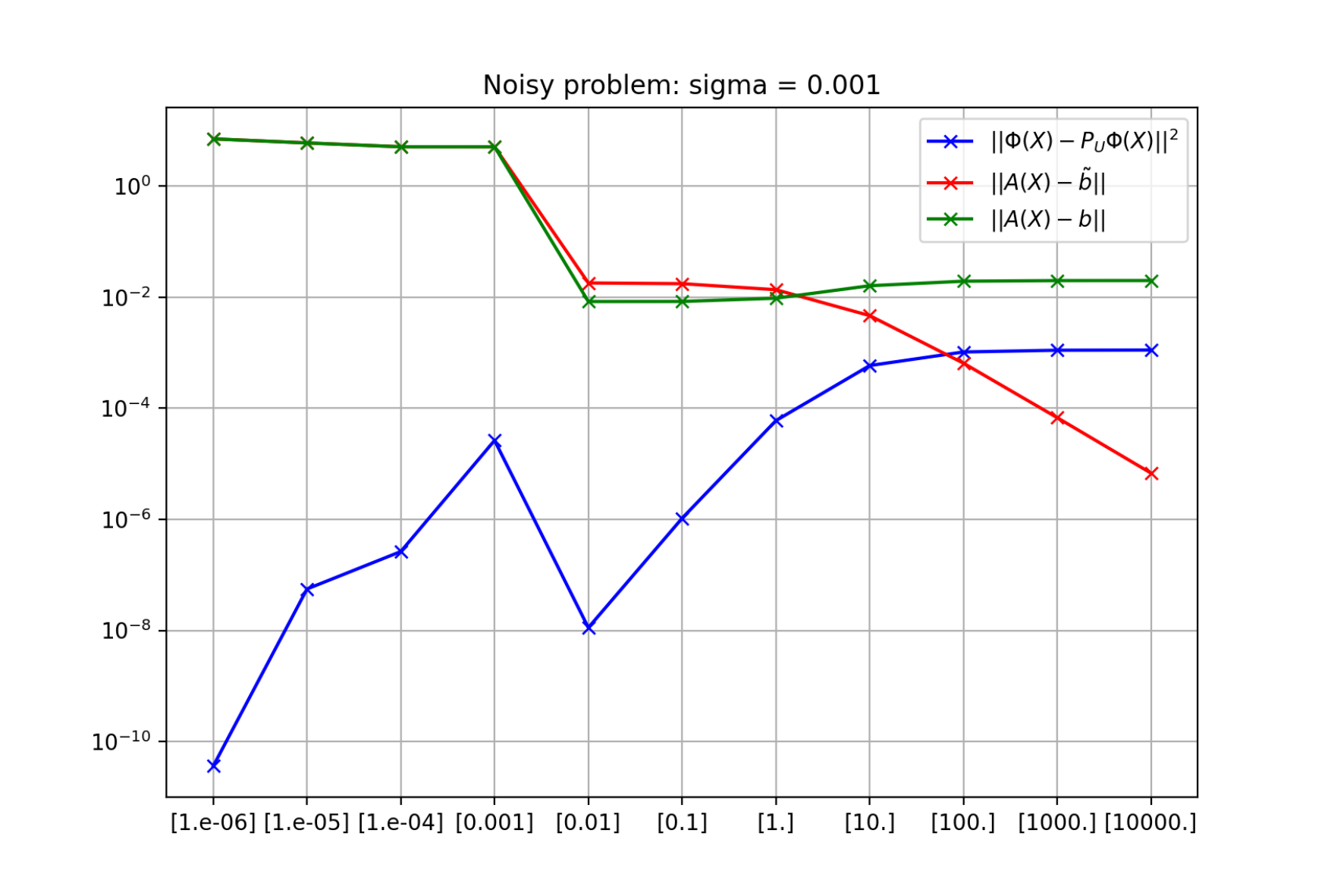}
}
\quad
\subfigure[$\sigma = 10^{-4}$]{
\includegraphics[width = 0.45\textwidth]{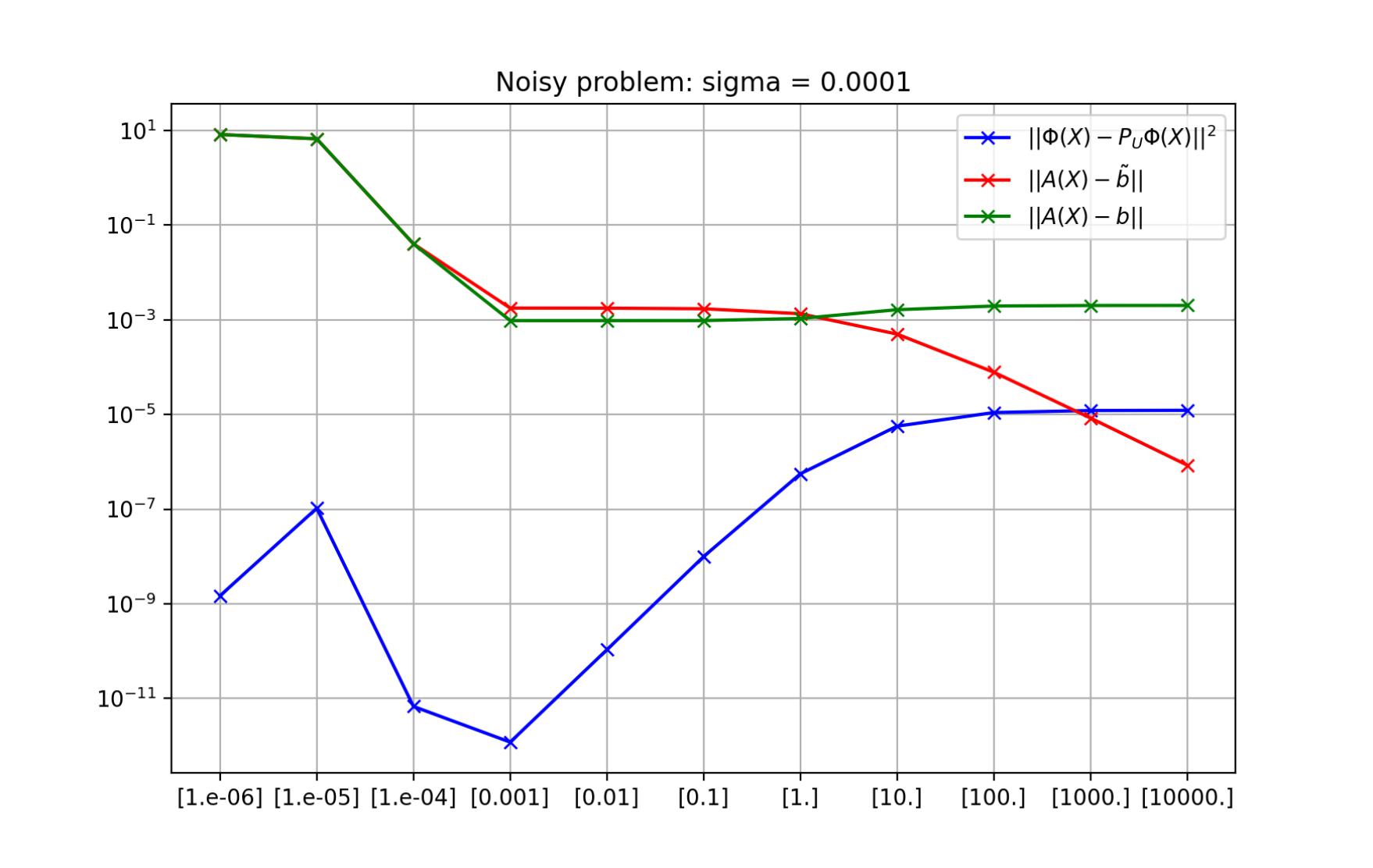}
}
\caption{Solutions for noisy problems as a function of the parameter $\lambda$ on the horizontal axis.}
\label{fig:noise12}
\end{figure}
\begin{table}
\begin{center}
\begin{tabular}{|c|c|c|c|c|}
\hline 
Standard deviation & $\norm{ \A(X^*)-\tilde b }$ & $\norm{ \A(X^*)-b }$ & $\norm{X^*-M}$ & $\norm{\A(M)-b}$ \\ 
\hline 
$\sigma = 10^{-2}$ & $2\cdot 10^{-1}$ & $7\cdot 10^{-2}$ & $8\cdot 10^{-2}$  & $0.2$\\ 
\hline 
$\sigma = 10^{-3}$ & $2\cdot10^{-2} $ & $8\cdot 10^{-3}$ & $8\cdot 10^{-3}$ & $0.02$ \\ 
\hline 
$\sigma = 10^{-4}$ & $2\cdot 10^{-3} $ & $8\cdot 10^{-4}$ & $9\cdot 10^{-4}$ & $0.002$ \\ 
\hline 
\end{tabular} 
\end{center}
\caption{Quality of the solution $X^*$ for different levels of noise $\sigma$ in the measurements.}
\label{table:noise}
\end{table}

\FloatBarrier
 
\subsubsection{Robustness to a bad estimate of the rank}
\label{sec:bad_rank_estimation}
In the case of a union of subspaces, the polynomial features are exactly low rank. Then, it is important to have an accurate upper bound on the rank. Recovery is sometimes possible if the upper bound is close to the correct value. If the estimated rank is less than the exact rank or much too large, recovery will normally fail. This intuition is guided by the cost function that we use. If the variable $\U$ is artificially constrained to be the leading singular vectors of $\Phi(X)$, that is $\U= \ttt{truncate-svd}(\Phi(X))$, then it can be substituted and the cost function in~\eqref{eq:p1} simplifies to 
\begin{equation}
\min_X \sum_{i = r+1}^{\min(N,s)} \sigma_i^2(\Phi(X)).
\end{equation}
The cost function represents the energy in the tail of the singular value decomposition, where $r$ is the estimation of the rank. In Figure~\ref{fig:uos_rank_phase}, the data belongs to a union of 2 subspaces of dimension $2$ in $\R^{15}$, with a total of $s= 150$ data points. With a kernel of degree $2$, the dimension of the feature space is $N(15,2) = 136$.

\begin{figure}[!htb]
\centering
\subfigure[Phase transition]{
\includegraphics[width = 0.45\textwidth]{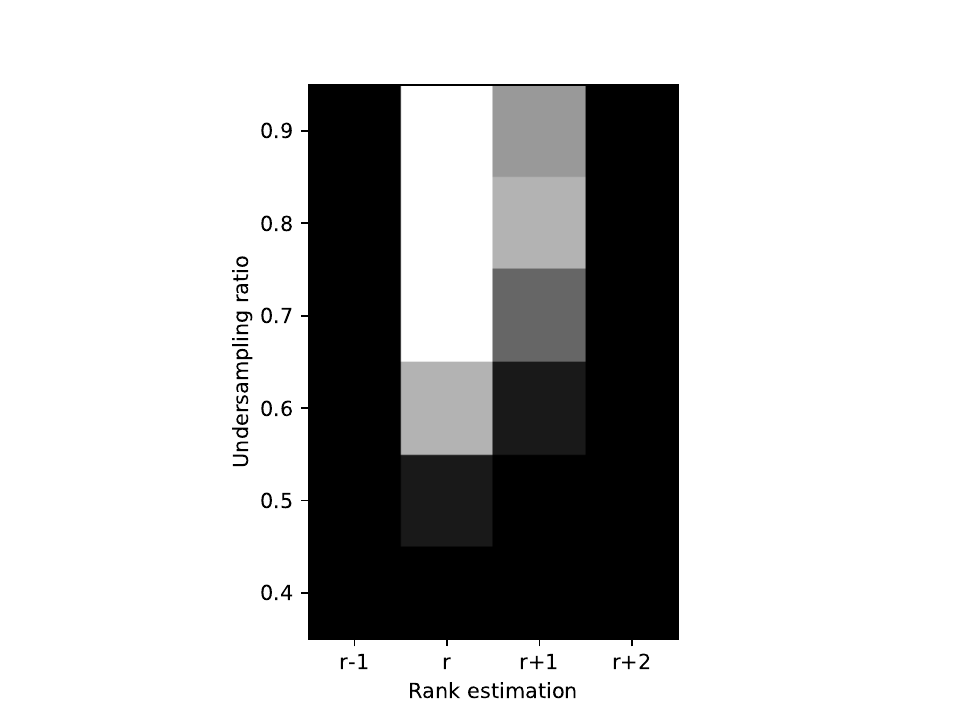}
\label{fig:uos_rank_phase}
}
\quad
\subfigure[Singular values of the monomial kernel $k_2(M,M)$]{
\includegraphics[width = 0.45\textwidth]{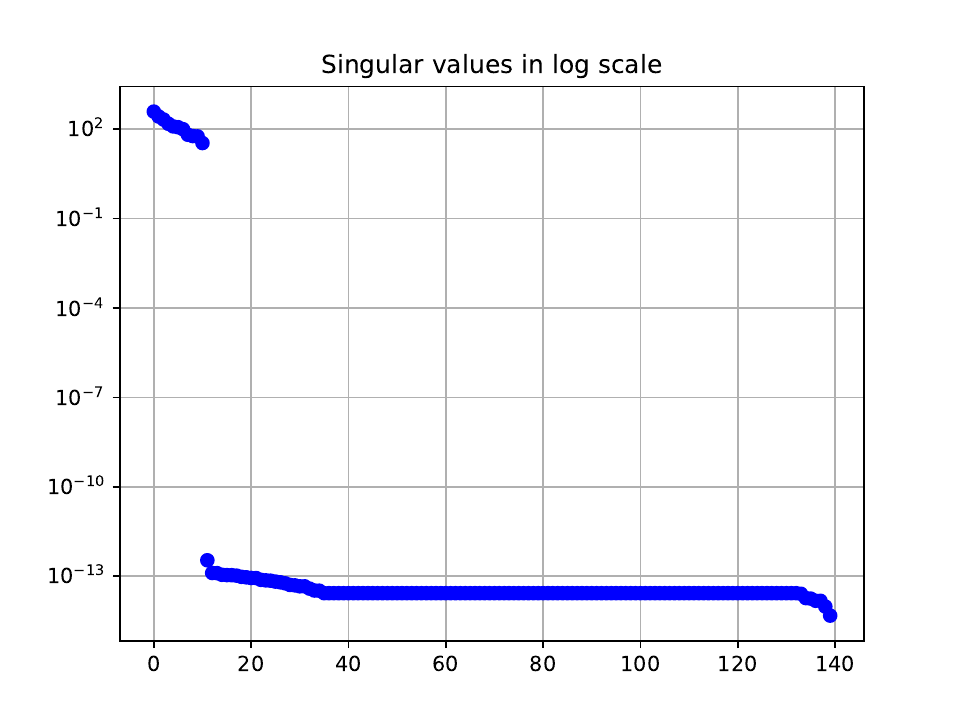}
\label{fig:uos_rank_phase_svd}
}
\caption{Impact of an incorrect estimate of the rank for the completion of a union of subspaces. Each square gives the proportion of problems solved over $50$  randomly generated problems, $\text{white}=100\%$ of instance recovered and black~$=0\%$, using the polynomial kernel of degree $d=2$ for the embedding.}
\label{fig:uos_rank}
\end{figure}

\FloatBarrier
\newcommand{\VMC}{\texttt{VMC }}
\newcommand{\Altminone}{\texttt{Altmin1 }}
\newcommand{\RTRtwo}{\texttt{RTR2 }}

\subsubsection{Comparison with other methods}
We compare the proposed methods \ttt{Altmin1} and \ttt{RTR2} with \VMC (variety matrix completion) from~\cite{Ongie2017}, described in the related work section on page~\pageref{sec:related_work}. We compare the methods on the recovery of a union of subspaces from a subset of entries, as \VMC is designed for matrix completion. \RTRtwo is a second-order method, while \Altminone and \VMC are both first-order methods which are quite similar in spirit. They both alternate between truncated SVDs of the kernel matrix and some gradient steps, which are performed on different cost functions. \Altminone minimizes a smooth approximation of the Schatten p-norm (Equation~\eqref{eq:Schatten}), while \Altminone minimizes Equation~\eqref{eq:s1}. \Altminone may perform several gradient steps between two SVD, while \VMC performs a single gradient step between two SVD.

Figure~\ref{fig:compare-vmc} shows the decrease in root mean square error (RMSE) over time for the three methods on the completion of a matrix $M\in \R^{15\times s}$ whose columns are contained in a union of two subspaces of dimension two. The total number of points (divided equally across each subspace) is taken as $s=\{100,200,400\}$. Figure~\ref{fig:compare-vmc-c} shows that \ttt{RTR2} clearly outperforms \ttt{VMC} and \Altminone in run-time for matrices with many columns (large $s$). For matrices with fewer columns (Figure~\ref{fig:compare-vmc-a}), \ttt{VMC} performs well in the early iterations in comparison with \ttt{RTR2}, but is consistently slower than \ttt{Altmin1}. Figure~\ref{fig:compare-vmc} indicates that for a comparable runtime, \VMC performs many more iterations than \Altminone does. Each iteration of \ttt{VMC} is therefore faster to compute than an iteration of \ttt{Altmin1}, but they yield a smaller decrease in RMSE.

  \begin{figure}[!htb]
  \centering
    \subfigure[100 data points]{
  \includegraphics[width = 0.45\textwidth]{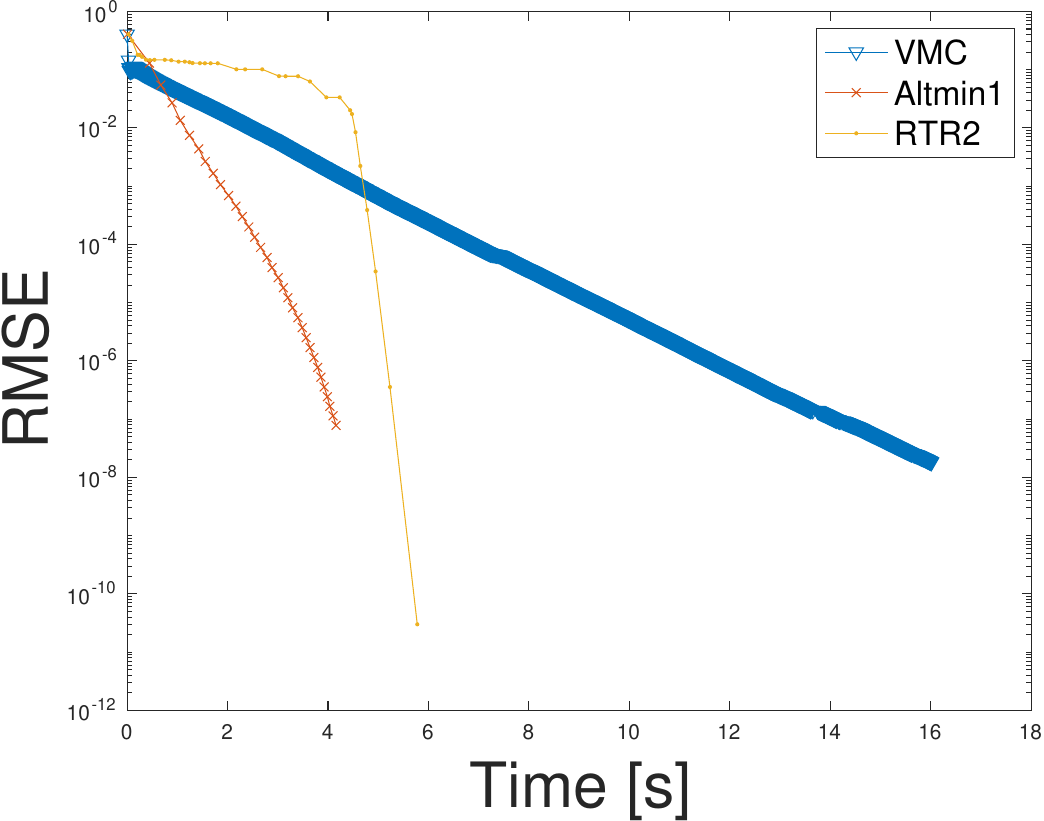}
    \label{fig:compare-vmc-a}
  }
  \quad
  \subfigure[200 data points]{
  \includegraphics[width = 0.45\textwidth]{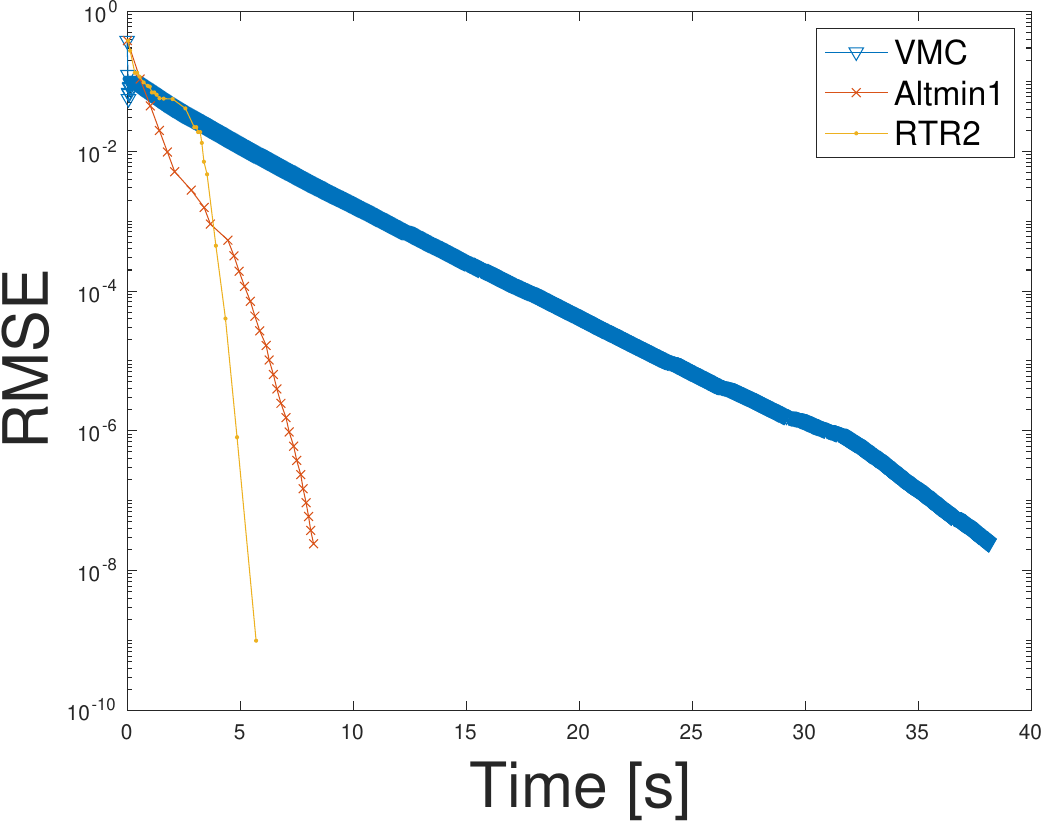}
  }
  \quad
  \subfigure[400 data points]{
  \includegraphics[width = 0.45\textwidth]{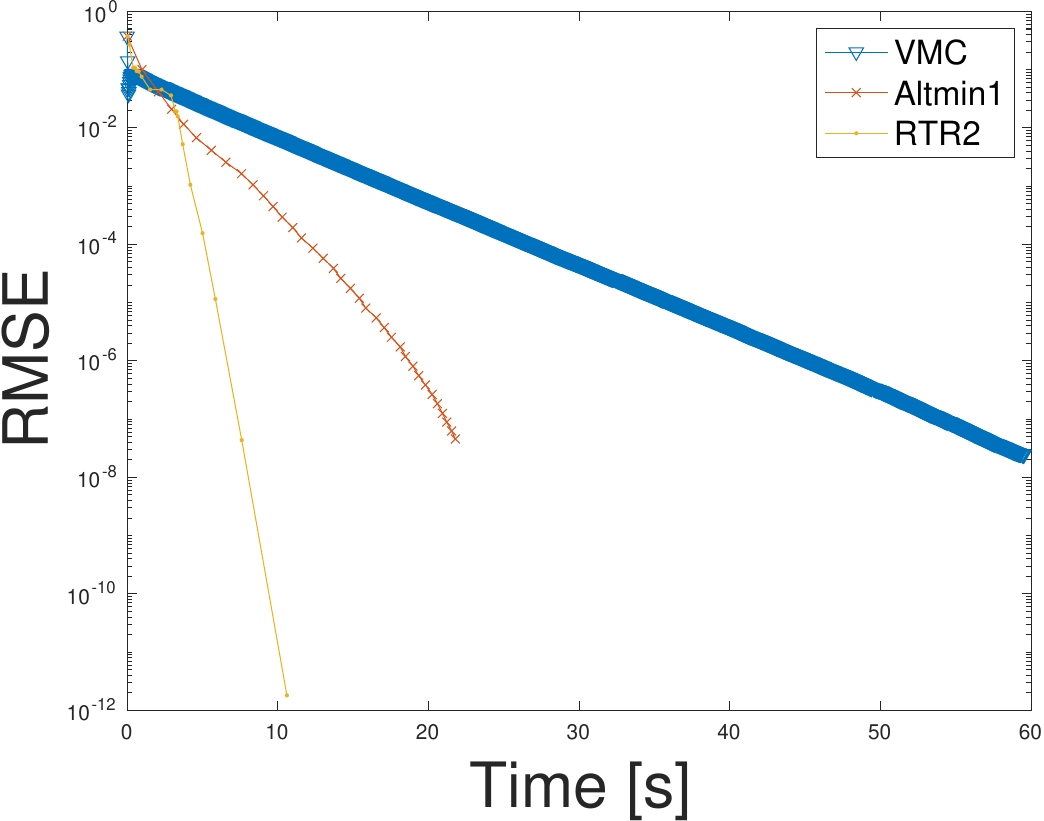}
  \label{fig:compare-vmc-c}
  }
  \caption{Comparison of the proposed \ttt{Altmin1} and \ttt{RTR2} methods with \VMC from~\cite{Ongie2017} on the recovery of a union of 2 subspaces of dimension 2 in $\R^{15}$ with an under-sampling ratio of $0.9$ and an increasing number of points $s$.}
  \label{fig:compare-vmc}
  \end{figure}  
  
Figure~\ref{fig:compare_n_vmc} shows the runtime of the methods \VMC , \Altminone and \RTRtwo for an increasing ambient dimension $n$. Again, \Altminone is consistently faster than \VMC; and we see that both first-order methods perform better than the second-order \RTRtwo in the early iterations as $n$ increases (Figure~\ref{fig:compare_n_c}).
  
   \begin{figure}[!htb]
  \centering
    \subfigure[$n=15$]{
  \includegraphics[width = 0.45\textwidth]{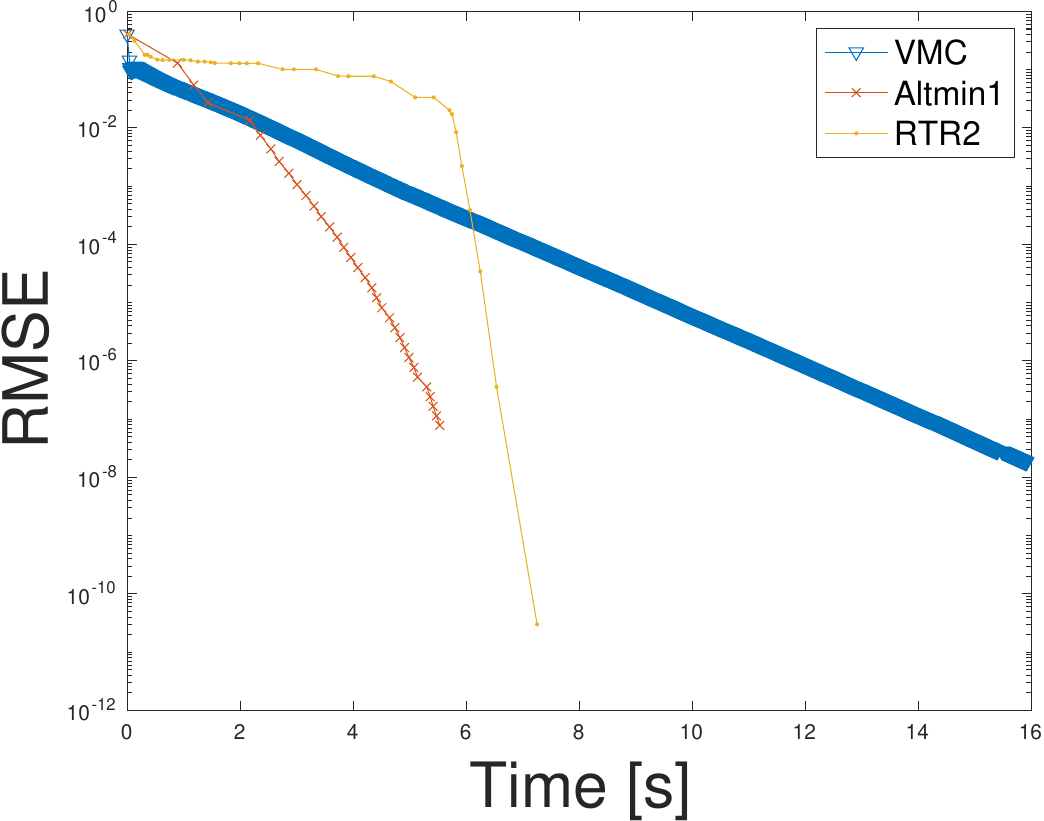}
  }
  \quad
  \subfigure[$n=30$]{
  \includegraphics[width = 0.45\textwidth]{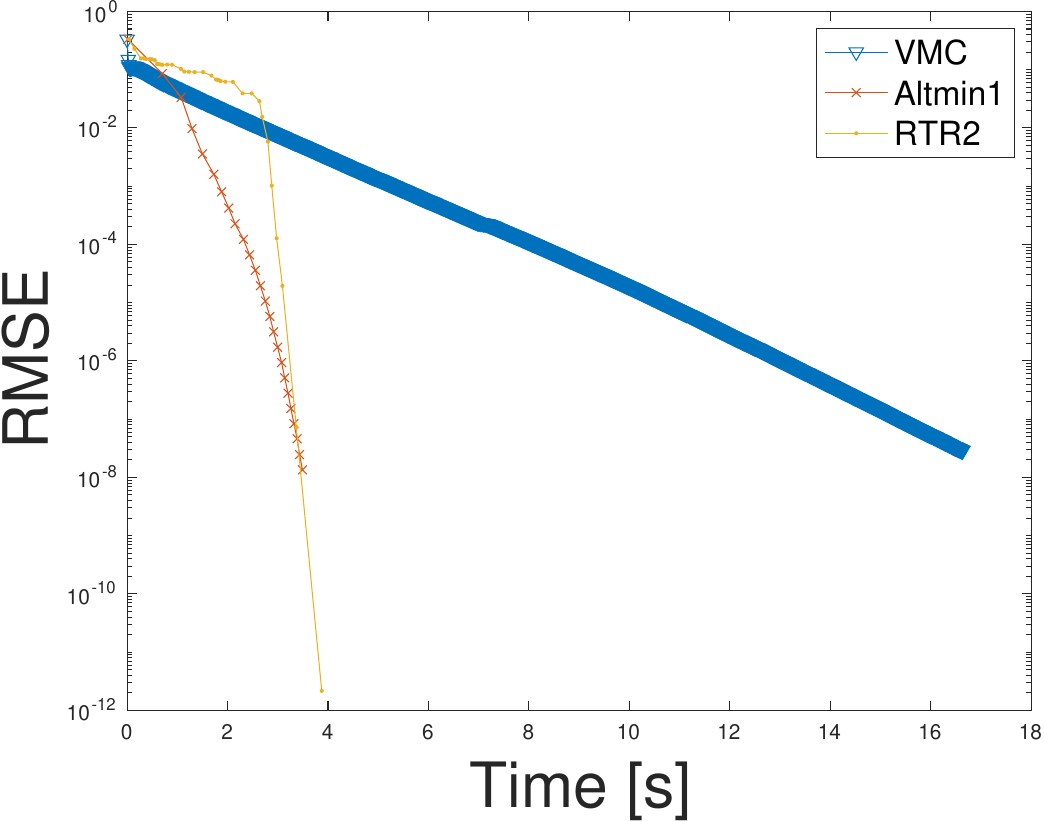}
  }
  \quad
  \subfigure[$n=50$]{
  \includegraphics[width = 0.45\textwidth]{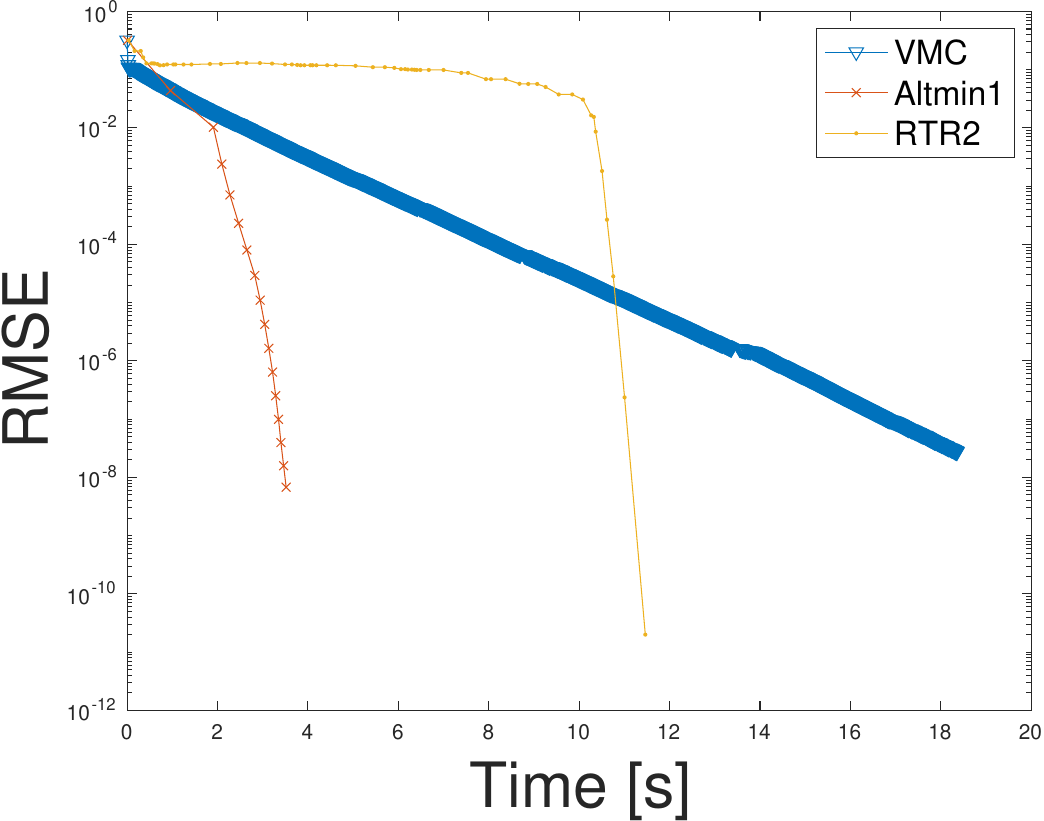}
  \label{fig:compare_n_c}
  }
  \caption{Comparison of the proposed \ttt{Altmin1} and \ttt{RTR2} methods with \VMC from~\cite{Ongie2017} on the recovery of 100 data points belonging to a union of 2 subspaces of dimension 2 in $\R^{n}$ for dimensions $n= \{15,30,50\}$ with an under-sampling ratio of $0.9$.}
  \label{fig:compare_n_vmc}
  \end{figure}
  
Figure~\ref{fig:compare_vmc_recovery} compares the proportion of problems solved for a decreasing undersampling ratio (defined in Equation~\eqref{eq:undersampling} as the ratio of observed entries over the size of the matrix to complete). The recovery rate indicates the proportion of problems solved over a set of $5$ randomly generated problems of recovery of a matrix $M\in \R^{15\times 100}$ whose columns belong to the union of two subspaces of dimension two. The methods \VMC and \RTRtwo perform slightly better than \Altminone at recovering the matrix $M$ when the number of available entries decreases; though the difference is not significant.
  \begin{figure}[!htb]
  \centering
  \includegraphics[width = 0.6\textwidth]{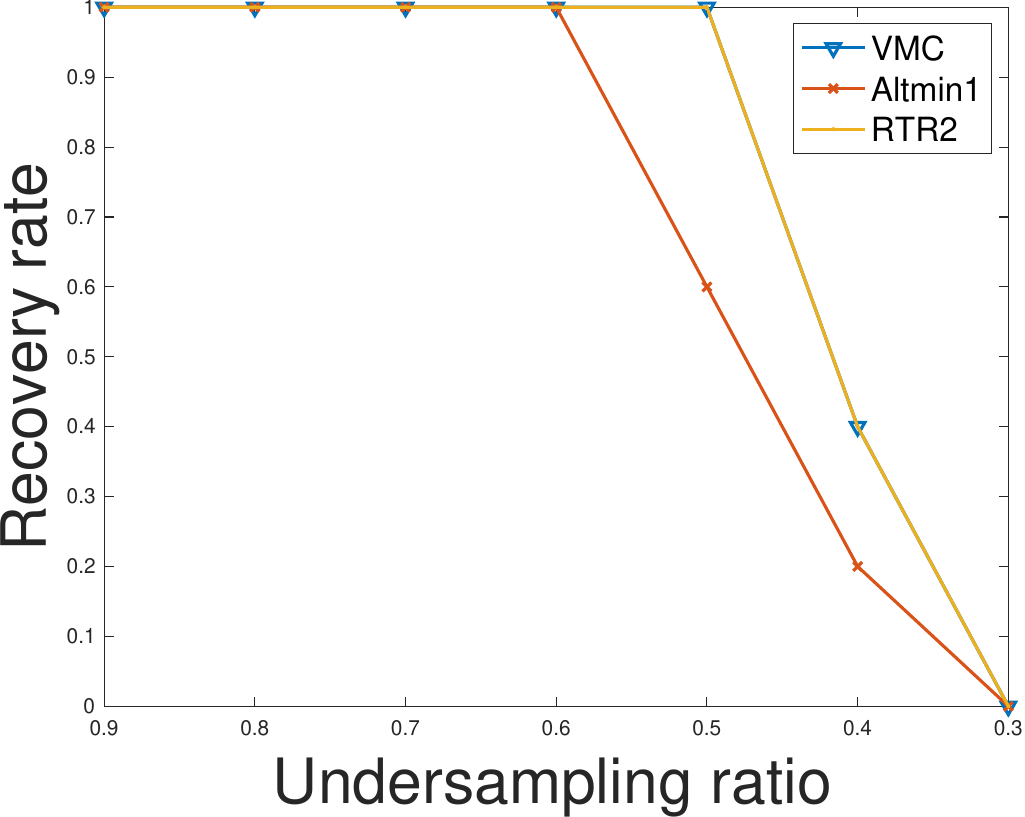}
  \caption{Proportion of problems solved for decreasing under-sampling ratio}
  \label{fig:compare_vmc_recovery}
  \end{figure}

\FloatBarrier

\section{Conclusion}

In this work,  we study the problem of nonlinear matrix completion where one tries to recover a high rank matrix that exhibit low rank structure in a feature space. 
In terms of the use cases considered, in addition to the union of subspaces and algebraic varieties, we propose the use of the Gaussian kernel for clustering problems with missing data, which we believe is novel in the context of nonlinear matrix completion. 

 We propose a novel formulation for the nonlinear matrix completion problem using the Grassmann manifold, which is inspired from low-rank matrix completion techniques.
 We then show how Riemannian optimization and alternating minimization methods can be applied effectively to solve this optimization problem. 
 The algorithms we propose, come with strong global convergence results to critical points and worst-case complexity guarantees. In addition, we show that the alternating minimization algorithm converges to a unique limit point using the Kurdyka-Lojasiewicz property.
 
We provide extensive numerical results that attest to the efficiency of the approach to recover high-rank matrices drawn from union of subspaces or clustered data. We note that the second-order Riemannian trust-region method allows to recover with high accuracy.  We expose the difficulty of using polynomials of high degree in the monomial kernel, as they require an exponentially increasing number of sample points to allow recovery. Our approach proves to be efficient at clustering a data set despite the presence of missing entries and our approach also shows great robustness against the presence of noise in the measurements. Finally, we show that our algorithm greatly outperforms other code available online for nonlinear matrix completion. 

\subsection*{Acknowledgement}
The authors would like to thank Estelle Massart and Greg Ongie for interesting discussions and their helpful ideas.
\appendix

\bibliography{references_nlmc.bib}

\section{Derivatives of cost functions}
\label{sec:appendix_derivatives}
In this section we compute by hand the first- and second-order derivatives of the cost function that appears in the optimization problem~\eqref{eq:p1kernel}. To compute the derivative of a matrix valued function, we write a Taylor expansion and identify the gradient by looking at first order terms. Before computing derivatives for a specific kernel, we look at the Lipschitz continuity properties of the gradient problem, for an arbitrary kernel.

\subsection{Lipschitz properties}
Let us consider the cost function of~\eqref{eq:p1kernel}, defined over $\M = \LAb \times \Grass(s,r)$. We work out conditions on the problem which ensure that the Riemannian gradient is Lipschitz continuous. Lipschitz continuity of a vector field on a smooth manifold is defined as follows using a vector transport.
\begin{definition}\emph{\textbf{(\cite[Definition 10.42]{boumal2022intromanifolds})}}
A vector field $V$ on a connected manifold $\M$ is L-Lipschitz continuous if, for all $x,y\in \M$ with $\dist(x,y) <\textrm{inj}(x)$,
$$\norm{\mathrm{PT}^\gamma_{0\leftarrow 1} V(y) - V(x)} \leq L \dist(x,y),$$
where $\gamma: [0,1] \rightarrow \M$ is the unique minimizing geodesic connecting $x$ to $y$ and $\textrm{PT}^\gamma_{0\leftarrow 1}$ denotes the parallel transport along $\gamma$.
\label{def:lipschitz-continuous-manifold}
\end{definition}
Functions with Lipschitz continuous gradient exhibit the following regularity condition for the pullback $\hat{f} = f\circ \Retr$, provided the retraction used is the exponential map, $\Retr = \Exp$.
\begin{proposition}\emph{\textbf{(\cite[Corollary 10.52]{boumal2022intromanifolds})}}
If $f\colon\M \rightarrow \R$ has L-Lipschitz continuous gradient, then
$$
\left| f(\Exp_x(s)) - f(x) - \inner{s}{\grad f(x)} \right| \leq \dfrac{L}{2}\norm{s}^2
$$
for all $(x,s)$ in the domain of the exponential map.
\label{pro:lip-grad-lip-pullback}
\end{proposition}
In order to show Lipschitz continuity of the gradient, we use both the definition, and the following proposition which uses an upper bound on the derivative of the gradient.
\begin{proposition}\emph{\textbf{(\cite[Corollary 10.45]{boumal2022intromanifolds})}}
If $f:\M \rightarrow \R$ is twice continuously differentiable on a manifold $\M$, then $\grad f$ is L-Lipschitz continuous if and only if $\Hess f(x)$ has operator norm bounded by L for all $x$, that is, if for all x we have
$$
\norm{\Hess f(x)} = \max_{\mathclap{\substack{x\in T_x\M \\
                  \norm{s}=1}}} ~\norm{\Hess f(x)[s]}\leq L.
$$
\label{prop:bounded-hessian-lip-gradient}
\end{proposition}

 We first compute the Euclidean gradient with respect to each variable,
\begin{equation}
\nabla_X f(X,\W) = \D \K(X)^* \Prm_{\W^\perp}
\end{equation}
and 
\begin{equation}
\nabla_\W f(X,\W) = -2 \K(X)\W.
\end{equation}
This naturally gives
\begin{equation}
\nabla^2_{\W\W} f(X,\W)[\Delta] = - 2 \K(X)\Delta. 
\end{equation}

\begin{proposition}
Consider the cost function of~\eqref{eq:p1kernel} and assume that the retraction being used is the exponential map. If $\D \K(X)$ is Lipschitz continuous over $\LAb$, then~\eqref{eq:lipschitz_x_am} holds where $L_x$ is the Lipschitz constant of $\D \K(X)$.  If $\fronorm{\K(X)}\leq M$ for all $X\in \LAb$, condition~\eqref{eq:lipschitz_u_am} holds with $L_u = 2M$.
\end{proposition}
\begin{proof}
For a given $X\in \LAb$, consider the function $f_\W(X,\cdot): \Grass(s,r) \rightarrow \R$. Its Riemannian Hessian is such that for $\Delta \in \Trm_\W \Grass(s,r)$, $\Hess f_\W(\W,X)[\Delta] = 	-2 \Prm_{\W\p} \K(x)\Delta$. Hence $\norm{\Hess f_\W(\W,X)}\leq \fronorm{k(X)}$. Using~\cite{boumal2022intromanifolds} Corollary 10.45, the vector field $\grad f_\W (X,\cdot)$ is Lipschitz continuous with constant $L_\W = 2\fronorm{k(x)}$. If the kernel is upper-bounded for all $X\in \LAb$, the constant $L_\W$ is independent of $X$. This implies that~\eqref{eq:lipschitz_u_am} holds.

We  also analyze Lipschitz continuity of the vector field $\grad f_X$.
\begin{align}
\fronorm{\grad_X f(X_1,\W) - \grad_X f(X_2, \W) } &= \fronorm{\Prm_{T \LAb} \left( \nabla_X f(X_1,\W) - \nabla_X f(X_2,\W) \right)}\\
&\leq \fronorm{ \nabla_X f(X_1,\W) - \nabla_X f(X_2,\W) }\\
&\leq \fronorm{\left(\D \K(X_1)^* - \D \K(X_2)^*\right) \Prm_{\W\p}}\\
&\leq \norm{\D \K(X_1)^* - \D \K(X_2)^*}_2 \fronorm{\Prm_{\W\p}} \\
&\leq \norm{\D \K(X_1) - \D \K(X_2)}_2.
\end{align}
If $\D \K(X)$ is $L_x$-Lipschitz over $\LAb$, we can write
\begin{equation}
\fronorm{\grad_X f(X_1,\W) - \grad_X f(X_2, \W) } \leq L_x \fronorm{X_1-X_2}
\end{equation}
and $\grad_X f(.,\W)$ is also $L_x$-Lipschitz, where the constant $L_x$ is independent of $\W\in \Grass(s,r)$.  This implies that~\eqref{eq:lipschitz_x_am} holds.
\end{proof}

\begin{proposition}
Consider the cost function of either~\eqref{eq:p1} or~\eqref{eq:p1kernel} and apply Algorithm~\ref{algo:am} or Algorithm~\ref{algo:RTR}  with the exponential map as the retraction. If the convex hull of the sequence of iterates $(X_k)_{k\in \mathbb{N}}$ and the trial points is a bounded set, then~\eqref{eq:grad-pullback-lipschitz},
\eqref{eq:hess-pullback-lipschitz} and \eqref{eq:lipschitz_u_am}-\eqref{eq:lipschitz_x_am} hold at every iterate $\left(X_k,\U_k\right)_{k\in\mathbb{N}}$ and trial points of the algorithm.
\end{proposition}
\begin{proof}
Provided the kernel is a smooth function, the derivatives of the cost function are continuous. As a consequence of the Weierstrass theorem, the derivatives are bounded on the closure of the convex hull of the iterates, which is compact ($\Grass(s,r)$ is compact). If the Hessian is bounded on the closure of convex hull of the iterates, the gradient is Lipschitz continuous on that set (Proposition~\ref{prop:bounded-hessian-lip-gradient}) and therefore~\aref{assu:lip-gradient-tr} and~\aref{assumption:pullback_lipschitz_altmin} hold with the exponential map as the retraction (Proposition~\ref{pro:lip-grad-lip-pullback}). The continuity of the third-order derivatives implies~\aref{assu:lip-hessian-tr} in a similar way.
\end{proof}

\subsection{Monomial kernel}
We wish to find the Euclidean derivatives of the cost function in~\eqref{eq:p1kernel} for the monomial kernel defined in~\eqref{eq:mono_kernel}. First, we make the following developments. Up to first order in $\Delta_X$,  
 \begin{equation}
 \begin{aligned}
 \K_d(X+\Delta_X) &= \K_d(X) + d \K_{d-1}(X)\odot(X^\top \Delta_X + \Delta_X^\top X) + \mathcal{O}(\Delta_X^2).
 \end{aligned}
 \end{equation}
 For $\Delta_W = W\p B$, we have 
 \begin{equation}
 \begin{aligned}
 \P_{W+\Delta_W} &= WW\transpose + W\Delta_W\transpose + \Delta_W W\transpose,\\
 				&= \P_W + W\Delta_W\transpose + \Delta_W W\transpose +  \mathcal{O}(\Delta_W^2).\\
 \end{aligned}
 \end{equation}
 Let us write $$f(X,W) = \tr\left(\K_d(X) - \P_{W}\K_d(X)\right). $$
We find the gradient in $X$ using direct computation, 
\begin{equation}
\begin{aligned}
\nabla_X f(X,W) &= \nabla_X  \tr\left(\P_{W\p }\K_d(X)\right),\\
						&= 2 d X\left(\K_{d-1}(X) \odot \P_{W\p }\right),
\end{aligned}
\end{equation}
since $\P_{W\p }$ is symmetric. 
Quite naturally we find $ \nabla_W f(X,W)$ with the expansion, 
 \begin{equation}
\begin{aligned}
 f(X,W+ \Delta_W ) &= \tr\left(\K_d(X) - \P_{W+\Delta_W}\K_d(X)\right) ,\\
				&=  \tr\left(\K_d(X) - (\P_W + W\Delta_W\transpose + \Delta_W W\transpose ) K_d(X)\right) ,\\
				&=  \tr\left(  \P_{W\p }\K_d(X) - ( W\Delta_W\transpose + \Delta_W W\transpose ) K(_dX)\right) ,\\
				&=  \tr\left(  \P_{W\p }\K_d(X)\right)  -\tr  \left( (W\Delta_W\transpose + \Delta_W W\transpose ) \K_d(X)\right) ,\\
				&=  f(X,W)  -\tr  \left( W\Delta_W\transpose \K_d(X)\right) -\tr \left( \Delta_W W\transpose \K(X) )\right) ,\\
				&=  f(X,W)  +\langle \Delta_W,-2  \K_d(X)W\rangle .\\
\end{aligned}
\end{equation}
Ans so we observe $ \nabla_W f(X,W) = -2  \K_d(X)W$.
Quickly we have $ \nabla^2_W f(X,W)[E] = -2  \K_d(X)E$. Now we want to find the second derivative in $X$
 \begin{equation}
\begin{aligned}
 \nabla_X f(X+\Delta_X,W ) &=  2 d (X+\Delta_X \left(\K_{d-1}(X+\Delta_X) \odot \P_{U\p }\right)\\
  				&=   2 d (X+\Delta_X) \left(  \left[\K_{d-1}(X) + (d-1) \K_{d-2}(X)\odot(X^\top\Delta_X + \Delta_X^\top X)\right]  \odot \P_{W\p }\right)\\
  				&=   2 d X \left( \K_{d-1}(X) \odot \P_{W\p }\right) +  2 d(d-1) X \left(  \K_{d-2}(X)\odot(X^\top \Delta_X + \Delta_X^\top X) \odot \P_{W\p }\right)  \\  				&~~~+ 2 d \Delta_X \left( \K_{d-1}(X) \odot \P_{W\p }\right) + \mathcal{O}(\Delta_X^2)\\
  				&=   \nabla_X f(X,W )  +  2 d(d-1) X \left(  \K_{d-2}(X)\odot(X^\top \Delta_X + \Delta_X^\top X) \odot \P_{W\p }\right)  \\  	
  				&~~~+ 2 d \Delta_X \left( \K_{d-1}(X) \odot \P_{W\p }\right) + \mathcal{O}(\Delta_X^2).\\
\end{aligned}
\end{equation}
 And so we identify 
 \begin{equation}
  \nabla_X^2 f(X,W)[\Delta_X] =2 d(d-1) X \left(  \K_{d-2}(X)\odot(X^\top \Delta_X + \Delta_X^\top X) \odot \P_{W\p }\right) + 2 d \Delta_X \left( \K_{d-1}(X) \odot \P_{W\p }\right).
 \end{equation}
 Now we need the cross derivatives $ \nabla_W \nabla_X f(X,W)[\Delta_W]$ and $ \nabla_X \nabla_W f(X,W)[\Delta_W]$. 
\begin{equation}
 \begin{aligned}
  \nabla_W f(X+\Delta_X,W ) 	&= -2  \K(X+\Delta_X)W 	\\
 	&= 	  -2  \left( K_d(X) + d \K_{d-1}(X)\odot(X^\top \Delta_X + \Delta_X^\top X) \right) W \\
 	 	&= 	  -2   K_d(X)W -2d\left(  \K_{d-1}(X)\odot(X^\top \Delta_X + \Delta_X^\top X) \right) W \\
 	&= 	 \nabla_U f(X,W )  -2d\left(  \K_{d-1}(X)\odot(X^\top \Delta_X + \Delta_X^\top X) \right) W. \\
 \end{aligned}
 \end{equation} 
 So $$ \nabla_X \nabla_W f(X,W)[\Delta_X] = -2d\left(  \K_{d-1}(X)\odot(X^\top \Delta_X + \Delta_X^\top X) \right) W \in s \times r.$$ 
Similarly we find
 \begin{equation}
 \begin{aligned}
  \nabla_X f(X,W+\Delta_W ) 	&= 2 d X\left(\K_{d-1}(X) \odot  \P_{(W+\Delta_W)\p } \right)	\\
 	&=  2 d X\left(\K_{d-1}(X) \odot ( \P_{W\p }-W\Delta_W\transpose - \Delta_W W\transpose  ) \right)		   \\
 	&= 2 d X\left(\K_{d-1}(X) \odot  \P_{W\p } \right)	 - 2 d X\left(\K_{d-1}(X) \odot  ( W\Delta_W\transpose - \Delta_W W\transpose  )\right) \\
 	 	&=  \nabla_X f(X,W ) - 2 d X\left(\K_{d-1}(X) \odot  ( W\Delta_W\transpose - \Delta_W W\transpose  )\right). \\
 \end{aligned}
 \end{equation} 
 And so
$$\nabla_W \nabla_X f(X,W)[\nabla_W] = - 2 d X\left(\K_{d-1}(X) \odot  ( W\Delta_W\transpose - \Delta_W W\transpose  )\right) \in n \times s .$$
In the end
\begin{equation}
\nabla^2 f(X,W) \begin{bmatrix}
\Delta_X\\
\Delta_W \\
\end{bmatrix} = \begin{pmatrix}
 \nabla_X^2 f(X,W)[\Delta_X] + \nabla_W \nabla_X f(X,W)[\nabla_W]\\
  	 \nabla_X \nabla_W f(X,W)[\Delta_X] + \nabla_W^2 f(X,W)[\Delta_W].\\
\end{pmatrix}
\end{equation}

\subsection{Gaussian kernel}
Consider the Gaussian kernel defined in~\eqref{eq:gaussian_kernel}. A direct computation gives
\begin{equation}
\nabla_X f(X,W)=  -\dfrac{2}{\sigma^2}X \left( \mathrm{diag} \left(\mathrm{sum}(\K^{G}\odot \P_{W^\perp},1)\right)-\K^{G}\odot \P_{W^\perp}\right)
\end{equation}
for $\K^{G}$ the Gaussian kernel and $\mathrm{sum}(\K^{G}\odot \P_{W^\perp},1)$ is the vector with the sum of each column of the matrix $\K^{G}\odot \P_{W^\perp}$. 
And similarly to the monomial kernel above, we have 
\begin{equation}
\nabla_W f(X,W) = -2\K^{G}(X)W.
\end{equation}
We do not compute the Hessian by hand for the Gaussian kernel. We either use automatic differentiation or finite differences of the gradient in the algorithm.

\section{Proofs for Section~\ref{sec:am_convergence} (Convergence of the alternating minimization algorithm)}
\label{appendix:proofs-6}

\begin{lemma}\emph{\textbf{(Descent lemma based on~\cite[Theorem 4]{boumal2019global})}}
Let \aref{assumption:bounded_below} and \aref{assumption:pullback_lipschitz_altmin} hold for $f:\M \to \R$. 
Then, for any $k\geq 0$, the iterates produced by Algorithm~\ref{algo:am} satisfy
\begin{equation}
f(X_{k},\U_k) - f(X_{k+1},\U_{k+1}) \geq \dfrac{1}{2L_u}\|\grad_\mathcal{U} f(X_{k+1},\U_k)\|^2,
\end{equation}
where $L_u$ is the Lipschitz constant of the gradient of the pullback~(\aref{assumption:pullback_lipschitz_altmin}).
\end{lemma}
\begin{proof}
We follow the development of~\cite[Theorem 4]{boumal2019global}. By Lipschitz continuity of the gradient we have,
\begin{equation}
\left| f\left(X_{k+1}, R_{\calU_k}(\eta)\right) - [f(X_{k+1},\calU_k) + \langle \grad_\calU f(X_{k+1},\calU_k), \eta \rangle] \right| \leq \dfrac{L_u}{2} \norm{\eta}^2~~~~ \forall \eta \in \Trm_\calU\Grass.
\end{equation}
Let $\eta = -\dfrac{1}{L_u} \grad_\calU f(X_{k+1},\calU_k)$ and define $\calU^+ = \Retr_{\calU_k}\left(\dfrac{-1}{L_u} \grad_\calU f(X_{k+1},\calU_k)\right)$, which gives 
\begin{align}
f(X_{k+1},\U^+) &\leq f(X_{k+1},\U_k) + \langle \grad_\U f(X_{k+1},\U_k), \dfrac{-1}{L_u} \grad_U f(X_{k+1},\U_k)\rangle \\\nonumber
&~~~~~+\dfrac{L_u}{2}\norm{\dfrac{-1}{L_u} \grad_\U f(X_{k+1},\U_k)}^2\\
 &\leq  f(X_{k+1},\U_k) -\dfrac{1}{2L_u}\| \grad_\U f(X_{k+1},\U_k)\|^2.
\end{align}
This gives 
\begin{equation}
f(X_{k+1},\calU_k) -f(X_{k+1},\U^+) \geq \dfrac{1}{2L_u}\|\grad_\U f(X_{k+1},\U_k)\|^2.
\end{equation}
Using that the singular value decomposition step of Algorithm~\ref{algo:am} finds the minimum of $f(X_{k+1},\cdot)$ over $\Grass(N,r)$, we have $f(X_{k+1},\U_{k+1}) \leq f(X_{k+1}, \U^+)$. Each update of the variable $X$ is non-increasing, that is, $f(X_{k},\U_k) \geq f(X_{k+1}, \U_k)$. Hence, we can conclude
\begin{equation}
f(X_{k},\calU_k) -f(X_{k+1},\U_{k+1}) \geq f(X_{k+1},\calU_k) -f(X_{k+1},\U_{k+1}) \geq \dfrac{1}{2L_u}\|\grad_\U f(X_{k+1},\U_k)\|^2.
\end{equation}
\end{proof}

\begin{lemma}
Under~\aref{assumption:pullback_lipschitz_altmin}, for the direction
$-\grad_X f(X_k^{(i)},\U_k) \in T_{X_k}\LAb$, 
the linesearch Algorithm~\ref{algoLS} produces a step size $\alpha_k^{(i)}$ that satisfies 
\begin{equation}
\underline{\alpha} := \min \left\lbrace \alpha_0,   \dfrac{ 2\tau(1  - \beta)  }{L_x } \right\rbrace\leq \alpha_k^{(i)} \leq \alpha_{0}
\end{equation}
and produces the following decrease
\begin{equation}
   f(X_k^{(i)},\U_k) -  f(X_k^{(i+1)},\U_k) \geq   \beta \alpha \fronorm{\grad_X  f(X_k^{(i)},\U_k)}^2,
   \label{eq:armijo_decrease_grad_appendix}
\end{equation}
where $X_k^{(i+1)} = X_k^{(i)} - \alpha_k^{(i)} \grad_X f(X_k^{(i)},\U_k)$.
\end{lemma}
\begin{proof}
It is clear from the algorithm that $\alpha_k^{(i)} \leq \alpha_{0}$. The Armijo condition also ensures~\eqref{eq:armijo_decrease_grad_appendix}.
For any $\alpha>0$, Lipschitz continuity of the gradient gives
\begin{align}
f\Big(X_k^{(i)} - \alpha \grad_X f(X_k^{(i)},\U_k) ,\U_k\Big) &\leq f(X_k^{(i)},\U_k) - \alpha \fronorm{ \grad_X f(X_k^{(i)},\U_k)}^2 \\
&~~~~~~~~+ \alpha^2 \dfrac{L_x}{2} \fronorm{\grad_X f(X_k^{(i)},\U_k) }^2.
\end{align}
Hence the Armijo condition~\eqref{eq:armijo_decrease_grad_appendix} is satisfied whenever
\begin{align}
  -\alpha \norm{ \grad_X f(X_k^{(i)},\U_k) }^2 + \alpha^2 \dfrac{L_x}{2} \fronorm{ \grad_X f(X_k^{(i)},\U_k)}^2  \leq -\alpha\beta \fronorm{ \grad_X f(X_k^{(i)},\U_k) }^2,
\end{align}
which simplifies to
\begin{equation}
\alpha \leq\dfrac{ 2(1  - \beta)  }{L_x }=:\alpha_{max}.
\end{equation}
If $\alpha_0$ satisfies Armijo, then $\alpha_k^{(i)} = \alpha_0$. Otherwise, we have $\alpha_k^{(i)} = \tau \alpha_l$ where $\alpha_l> \alpha_{max}$ is the last $\alpha$ that does not satisfy Armijo and $\alpha_{l+1} = \tau \alpha_l$ satisfies Armijo. In this case we have $\alpha_k^{(i)} \geq \tau \alpha_{max}=   \dfrac{ 2\tau(1  - \beta)  }{L_x }$. 
\end{proof}

\end{document}

%% file: preamble_nlmc.tex
\usepackage[T1]{fontenc}
\usepackage[utf8]{inputenc}
\usepackage[english]{babel}
\usepackage{color}
\usepackage{wrapfig}
\usepackage{placeins}
\usepackage{subfigure}
\usepackage{tabu}
\usepackage{fullpage}
\usepackage[squaren]{SIunits}
\usepackage{graphicx}
\usepackage{tikz}
\usepackage{pgfplots}
\usepackage{epstopdf}
\usepackage{epsfig}
\usepackage[hidelinks]{hyperref}
\usepackage{tikz}
\usepackage{tikz-qtree}
\usepackage{eurosym}
\usepackage{amsmath}
\numberwithin{equation}{section}
\usepackage{amssymb}
\usepackage{amsthm}
\usepackage{mathrsfs}
\usepackage{blkarray}

\usepackage{dsfont}
\newcommand{\p}{^{\perp}}
\renewcommand{\P}{\mathrm{P}}

\newcommand{\D}{\mathrm{D}}
\newcommand{\Hrm}{\mathrm{H}}
\newcommand{\Vrm}{\mathrm{V}}
\newcommand{\K}{\mathrm{K}}

\newcommand{\R}{\mathbb{R}}
\newcommand{\Rns}{\mathbb{R}^{n\times s}}
\newcommand{\Rn}{\mathbb{R}^{n}}
\newcommand{\Mcal}{\mathcal{M}}
\newcommand{\M}{\mathcal{M}}
\newcommand{\Ucal}{\mathcal{U}}
\newcommand{\calU}{\mathcal{U}}
\newcommand{\calL}{\mathcal{L}}
\newcommand{\U}{\mathcal{U}}
\newcommand{\W}{\mathcal{W}}
\newcommand{\Acal}{\mathcal{A}}
\newcommand{\A}{\mathcal{A}}
\newcommand{\bigO}{\mathcal{O}}

\newcommand{\dis}{\displaystyle}

\newcommand{\txt}{\text}

\newcommand{\ttt}{\texttt}

\newcommand{\St}{\mathrm{St}}
\newcommand{\grad}{\mathrm{grad}}
\newcommand{\dist}{\mathrm{dist}}

\newcommand{\trace}{\mathrm{trace}}
\newcommand{\Hess}{\mathrm{Hess}}
\newcommand{\Exp}{\mathrm{Exp}}
\newcommand{\LAb}{\mathcal{L}_{\Acal,b}}
\newcommand{\Trm}{\mathrm{T}}
\newcommand{\Prm}{\mathrm{P}}
\newcommand{\Id}{\mathrm{Id}}
\newcommand{\I}{\mathrm{I}}
\newcommand{\rank}{\operatorname{rank}}
\newcommand{\nl}{\newline}
\newcommand{\transpose}{^\top\! }
\newcommand{\inner}[2]{\left\langle{#1},{#2}\right\rangle}

\newcommand\restr[1]{\raisebox{-.5ex}{$|$}_{#1}}
\DeclareMathOperator*{\conv}{conv}
\DeclareMathOperator*{\cl}{cl}
\newcommand{\andt}{\txt{   and   }}

\newcommand{\argmin}{\mathrm{argmin}}
\newcommand{\tr}{\mathrm{tr}}

\newcommand{\Null}{\mathrm{null}}
\newcommand{\Retr}{\mathrm{Retr}}
\newcommand{\range}{\mathrm{range}}
\newcommand{\Grass}{\mathrm{Grass}}
\newcommand{\lambdamin}{\lambda_\mathrm{min}}
\usepackage{algorithm}
\usepackage{algpseudocode}
\usepackage{todonotes}
\usepackage{bm}
\def \SVD {\mathrm{SVD}}

\newcommand{\norm}[1]{\left\|#1\right\|}
\newcommand{\fronorm}[1]{\left\|#1\right\|_\mathrm{F}}

\usepackage{apxproof}
\newtheorem{theorem}			     {Theorem}	[section]
\newtheorem{proposition}[theorem]	 {Proposition}	
\newtheorem{corollary}	  [theorem]	 {Corollary}	
\newtheorem{lemma}	      [theorem]  {Lemma}		
\newtheorem{definition}	         {Definition}[section]
\newtheorem{assumption} {A\ignorespaces}

\newtheorem{testcase}			     {Case study}	[section]
\newtheorem{remark}			     {Remark}	[section]

\usepackage{listings}
\usepackage{textcomp}
\definecolor{listinggray}{gray}{0.9}
\definecolor{lbcolor}{rgb}{0.9,0.9,0.9}
\usepackage{mathtools}
\mathtoolsset{showonlyrefs=false}

%% file: framework.tex
\section{Framework for nonlinear matrix recovery}
\label{sec:framework}
We summarize the different components of the nonlinear matrix recovery problem. The matrix to be completed must be lifted to a higher dimensional space. This can be done through a kernel, in which case one solves problem~\eqref{eq:p1kernel}, or a matrix of features, in which case one solves~\eqref{eq:p1}. When the matrix $M$ to be recovered follows an algebraic variety model, one should use the monomial features or kernel (case study~\ref{example:variety}). When the data is scattered in clusters, 
the Gaussian kernel must be used as lifting (case study~\ref{example:clusters}). 
\begin{center}
\begin{tikzpicture}
\node[above] at (0,10.5) {Data Structure};
\node[above] at (0,10) { in matrix $M$};
\draw[-,thick] (-2,10) to  (2,10);
\draw[-,thick] (2,10) to  (2,11);
\draw[-,thick] (2,11) to  (-2,11);
\draw[-,thick] (-2,11) to  (-2,10);

\draw[->,thick] (-1.5,10) to  (-2.5,9);
\node[above] at (-2.5,8) {Algebraic varieties (e.g. Union of subspaces)};
\draw[->,thick] (-2.5,8) to (-2.5,7);
\node[above] at (-2.5,6) {Lift};

\draw[-,thick] (-3.5,6) to  (-1.5,6);
\draw[-,thick] (-1.5,6) to  (-1.5,6.7);
\draw[-,thick] (-1.5,6.7) to  (-3.5,6.7);
\draw[-,thick] (-3.5,6.7) to  (-3.5,6);

\draw[->,thick] (-2.5,6) to (-0.5,4);
\node[below] at (0,4) {Monomial features};
\draw[->,thick] (-3,6) to (-3.5,5);
\node[below] at (-3.4,5) {Monomial kernel};

\draw[->,thick] (1.5,10) to  (3.5,9);
\node[above] at (4,8) {Clusters};
\draw[->,thick] (4,8) to (4,7);
\node[above] at (4,6) {Lift};
\draw[->,thick] (4,6) to (4,5);
\node[below] at (4,5) {Gaussian kernel};

\draw[-,thick] (3,6) to  (5,6);
\draw[-,thick] (5,6) to  (5,6.7);
\draw[-,thick] (5,6.7) to  (3,6.7);
\draw[-,thick] (3,6.7) to  (3,6);
\end{tikzpicture}
\end{center}
\paragraph{}
In addition, one needs to choose an algorithm to solve the problem formulation~\eqref{eq:p1} or~\eqref{eq:p1kernel}. We propose two families of algorithms: Alternating minimization (Algorithm~\ref{algo:am}) and Riemannian trust region (Algorithm~\ref{algo:RTR}). Each of these algorithms have first- and second-order variants, depending on whether they use the Hessian of the cost function. \nl 
\begin{center}
\begin{tikzpicture}
\node[above] at (0,10.2) {Algorithm};
\draw[-,thick] (-2,10) to  (2,10);
\draw[-,thick] (2,10) to  (2,11);
\draw[-,thick] (2,11) to  (-2,11);
\draw[-,thick] (-2,11) to  (-2,10);
\draw[->,thick] (-1.5,10) to  (-2.5,9);
\node[above] at (-2.5,8) {Riemannian trust-region};
\node[above] at (-2.5,7.5) {(Algorithm~\ref{algo:RTR})};
\draw[->,thick] (1.5,10) to  (3.5,9);
\node[above] at (4,8) {Alternating minimization};
\node[above] at (4,7.5) {(Algorithm~\ref{algo:am})};

\draw[->,thick] (4,7.5) to (5,6.5);
\draw[->,thick] (4,7.5) to (3,6.5);
\node[below] at (2.75,6.5) {first-order};
\node[below] at (5.25,6.5) {second-order};

\draw[->,thick] (-2.5,7.5) to (-3.5,6.5);
\draw[->,thick] (-2.5,7.5) to (-1.5,6.5);
\node[below] at (-3.75,6.5) {first-order};
\node[below] at (-1.25,6.5) {second-order};
\end{tikzpicture}
\end{center}